\documentclass{article} % For LaTeX2e
\usepackage{iclr2027_conference,times}

\usepackage{amsmath,amsfonts,bm}

\def\eqref#1{equation~\ref{#1}}
\def\Eqref#1{Equation~\ref{#1}}
\def\1{\bm{1}}

\DeclareMathAlphabet{\mathsfit}{\encodingdefault}{\sfdefault}{m}{sl}
\SetMathAlphabet{\mathsfit}{bold}{\encodingdefault}{\sfdefault}{bx}{n}

\usepackage{hyperref}       % hyperlinks
\usepackage{url}            % simple URL typesetting
\usepackage{booktabs}       % professional-quality tables
\usepackage{multirow}       % multi-row table cells
\usepackage{amsfonts}       % blackboard math symbols
\usepackage{nicefrac}       % compact symbols for 1/2, etc.
\usepackage{microtype}      % microtypography
\usepackage{xcolor}         % colors

\usepackage{algorithm}
\usepackage{algpseudocode}

\usepackage[pdftex]{graphicx}
\usepackage{amsmath}
\usepackage{amsthm}

\newtheorem{proposition}{Proposition}
\newtheorem{lemma}{Lemma}
\newtheorem*{propone*}{Proposition 1}
\newtheorem*{proptwo*}{Proposition 2}
\newtheorem*{propthree*}{Proposition 3}

\newtheorem{corollary}{Corollary}[proposition]
\newtheorem*{cor21*}{Corollary 2.1}
\newtheorem{assumption}{Assumption}
\theoremstyle{definition}
\newtheorem{definition}{Definition}
\theoremstyle{remark}
\newtheorem*{remark}{Remark}

\title{Walking the Score Manifold: Continuous-time Generative Dynamics on Learned Data Manifolds}

\author{%
  Jan Tauberschmidt$^{1,2}$ \hspace{3mm}
  Brian B. Moser$^1$ \hspace{3mm}
  Stanislav Frolov$^1$\hspace{3mm}
  Andreas Dengel$^{1,2}$ \\
  \vspace{1pt}
  \textbf{Andrew B. Duncan}$^3$ \hspace{3mm}
  \textbf{Sebastian J. Vollmer}$^{1,2}$ \vspace{4pt}\\
  $^1$DFKI Kaiserslautern\quad
  $^2$RPTU Kaiserslautern\quad
  $^3$Imperial College London\\
}

\iclrfinalcopy % Uncomment for camera-ready version, but NOT for submission.
\begin{document}

\maketitle

\begin{abstract}
Generative modeling of time-dependent data is typically formulated on a discrete temporal grid, restricting supervision to the observed timestamps in the training data. We instead frame generation as continuous-time evolution on a learned data manifold. To this end, we leverage pretrained score-based models as geometric priors and learn a vector field that evolves data along score-induced interpolation paths. Because these dynamics follow transitions that respect the geometry learned by the score model, they support generation at arbitrary timestamps and temporal super-resolution beyond the discretization of the training data. Moreover, this geometric formulation allows us to train the vector field simulation-free through a regression objective.
To improve long-horizon rollout robustness, we introduce an objective that promotes path-relative transverse exponential stability. While motivated by stability theory, it admits a practical interpretation as denoising score matching transverse to the interpolation path. Further, we extend the framework to a probabilistic setting that models a distribution over plausible future trajectories.
We demonstrate the method on natural video and scientific dynamical data, including temporal super-resolution, PDE-based spatiotemporal fields, and molecular dynamics. Our results show that score-based priors provide a strong foundation for learning stochastic continuous-time generative dynamics.
\end{abstract}

\section{Introduction}
Generative models of time-dependent data are commonly formulated on a fixed temporal grid \citep{ho2022videodiffusion,singer2022makeavideo,harvey2022flexible}. This discretization is natural for training, since observations are typically available only at sampled timestamps, but it also ties the learned dynamics to the resolution of the data.
Continuous-time models, such as Neural ODEs \citep{chen2018neural}, aim to remove this dependence by representing evolution through a vector field that can be queried at arbitrary times \citep[see also][]{rubanova2019latent,kidger2020neuralcde,park2021vidode,gordon2021latentnde}.
However, even in these models, supervision still enters only at the observed timestamps of the training data. This raises a central question for continuous-time generative dynamics: \emph{how should the state evolve between observed time points?} A continuous-time parameterization alone does not answer this. Without additional structure, intermediate states are constrained only indirectly by the data, and simple Euclidean paths---including the linear interpolants used in many transport-based objectives---can pass through implausible off-manifold regions \citep{lipman2023flowmatching,liu2023rectifiedflow,liu2025videobiflow}.

\looseness=-1 We argue that intermediate-time evolution should instead be guided by a learned prior.
Recent work has shown that score-based models encode nontrivial geometric information about data manifolds, and that this geometry can be used to define manifold-aware, geodesic-like interpolations rather than simple Euclidean straight lines \citep{stanczuk2024intrinsic,diepeveen2025scorebased,park2023understanding,debortoli2022riemannian,saito2025tangential,kapusniak2024metricflowmatching}.
This suggests a different perspective on generative modeling: rather than treating temporal generation as repeated discrete synthesis, we frame it as \emph{walking the score manifold in continuous time}.
The score-based prior provides a learned local geometry, while the learned dynamics determine how to evolve along this geometry over time.

\looseness=-1 In this work, we propose a stochastic continuous-time generative dynamics framework based on this principle.
Conditioned on an observed history, our model learns a vector field in data space.
Its tangential component evolves states along score-induced interpolation paths, while its transverse component performs denoising score matching, encouraging the dynamics to remain close to states that are plausible under the prior.
This yields continuous-time evolution that respects the learned geometry, supports generation at arbitrary timestamps and temporal super-resolution beyond the discretization of the training data, and enables simulation-free learning through local regression targets.
The probabilistic formulation further allows the model to represent distributions over plausible future trajectories.
We demonstrate the resulting framework on natural video and scientific dynamical data, including temporal super-resolution, spatiotemporal fields governed by partial differential equations (PDEs), and molecular dynamics.

Our main contributions are:
1) We formulate time-dependent generative modeling as latent-conditioned continuous-time evolution on a score-induced manifold, enabling arbitrary-time generation from discrete observations.
2) We derive a simulation-free velocity-matching objective by using score-induced interpolation paths as local evolution targets, with a CVAE likelihood interpretation for stochastic futures.
3) We introduce a path-relative transverse robustness objective that corresponds to exponential return toward the interpolation path and admits a denoising-score-matching interpretation.

\section{Score-based Models as Interpolation Priors}
Score-based generative models are trained to approximate the \emph{score} function, i.e. the gradient of the log-density of the data distribution.
Beyond their generative capabilities, recent work has shown that the score function encodes nontrivial geometric information about the data manifold. This includes its intrinsic dimensionality as well as its local tangent or normal structure and associated Riemannian geometry
\citep{stanczuk2024intrinsic,diepeveen2025scorebased,park2023understanding,debortoli2022riemannian,chen2024riemannianflowmatching,saito2025tangential}.
A more detailed review of relevant background is provided in Appendix~\ref{app:score_manifold_background}.

\looseness=-1 In particular, score-based models can be used to define a Riemannian metric \(G(x)\).
Since the score geometry is not generally well behaved at the clean-data limit, we evaluate the metric at a small positive noise level rather than directly on the data space. For notational simplicity, we suppress this distinction in the main text and defer the precise construction to Appendix~\ref{app:score_manifold_background}.
As one concrete option, we follow \citet{saito2025tangential} and define
\(G(x)=J_x^T J_x + \lambda_G I\) through the Jacobian of the score \(J_x\).
This metric will be used to interpolate between two data points \(x_0\) and \(x_1\) through a geodesic construction, which means to find a path \(\gamma^*\) that minimizes the metric-induced path length
\begin{equation}
\label{eq:geodesic_min}
    \gamma^* = \arg\min_{\gamma} \int_0^1 \|\dot{\gamma}_t\|_{G(\gamma_t)} \, dt,
\end{equation}
subject to \(\gamma_0 = x_0\) and \(\gamma_1 = x_1\).
Solving \eqref{eq:geodesic_min} online during training would be expensive and may yield inconsistent targets~\citep{chen2024riemannianflowmatching,saito2025tangential}. We therefore use an amortized interpolation map that can be evaluated efficiently and reused throughout training.
To this end, \citet{kapusniak2024metricflowmatching} propose to learn a neural \emph{bridging correction} term \(\varphi\) that augments linear interpolation:
\begin{equation}
    \gamma_t(x_0,x_1) = (1-t) x_0 + t x_1 + t(1-t)\,\varphi(x_0, x_1, t).
\end{equation}
The correction term is trained to approximate the resulting score-induced geodesic \(\gamma^*\) under the metric \(G\) by minimizing metric energy. Details of the training can be found in Appendix~\ref{app:interp_training}.
Once trained, \(\varphi\) allows for fast and differentiable interpolation and effectively yields an interpolation prior that respects the geometry encoded by the score model.

\looseness=-1 Beyond the geometry, the energy objective also fixes \emph{how} the path is traversed. With \(\mathrm{Len}_G(\gamma)\) the metric length minimized in \eqref{eq:geodesic_min} and \(\beta_\gamma:[0,1]\to[0,1]\), \(\beta_\gamma(t):=\mathrm{Len}_G(\gamma)^{-1}\int_0^t\|\dot\gamma_u\|_{G(\gamma_u)}\,du\), the fraction of metric arc length covered by time \(t\), the energy factorizes as
\[
\frac12\int_0^1\|\dot\gamma_t\|_{G(\gamma_t)}^2\,dt
=
\frac12\,\mathrm{Len}_G(\gamma)^2
\int_0^1\dot\beta_\gamma(t)^2\,dt
\ge
\frac12\,\mathrm{Len}_G(\gamma)^2,
\]
with equality exactly when \(\gamma\) has uniform metric speed. A formal statement and proof are given in Appendix~\ref{app:uniform_speed}.

%%%%%%%%%%%%%%%%%%%%%%%%%%%%%%%%%%%%%%%%%%%%%%%%%%%%%%%%%%%%%%%%%%%%%%%%%%%%%%%%%%%%%%%%%%%%%%%%%%%%%%%%

\section{Method: Generative Dynamics along Interpolation Paths}

\looseness=-1 We propose to frame the problem of continuous-time generation as navigating the manifold of a trained score-based model. To this end, we learn an autonomous vector field \(v_\theta(x)\) such that solving the initial value problem
\(\dot x_t = v_\theta(x_t)\) with \(x_{t_0}=x_0\)
generates trajectories that follow the interpolation paths induced by the score model.\\
Given a discrete training video
\(
x_{0:K} := (x_0,x_1,\dots,x_K) \in (\mathbb R^d)^{K+1}
\)
with time steps
\(
{t_{0:K}=(t_0,t_1,\dots,t_K)},
\)
we obtain an interpolating path \(\gamma:[0,T]\to\mathbb R^d\) induced by the score model and interpolator. Each interpolated point is indexed by a unique continuous time \(t\in[0,T]\), and we write
\(
x_t := \gamma(t).
\)
The interpolating path provides local evolution targets at multiple time scales. We therefore define
\[
\Delta_h x_t :=
\begin{cases}
\dfrac{x_{t+h}-x_t}{h}, & h>0,\\[1ex]
\dot x_t, & h=0.
\end{cases}
\]
and refer to \(\Delta_h x_t\) as the local evolution target. We condition the velocity model on the desired step size and learn \(v_\theta(x,h)\) to match these local targets.
Training is simulation-free, reducing continuous-time learning to local regression along interpolation paths without backpropagating through an ODE solver. At inference time, the \(h=0\) slice defines the continuous-time vector field, while \(h>0\) provides finite-step updates for fast rollouts.
For fitting a single sample trajectory \(x_{0:K}\), the natural regression objective then is
\[
\mathcal L(\theta; x_{0:K}) = \mathbb E_{t,h}\bigl[\|\Delta_h x_t - v_\theta(x_t,h)\|^2\bigr].
\]
Here, the expectation uses suitably chosen distributions over sampled time steps and prediction step sizes.
All notation is written in data space. However, as is standard, our method can be applied in the latent space of a pre-trained autoencoder without further modification, as we demonstrate on natural video in Section~\ref{sec:temporal_sr}.
Figure \ref{fig:fig1} summarizes the underlying idea of our model,
showing how we leverage the score field of a pre-trained model to obtain a generative vector field.
We now extend this idea presented for the deterministic single-trajectory case above into a full generative model.

\begin{figure}
  \centering
  \begin{minipage}[t]{0.4606\textwidth}\centering
    \includegraphics[width=\linewidth]{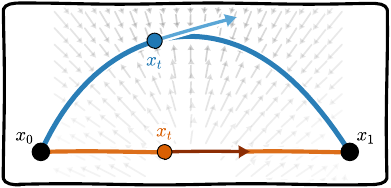}
  \end{minipage}
  \hfill
  \begin{minipage}[t]{0.4606\textwidth}\centering
    \includegraphics[width=\linewidth]{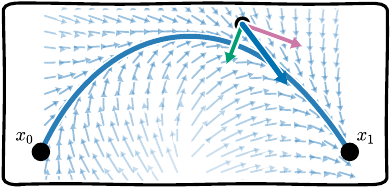}
  \end{minipage}
  \caption{Overview. Left: a pretrained score model (grey arrows) induces manifold-aware interpolation paths (blue curve). Right: the generative vector field follows these paths tangentially while applying a transverse correction (see Sec.~\ref{sec:method_stability}).}
  \label{fig:fig1}
\end{figure}

\subsection{Statistical framework}

We model time-dependent data as a continuous-time stochastic process \({X=(X_t)_{t\in\mathcal T}}\) with \({\mathcal T=[0,T]}\)
assumed to be almost surely differentiable. We denote random variables using capital letters and realizations or integration variables in lowercase. 
Directly learning a deterministic vector field is ill-posed when the same history \(x_{\le t}\) admits multiple plausible future continuations. To capture this ambiguity, we introduce a latent variable \(z\) and model the conditional distribution of local evolution through
\begin{equation}
p(\Delta_h x_t \mid x_{\le t})
=
\int p_\theta(\Delta_h x_t \mid x_t,z)\, p_\vartheta(z\mid x_{\le t})\,dz.
\label{eq:latent_local_evolution}
\end{equation}
We assume the conditional independence structure
\(
p_\theta(\Delta_h x_t \mid x_{\le t},z)
=
p_\theta(\Delta_h x_t \mid x_t,z),
\)
that is, once \(z\) is given, the dependence of the future evolution on the past is absorbed into the latent variable.
In this sense, \(z\) is drawn once per sequence and captures trajectory-level hidden factors that help disambiguate one particular future realization among those compatible with the observed history.

Under consistency assumptions, Kolmogorov's extension theorem provides the existence of a path-space distribution compatible with these local specifications in \Eqref{eq:latent_local_evolution}.
The generative dynamics, however, follow an ODE driven by a random latent. In particular, this means a process with zero quadratic variation, so the standard machinery for deriving a continuous-time path-space lower bound does not directly apply.
Instead, we proceed with a local variational formulation and support the resulting training objective through a discrete-time statement in Proposition~\ref{prop:lower_bound}.

For each fixed \((t,h)\), \Eqref{eq:latent_local_evolution} defines a conditional variational autoencoder (CVAE) model for local trajectory evolution~\citep{sohn2015learning}.
As in learned-prior CVAEs for stochastic video prediction~\citep{denton2018stochastic}, we introduce an amortized full-trajectory posterior \(q_\phi(z\mid x)\).
Since the posterior receives the full sequence, it allows to condition the local evolution on the specific realization given by a training sample.
Jensen's inequality then yields the usual CVAE evidence lower bound (ELBO)
\begin{equation}
\log p(\Delta_h x_t \vert x_{\le t})
\ge
\mathbb E_{q_\phi}
\bigl[\log p_\theta(\Delta_h x_t \vert x_t,z)\bigr]
-
\mathrm{KL}\bigl(q_\phi(z\vert x)\,\|\,p_\vartheta(z\vert x_{\le t})\bigr).
\label{eq:local_elbo}
\end{equation}
We treat the distributions over time points \(t\) and horizons \(h\) as design choices of the training procedure. 
Averaging the pointwise bound over these choices gives
\begin{equation}
\mathbb E_{t,h}\bigl[\log p(\Delta_h x_t \vert x_{\le t})\bigr]
\ge
\mathbb E_{t,h,q_\phi}
\bigl[\log p_\theta(\Delta_h x_t \vert x_t,z)\bigr]
-
\mathbb E_t\Bigl[
\mathrm{KL}\bigl(q_\phi(z\vert x)\,\|\,p_\vartheta(z\vert x_{\le t})\bigr)
\Bigr].
\label{eq:avg_local_elbo}
\end{equation}
To obtain the regression objective used for training, we choose a Gaussian likelihood
\({
\Delta_h x_t \vert x_t,z
\sim
\mathcal N\!\bigl(\Delta_h x_t;\,v_\theta(x_t,z,h),\,\sigma^2 I\bigr),
}\)
where we condition the velocity \(v_\theta\) on the latent \(z\) and on the step size \(h\).
Then, the base minimization objective becomes \( \mathbb E\bigl[ \mathcal L_{\mathrm{base}}(X)\bigr]\) with
\begin{equation}
\mathcal L_{\mathrm{base}}(x)
=
\frac{1}{2\sigma^2}\,
\mathbb E_{t,h}\mathbb E_{q_\phi(z\vert x)}
\Bigl[
\|\Delta_h x_t - v_\theta(x_t,z,h)\|^2
\Bigr]
+
\mathbb E_t\Bigl[
\mathrm{KL}\bigl(q_\phi(z\vert x)\,\|\,p_\vartheta(z\vert x_{\le t})\bigr)
\Bigr].
\label{eq:training_loss}
\end{equation}
The reconstruction term therefore reduces to a squared velocity-matching objective, while the KL term aligns the full-trajectory posterior with the history-conditioned prior. 
We now return to the sequence-level likelihood question raised above.
For the discrete-time statement, we consider the natural joint model in which a single latent variable drives the full trajectory. 
Under this model, Proposition~\ref{prop:lower_bound} shows that the purely local objective also yields a lower bound on the likelihood of the full sequence.

\begin{proposition}[Discretized likelihood lower bound]
\label{prop:lower_bound}
Under the fixed discretization setup of Appendix~\ref{app:prop1}, the lower bound of \eqref{eq:avg_local_elbo} also is a lower bound on the log-likelihood of a sequence $(x_{t_0},\dots,x_{t_K})$, up to scaling by \(K\) and additive constants.
A full statement and proof are given in Appendix~\ref{app:prop1}.
\end{proposition}

\subsection{Path-relative transverse exponential stability}
\label{sec:method_stability}

Neural ODE-based generative models suffer from instability, drift, and numerical error accumulation during rollout. In our setting, this issue is amplified by the fact that training is simulation-free: supervision is purely local via the regression objective, without explicitly constraining long-term trajectories. As a result, small local errors may accumulate and lead to deviations from the learned manifold.
To address this, we study the ideal geometric structure of a stable vector field around a reference trajectory and then incorporate this structure into the training objective.

\looseness=-1 Let \(\Gamma=\{x_t:t\in[0,T]\}\) be a smooth reference path. To avoid endpoint artifacts, we assume that \(\Gamma\) is contained in a smooth embedded extension \(\widetilde\Gamma\) and work in a tubular neighborhood \(U_\rho\) of \(\Gamma\) on which the nearest-point projection
\(\pi:U_\rho\to\widetilde\Gamma\)
is well defined. For \(x\in U_\rho\), write
\(e(x):=x-\pi(x)\)
for the transverse displacement. Let \(\tau(\pi(x))\) denote the tangent velocity of the extended path at the projected point, and let \(\Pi^\perp(\pi(x))\)
denote the Euclidean projection onto the normal space at \(\pi(x)\). The full tubular-neighborhood assumptions are stated in Appendix~\ref{app:prop2}.
We define the ideal path-relative field as
\begin{equation}
\label{eq:target_vf}
v^\star(x)
=
\tau(\pi(x))-\lambda e(x),
\end{equation}
where \(\lambda>0\). Its first term preserves tangential motion along the reference path, while the second term contracts only transverse displacement. Thus \(v^\star\) is the simplest field that follows the interpolation path and returns off-path states back toward it.

\begin{definition}[Path-relative transverse exponential stability]
A vector field \(v\) is said to be \emph{path-relative transversely exponentially stable} with rate \(\mu>0\) if every trajectory \(y_t\) satisfying \(\dot y_t = v(y_t)\) obeys
\[
\|e(y_t)\|
\le
\exp(-\mu(t-t_0))\|e(y_{t_0})\|
\]
% for all initial conditions \(y_{t_0}\in U_\rho\) and all times \(t>t_0\) for which the solution is defined.
for all \(t\in[t_0,T_{\max})\), where \(T_{\max}\) is the maximal time such that \(y_s\in U_\rho\) for all \(s\in[t_0,t]\).
\end{definition}
By design (Equation~\ref{eq:target_vf}), the target field \(v^\star\) has the desired exponential-return property.
\begin{proposition}[Exponential decay of transverse displacement]
\label{prop:exp_error_decay}
Under the tubular-neighborhood assumptions of Appendix~\ref{app:prop2}, the vector field \(v^\star\) is path-relative transversely exponentially stable on \(U_\rho\) with rate \(\lambda\), up to the possible longitudinal exit time from the tubular neighborhood.
\end{proposition}

This is a finite-amplitude, trajectory-relative stability statement. The same construction, however, also implies the standard infinitesimal transverse contraction criterion of \citet{manchester2017control}.

\begin{corollary}[Transverse contraction]
\label{cor:standard_transverse_contraction}
Under the assumptions of Proposition~\ref{prop:exp_error_decay}, the vector field \(v^\star\) is transversely contracting in the sense of~\cite{manchester2017control}. A precise statement and proof are given in Appendix~\ref{app:cor21}.
\end{corollary}

For generation, we will approximate \(v^\star\) with a learned \(v_\theta\).
Proposition~\ref{prop:approx_tube} shows that approximation errors are well-behaved by comparing transverse contraction at rate  \(\lambda\) with transverse field error \(\varepsilon_\perp\).

\begin{proposition}[Approximation error contracts to a tube]
\label{prop:approx_tube}
Under the assumptions of Proposition~\ref{prop:exp_error_decay}, let \(v_\theta\) be locally Lipschitz on \(U_\rho\) with transverse approximation error \({\varepsilon_\perp:=\sup_{x\in U_\rho}\|\Pi^\perp(\pi(x))(v_\theta(x)-v^\star(x))\|<\lambda\rho}\). Then, up to the possible longitudinal exit time from the tubular neighborhood, every trajectory of \(\dot y_t=v_\theta(y_t)\) starting in \(U_\rho\) satisfies \({\|e(y_t)\|\le \exp(-\lambda(t-t_0))\|e(y_{t_0})\|+\varepsilon_\perp/\lambda}\). A precise statement and proof are given in Appendix~\ref{app:prop3}.
\end{proposition}

\looseness=-1 We now turn the ideal field \(v^\star\) into a practical regression target for \(v_\theta\). The stability result above prescribes two components: along the path, the model should match the interpolation velocity, and away from the path, the model should apply a transverse correction that contracts perturbations exponentially. This construction is meaningful only after conditioning on \(z\), since \(z\) must localize the relevant future path.
Given a point \(x_t\) on an interpolation path and a sampled step size \(h\), we perturb it as \(\widetilde x_t=x_t+\sigma_k\eta\), where \(\sigma_k\) is a noise scale, \(\|\eta\|=1\), and \(\eta^\top\Delta_hx_t=0\). Thus \(\eta\) is normal to the tangent when \(h=0\) and to the secant otherwise. For \(h>0\), exact decay requires the updated point to lie at distance \(\exp(-\lambda h)\|e\|\) from \(x_{t+h}\), where \(e=\sigma_k\eta\). Since this does not determine a unique target, we select the one closest to the uncorrected secant target \(\Delta_h x_t\). The full derivation in Appendix~\ref{app:secant_objective} gives
\begin{equation}
\label{eq:corrected_secant_target}
\widetilde{\Delta}_h x_t=\Delta_h x_t+c(h,\lambda)\sigma_k\eta,
\qquad
c(h,\lambda):=\frac{\exp(-\lambda h)-1}{h}\ \ (h>0),
\qquad
c(0,\lambda):=-\lambda.
\end{equation}

At \(h=0\), this correction also has a denoising-score-matching interpretation. If the normal perturbation is viewed as Gaussian noise in the normal bundle, then the conditional score in the normal subspace is \(\nabla_{\widetilde x}^{\perp}\log p(\widetilde x\mid x)=-e/\sigma_k^2\), where \(\widetilde x=x+e\) and \(e\sim \mathcal N\!\bigl(0,\sigma_k^2\Pi^\perp(\pi(x))\bigr)\).
Thus, for \(\lambda=\sigma_k^{-2}\), the infinitesimal correction \(-\lambda e\) in \eqref{eq:target_vf} coincides with the normal denoising score. In this sense, the robustified target combines tangential velocity matching with denoising score matching in transverse directions~\citep{song2019generative}.
Using multiple noise scales \(\sigma_1<\cdots<\sigma_K\), we replace the reconstruction term in \eqref{eq:training_loss} by
\[
\mathcal L_{\mathrm{stab}}(x)
=
\mathbb E_{t,h,k,\eta}\mathbb E_{q_\phi(z\mid x)}
\bigl[
\|\widetilde{\Delta}_h x_t - v_\theta(x_t+\sigma_k\eta,z,h)\|^2
\bigr].
\]
Figure~\ref{fig:fig1} illustrates the resulting geometry: the learned field follows the interpolation path tangentially while pointing perturbed states back toward the path in transverse directions.
A central challenge in this formulation is that the latent variable \(z\) must disambiguate the underlying trajectory. Otherwise, the contraction target may assign inconsistent labels across nearby but distinct paths, collapsing modes by pulling states into the wrong basin. We empirically probe this failure mode in the 2D example of Section~\ref{sec:toy} and further demonstrate stable long-horizon stochastic rollouts in Section~\ref{sec:pde_exp}.

Overall, our method applies stochastic gradient descent on \(\mathbb E \bigl[ \mathcal L(X)\bigr]\) with
\[
\mathcal L(x)
=
\frac{1}{2\sigma^2}\,\mathbb E_{t,h,k,\eta}\mathbb E_{q_\phi(z\mid x)}
\bigl[
\|\widetilde{\Delta}_h x_t - v_\theta(x_t+\sigma_k\eta,z,h)\|^2
\bigr]
+
\mathbb E_t\Bigl[
\mathrm{KL}\bigl(q_\phi(z\mid x)\,\|\,p_\vartheta(z\mid x_{\le t})\bigr)
\Bigr].
\]
Here, \(t\), \(h\), and the noise scale index \(k\) are typically sampled uniformly, while the noise levels \(\sigma_k\) are chosen on a geometric grid.
Therefore, the complete model consists of a frozen score-based prior, a frozen amortized interpolator, a latent posterior \(q_\phi(z\mid x)\), a history-conditioned prior \(p_\vartheta(z\mid x_{\le t})\), and a step-conditioned velocity model \(v_\theta(x_t,z,h)\). Only the latter three components are trained in the final stage.
The full training algorithm is given in Appendix~\ref{app:full_training_alg}.

%%%%%%%%%%%%%%%%%%%%%%%%%%%%%%%%%%%%%%%%%%%%%%%%%%%%%%%%%%%%%%%%%%%%%%%%%%%%%%%%%%%%%%%%%%%%%%%%%%%%%%%%%%%%%%%%%%%%%%%%%%%%%%%%%%%%%%%%

\section{Experiments}
The experimental evaluation of our method is intended to demonstrate feasibility of our concept, potential applications, and to highlight failure modes of unconstrained continuous-time methods.
First, we provide a two-dimensional example that already contains every element of problem we address (Section~\ref{sec:toy}). The remaining experiments demonstrate applications to PDE-governed fields, molecular dynamics, and natural video, and analyze rollout stability (Section~\ref{sec:robustness_exp}), the plausibility of generated intermediates (Section~\ref{sec:plausibility}), and stochastic long-horizon generation (Section~\ref{sec:pde_exp}).

\begin{figure}[t]
  \centering
  \newcommand{\panelw}{0.1786\textwidth}

  \begin{minipage}[t]{\panelw}\centering
    \includegraphics[width=\linewidth]{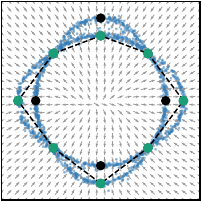}\\
    \footnotesize (a) score field
  \end{minipage}
  \hfill
  \begin{minipage}[t]{\panelw}\centering
    \includegraphics[width=\linewidth]{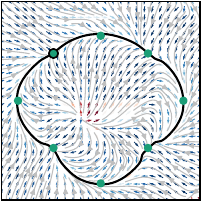}\\
    \footnotesize (b) ours
  \end{minipage}
  \hfill
  \begin{minipage}[t]{\panelw}\centering
    \includegraphics[width=\linewidth]{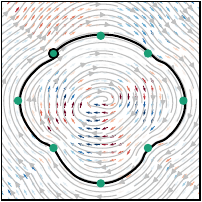}\\
    \footnotesize (c) no correction
  \end{minipage}
  \hfill
  \begin{minipage}[t]{\panelw}\centering
    \includegraphics[width=\linewidth]{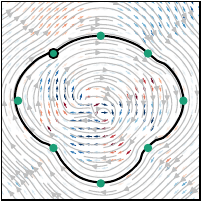}\\
    \footnotesize (d) no noise
  \end{minipage}
    \hfill
  \begin{minipage}[t]{\panelw}\centering
    \includegraphics[width=\linewidth]{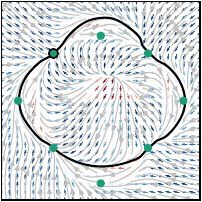}\\
    \footnotesize (e) no latent
  \end{minipage}
  \caption{2D example. (a) Score prior: model samples (blue), score field (gray), candidate nodes (black), active loop (green). (b)--(e) Learned fields \(v_\theta(x\mid z)\), conditioned on the green nodes. Arrows colored by their transverse component toward (blue) or away from (red) the reference path (black).}
  \label{fig:toy_fields}
\end{figure}

\begin{table}[b]
\centering
\caption{Measured transverse contraction: median pointwise eigenvalue \(\lambda_\perp\) with the fraction of the tube where it is negative and median empirical decay rate of transverse perturbations.}
\label{tab:contraction}
\footnotesize
\begin{tabular}{l ccc cc}
  \toprule
  & \multicolumn{3}{c}{2D example ($\lambda=10$)} & \multicolumn{2}{c}{Navier--Stokes ($\lambda=1$)} \\
  \cmidrule(lr){2-4}\cmidrule(lr){5-6}
  Model & med.\ $\lambda_\perp$ & frac.\ $\lambda_\perp<0$ & rate $\downarrow$ & rate $\downarrow$ & frac.\ contr.\ $\uparrow$ \\
  \midrule
  ours        & \textbf{$-$8.28} & \textbf{0.98} & \textbf{$-$5.4} & \textbf{$-$0.90} & \textbf{1.00} \\
  input noise & $-$0.31 & 0.68 & $+$0.0 & $+$0.00 & 0.47 \\
  vanilla     & $-$0.20 & 0.81 & $-$0.2 & $-$0.13 & 0.75 \\
  \bottomrule
\end{tabular}
\end{table}

\subsection{Demonstrative Example}
\label{sec:toy}

\looseness=-1 We construct a two-dimensional example in which the data manifold is a square whose adjacent corners are connected by an inner and an outer arc.
Each sequence is a closed loop of length \(9\) alternating (noised) corners and arc midpoints with the arc chosen independently on every side, resulting in \(2^4\) possible branching combinations.
Figure~\ref{fig:toy_fields} (a) shows this setup with samples from the data manifold and one training sequence.
The score prior is pretrained on the arcs, while the generative model observes only the sparse sequences (Appendix~\ref{data:toy}), testing three properties:
\textbf{(1)~Manifold.} Generation needs to respect the curved arcs between observed points.
\textbf{(2)~Stochasticity.} Branching requires conditioning and a partial history should preserve stochasticity in the remaining trajectory.
\textbf{(3)~Transverse contraction.} The field should contract towards a unique path, meaning that latent conditioning needs to prevent states being pulled toward the wrong branch.

\looseness=-1 Figure~\ref{fig:toy_fields} (b)-(e) shows \(v_\theta\) conditioned on the sequence in green and a trajectory generated from this field starting from the marker with the black edge. We show multiple ablations, where all models have learned manifold-respecting vector fields that follow the arcs.
The full method's field in (b) is globally coherent and transports the entire domain onto the conditioned path.
Without the target correction (c), and likewise without noise (d), off-path states circulate instead of returning. Without the latent (e) the field contracts but not to a unique path.

Further results in Appendix~\ref{app:res_toy} demonstrate plausible stochasticity from partial conditioning (Figure~\ref{fig:toy_example} and Figure~\ref{fig:app_toy}) and the qualitative results of Figure~\ref{fig:toy_fields} are also reflected in Table~\ref{tab:toy_ablation}, which quantifies manifold adherence, branching accuracy, and measured contraction rates for the ablations.

\subsection{Rollout stability from transverse contraction}
\label{sec:robustness_exp}

\looseness=-1 We isolate rollout stability in a controlled reconstruction setting by training on a single synthetic Navier--Stokes sequence (details in Appendix~\ref{data:pde}). We compare the full method against two ablations: training without noise, and training with orthogonal noise but without correcting the regression targets, i.e.\ pure input augmentation. Figure~\ref{fig:contracting} (Appendix~\ref{app:res_contracting}) reports the mean LPIPS~\citep{zhang2018unreasonable} over a 64-frame rollout under three inference regimes---direct steps with \(h>0\), fixed-step, and adaptive ODE solving---and shows consistently improved stability with the proposed robustification across solver settings. Qualitative rollouts in the same appendix show an even larger effect: without robustification, iterative errors accumulate and eventually cause severe artifacts.

\looseness=-1 To test the mechanism itself, we extract pointwise transverse eigenvalues \(\lambda_\perp\) around the conditioned path and report empirical decay rates, obtained by integrating a perturbed and an unperturbed state under the same field (details in Appendix~\ref{app:inference_details}). Results in Table~\ref{tab:contraction} verify the desired contracting properties for both diagnostics and confirm the qualitative behavior of the vector fields shown in Figure~\ref{fig:toy_fields}.

\begin{figure}[b]
  \centering
  \begin{minipage}[t]{0.2256\textwidth}\centering
    \includegraphics[trim=21 30 14 14, clip, width=\linewidth]{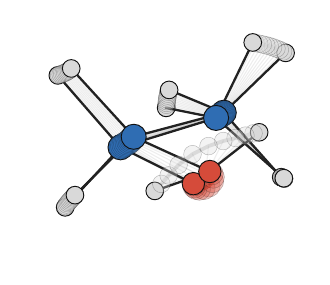}\\
    \footnotesize (a) Ours
  \end{minipage}
  \hfill
  \begin{minipage}[t]{0.2256\textwidth}\centering
    \includegraphics[trim=21 30 14 14, clip, width=\linewidth]{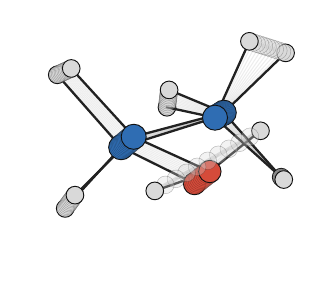}\\
    \footnotesize (b) Linear
  \end{minipage}
  \hfill
  \begin{minipage}[t]{0.2256\textwidth}\centering
    \includegraphics[width=\linewidth]{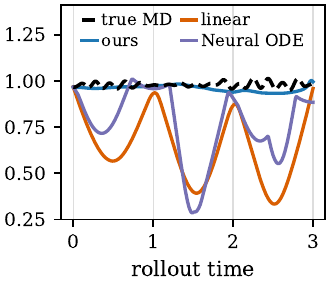}\\
    \footnotesize (c) O-H bond length (\r{A})
  \end{minipage}
  \hfill
  \begin{minipage}[t]{0.2256\textwidth}\centering
    \includegraphics[width=\linewidth]{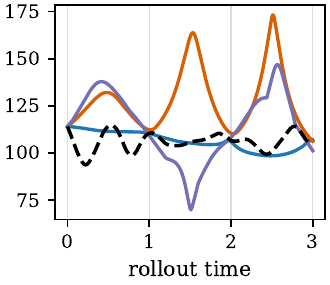}\\
    \footnotesize (d) C-O-H angle (deg)
  \end{minipage}
  \caption{Molecular dynamics on MD17 ethanol: learned dynamics preserve molecular geometry more consistently than linear interpolation, as reflected by the bond-length variation.}
  \label{fig:md_panels}
\end{figure}

\subsection{Plausibility of Intermediates}
\label{sec:plausibility}
\looseness=-1 Sec.~\ref{sec:toy} already tests whether generation follows the data manifold in two dimensions, and we now ask the same on PDE data. Gray--Scott reaction--diffusion describes interacting chemical species that form spatial patterns. We observe only one species, so future rollouts depend on both the observed and the hidden initial state.
For evaluation, we generate trajectories ten times finer in time to obtain true unseen states between two training frames.
Each path is scored against these intermediates by relative \(L_2\) distance and a spectral plausibility diagnostic (Appendix~\ref{app:gs_fine}).

\begin{table}[h]
  \begin{minipage}[t]{0.5\linewidth}
    \centering
    \caption{Frozen interpolation paths against the true fine-time Gray--Scott intermediates (mean \(\pm\) std over eight trajectories).}
    \label{tab:gs_interp_fine}
    \footnotesize
    \setlength{\tabcolsep}{3.5pt}
    \begin{tabular}{l cc}
      \toprule
      Path & rel-$L_2$ $\downarrow$ & spectral $\downarrow$ \\
      \midrule
      linear       & 0.055 $\pm$ 0.001 & 0.471 $\pm$ 0.004 \\
      ours (score) & \textbf{0.019 $\pm$ 0.001} & \textbf{0.134 $\pm$ 0.017} \\
      \bottomrule
    \end{tabular}
  \end{minipage}\hfill
  \begin{minipage}[t]{0.46\linewidth}
    \centering
    \caption{Generative fields at twice the training frame spacing (full results in Appendix~\ref{app:res_plausibility}).}
    \label{tab:x1w}
    \footnotesize
    \setlength{\tabcolsep}{3.5pt}
    \begin{tabular}{l cc}
      \toprule
      Supervision & rel-$L_2$ $\downarrow$ & spectral $\downarrow$ \\
      \midrule
      ours (score targets) & 0.115 & \textbf{0.335} \\
      linear targets       & 0.177 & 0.439 \\
      Neural ODE           & \textbf{0.060} & 0.482 \\
      \bottomrule
    \end{tabular}
  \end{minipage}
\end{table}

\looseness=-1 The frozen interpolation prior alone is already 2.9 times closer to the true intermediates than linear interpolation and 3.5 times better on the spectral diagnostic (Table~\ref{tab:gs_interp_fine}).
We then train generative fields on these paths at twice the frame spacing, where the endpoints no longer determine the interior, using the same backbone with score targets, linear targets, and a solver-in-the-loop Neural ODE (Appendix~\ref{app:x1w}). The Neural ODE fits the dynamics of this single trajectory and tracks the true states most closely, yet its intermediates are the least plausible of the three under the spectral diagnostic (Table~\ref{tab:x1w}).

While recovery of true intermediate dynamics cannot be guaranteed from the score prior alone, knowledge such as a governing PDE can be added to the interpolator. 
We demonstrate this on a fully observed variant that stores both species, such that the PDE residual is computable (Appendix~\ref{app:x2}).
Results in Appendix~\ref{app:res_plausibility} show that the score paths provide a physically better baseline than linear interpolation and that they behave better under subsequent refinement based on the PDE.

\looseness=-1 On this Gray--Scott data, the Neural ODE seems to recover the true path while violating its spectral structure. We demonstrate that on molecular dynamics (MD), which models atoms evolving under physical interactions, it fails at both.
Here the intermediate states should stay on the molecular configuration manifold. We use ethanol trajectories from MD17~\citep{chmiela2017machine}, atomic coordinates with fixed atom identities, and an EGNN backbone~\citep{satorras2021en}. We also add prior knowledge to the interpolator training, encouraging smooth segment transitions and uniform Euclidean speed (App.~\ref{data:md}).
Figure~\ref{fig:md_panels} compares fields trained on score-guided paths (ours), linear paths, and a Neural ODE on four consecutive MD frames.
The true trajectory has small-scale jitter that we cannot recover, but our rollout stays close to the true geometry.
While the linear-path field and the Neural ODE fit the data at observed times, they produce impossible intermediate states with bond lengths shrinking to about a third of their true length and large spikes in the C-O-H angle.

\subsection{Generation of PDE-Governed Spatiotemporal Dynamics}
\label{sec:pde_exp}

\begin{table}[t]
\centering
\caption{Distributional rollout comparison on Gray--Scott: sliced Wasserstein distance normalized by the true-vs-true finite-sample reference (values near one indicate simulator-level variability).}
\label{tab:gray_scott_dist}
\footnotesize
\begin{tabular}{lccccc}
\toprule
Method & $T=16$ & $T=32$ & $T=64$ & $T=128$ & $T=256$ \\
\midrule
Ours & 1.35 & 1.23 & 0.98 & 0.75 & 0.75 \\
FNO  & 5.91 & 6.03 & 5.70 & 5.72 & 6.49 \\
\bottomrule
\end{tabular}
\end{table}

\looseness=-1 The two-dimensional example shows plausible stochasticity under partial conditioning (Appendix~\ref{app:res_toy}).
Now, we test this on the partially observed Gray--Scott data of Sec.~\ref{sec:plausibility}, trained on sequences of length \(T=64\), where a starting point does not uniquely determine the future evolution since the second species is unobserved.
We generate a reference distribution of possible trajectories by resampling the unobserved field for each observed field, and compare it against samples from our model obtained by drawing multiple latents for the same observation.
Table~\ref{tab:gray_scott_dist} reports the sliced Wasserstein distance at each horizon \(T\), normalized by a true-vs-true simulator split (details in App.~\ref{data:pde}).
At \(T=64\) the normalized distance is close to one, so our samples sit about as far from the reference set as it sits from itself, and it stays at a comparable level out to four times the training horizon.
Rollouts for two latent samples are given in Appendix~\ref{app:res_pde}, showing distinct yet plausible evolutions from the same starting point.
To confirm that the task is not solved by a single prediction, we also report a deterministic Fourier Neural Operator (FNO)~\citep{li2020fourier} one-step predictor rolled out autoregressively, which produces one future per observation.

\begin{figure}[b]
    \centering
    \includegraphics[width=0.95\linewidth]{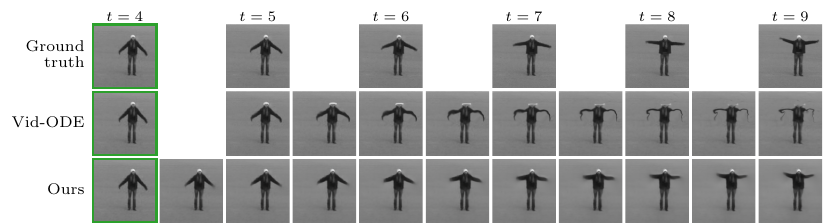}
    \caption{Extrapolation on KTH (rows: ground truth with the last observed input framed in green)}
    \label{fig:kth_extrap}
\end{figure}

\subsection{Natural Video Generation}
\label{sec:temporal_sr}

We evaluate on the KTH Actions benchmark~\citep{schuldt2004recognizing} under the two Vid-ODE protocols~\citep{park2021vidode} at \(128{\times}128\): \emph{interpolation} (observe every other frame of a ten-frame window, predict the held-out frames) and \emph{extrapolation} (observe the first five frames, predict the next five).
Unlike the other experiments, the full pipeline here runs in the latent space of a pretrained frame autoencoder, mirroring Vid-ODE's GRU backbone, which likewise evolves an encoded state.
We retrain Vid-ODE with its official code, evaluating all methods with its protocol and metric implementation (Appendix~\ref{data:kth}).
Figure~\ref{fig:kth_extrap} shows the failure mode that motivates this work. Supervised only at observed timestamps, Vid-ODE drifts off the data manifold with compounding error over the extrapolation horizon. Figure~\ref{fig:kth_interp} (Appendix~\ref{app:res_kth}) shows the same under interpolation, where its warped intermediates ghost at unobserved times.
Tables~\ref{tab:vidode} and~\ref{tab:vidode_extrap} report SSIM, PSNR, and LPIPS, with DVF and UVI numbers from~\citet{park2021vidode}.
Under extrapolation our method surpasses Vid-ODE on all three metrics and remains competitive at interpolation.

\begin{table}[t]
  \begin{minipage}[t]{0.48\linewidth}
    \centering
    \caption{Video interpolation on KTH Actions (DVF and UVI quoted from~\citealp{park2021vidode}).}
    \label{tab:vidode}
    \footnotesize
    \setlength{\tabcolsep}{3.5pt}
    \begin{tabular}{l ccc}
      \toprule
      Method & SSIM $\uparrow$ & PSNR $\uparrow$ & LPIPS $\downarrow$ \\
      \midrule
      DVF (supervised)    & 0.954 & 36.28 & 0.037 \\
      UVI (unsupervised)  & 0.934 & 29.97 & 0.055 \\
      Vid-ODE             & 0.912 & 33.08 & 0.049 \\
      Ours (linear)       & 0.890 & 29.59 & 0.086 \\
      Ours                & 0.904 & 31.86 & 0.090 \\
      \bottomrule
    \end{tabular}
  \end{minipage}\hfill
  \begin{minipage}[t]{0.48\linewidth}
    \centering
    \caption{Video extrapolation on KTH Actions (observe frames 1--5, generate frames 6--10).}
    \label{tab:vidode_extrap}
    \footnotesize
    \setlength{\tabcolsep}{3.5pt}
    \begin{tabular}{l ccc}
      \toprule
      Method & SSIM $\uparrow$ & PSNR $\uparrow$ & LPIPS $\downarrow$ \\
      \midrule
      Vid-ODE             & 0.859          & 29.21          & 0.125          \\
      Ours (linear)       & \textbf{0.869} & 29.10          & 0.103          \\
      Ours                & 0.868          & \textbf{29.46} & \textbf{0.099} \\
      \bottomrule
    \end{tabular}
  \end{minipage}
\end{table}

\section{Related Work}
Neural ODEs and related continuous-time models enable evaluation at arbitrary query times~\citep{chen2018neural,rubanova2019latent,kidger2020neuralcde,li2020latent}, including visual generation in Vid-ODE~\citep{park2021vidode}. These methods motivate continuous-time generative modeling, but supervision typically enters only at observed timestamps. Our work instead uses score-induced interpolation to provide intermediate-time targets, connecting flow-matching regression objectives~\citep{lipman2023flowmatching,liu2023rectifiedflow} with recent score-based data-manifold geometries~\citep{chen2024riemannianflowmatching,saito2025tangential}. Stability regularizers for Neural ODEs usually encourage smoother flows through generic Jacobian, kinetic, or contractivity penalties~\citep{finlay2020train,ghosh2020steer,xu2023robust}. Our robustification is instead path-relative, with the exponential return property of Proposition~\ref{prop:exp_error_decay} and the contraction consequence of Corollary~\ref{cor:standard_transverse_contraction}. A broader overview is given in Appendix~\ref{app:related_work}.

\section{Limitations}
We inherit limitations from score-based models in general, and in particular the ones listed in the data-driven Riemannian geometry literature (e.g. \citet{horvat2024gauge} or \citet{hauberg2018only}).
Beyond these, as discussed in the experiments, we cannot guarantee recovery of the true intermediate dynamics, but we demonstrate plausible behavior that accommodates additional knowledge where it is available.
Our experiments focus on low-resolution settings and moderate training horizons and scaling the method might require further engineering. Relatedly, the probabilistic formulation uses a single latent vector, which is unlikely to capture semantics at multiple levels. Hierarchical or multi-timescale latents, and improved latent identification under local regression objectives, are important directions for future work.

\section{Conclusion}
We introduced a framework for generative modeling of time-dependent data as continuous-time evolution on a learned data manifold. Score-based models as geometric priors define intermediate-time supervision and turn sparse temporal observations into dense local training targets for stochastic vector fields, yielding a simulation-free objective.
We further showed how path-relative transverse corrections provide a principled robustness mechanism with exponential return toward the interpolation path.
Across natural video, PDE-governed dynamics, and molecular trajectories, our experiments illustrate that score-induced geometry can organize continuous-time generation beyond task-specific architectures.

\subsubsection*{Acknowledgments}
This work was supported by funding from the German Federal Ministry for Education and Research
(Bundesministerium f\"ur Bildung und Forschung, BMBF) under grants 16IS24071B and 01IW23005.

\bibliography{iclr2027_conference}
\bibliographystyle{iclr2027_conference}

\appendix

\section*{Overview of the Appendix}
The appendix provides background material, proofs, experimental details, and additional results.
\begin{itemize}
  \item Appendix~\ref{app:related_work} gives an extended discussion of related work.
  \item Appendix~\ref{app:score_manifold_background} reviews the score-induced geometry behind the interpolation prior.
  \item Appendix~\ref{app:proofs} contains the proofs and derivations for all stated results.
  \item Appendix~\ref{app:data_training_details} describes datasets, architectures, and training details for every experiment.
  \item Appendix~\ref{app:additional_results} collects additional results and qualitative comparisons.
\end{itemize}

\section{Background: Extended Related Work}
\label{app:related_work}

In this appendix, we place our method in the broader literature on temporal generative modeling, continuous-time dynamics, score-based geometry, and stability. The main text focuses only on the closest methodological neighbors. Here we provide a wider overview of the surrounding research landscape and situate the components of our approach more explicitly.

\subsection{Discrete-Time Generative Modeling of Temporal Data}
A large body of work models temporal data on a fixed discrete grid. In video generation, this includes recurrent and stochastic predictors that generate future frames autoregressively \citep{denton2018stochastic,villegas2019highfidelity}, as well as more recent diffusion-based and clip-level generators that synthesize multiple frames jointly \citep{ho2022videodiffusion,singer2022makeavideo}. Other works combine diffusion with flexible conditioning or sequential generation over long horizons \citep{harvey2022flexible,liu2025videobiflow}. Despite substantial progress in visual quality and temporal consistency, these methods remain fundamentally tied to the temporal resolution at which data are observed and generated. Intermediate times are typically not modeled as first-class objects, but are instead recovered implicitly through discrete synthesis or simple interpolation.

Our approach departs from this setting by treating temporal generation as a continuous-time dynamics problem rather than repeated synthesis on a prescribed grid. This distinction is central in our setting, since we explicitly care about how the state evolves between observed timestamps and not only about the observations on the training grid.

\subsection{Continuous-Time Sequence and Generative Modeling}
Continuous-time modeling based on learned differential equations provides a natural alternative to fixed-step temporal architectures. Neural ODEs \citep{chen2018neural} introduced the idea of parameterizing dynamics by a continuous vector field and solving for trajectories with an ODE solver. This viewpoint was extended to irregularly sampled sequences by Latent ODEs \citep{rubanova2019latent}, to controlled temporal signals by Neural CDEs \citep{kidger2020neuralcde}, and to generative settings through continuous normalizing flows such as FFJORD \citep{grathwohl2019ffjord} and latent stochastic differential-equation models \citep{li2020latent}. Together, these works establish learned differential equations as a flexible language for modeling time-dependent data in continuous time.

Our method is closely aligned with this line in that it also represents evolution through a continuous-time vector field. The main methodological difference is in how the dynamics are learned: rather than training through explicit solver rollouts and adjoint-based optimization, we use a simulation-free regression objective on local targets induced by score-guided interpolation paths. In this sense, our framework combines the representational advantages of continuous-time models with a training procedure that provides guidance at unobserved intermediates and avoids differentiating through ODE solves.

\subsection{Continuous-Time Video Generation}
Continuous-time generation has been explored explicitly in video through neural ODE-based architectures. Vid-ODE \citep{park2021vidode} uses a latent ODE to model video evolution in continuous time and demonstrates arbitrary-time interpolation and prediction. Related latent differential-equation approaches have also been studied for video generation \citep{gordon2021latentnde}. These models are important predecessors because they show that continuous-time parameterizations can be effective for visual sequence modeling and can naturally support flexible query times.

At the same time, these approaches primarily change the temporal architecture, not the supervision signal between observations. They do not use a learned observation-space prior to guide intermediate states, nor do they explicitly construct manifold-aware interpolation targets between discrete timestamps. Our method addresses precisely this gap by combining continuous-time generation with score-based geometric supervision in observation space---either directly in data space or, as in our KTH experiments, in the latent space of a frozen frame autoencoder.

\subsection{Score-Based Geometry and Manifold-Aware Interpolation}
Recent work has shown that score-based and diffusion models encode more than generative density information: they also capture geometric structure of the underlying data manifold. Riemannian score-based generative modeling and related diffusion formulations connect score functions to manifold geometry and diffusion processes on non-Euclidean spaces \citep{debortoli2022riemannian,huang2022riemannian}. Subsequent work showed that diffusion models encode intrinsic manifold dimension \citep{stanczuk2024intrinsic}, and that the Jacobian of the score can be used to construct a Riemannian metric and tangent-normal decomposition in data space \citep{saito2025tangential}. These results provide a mathematical basis for treating score models as learned geometric priors.

A particularly relevant line of work studies interpolation under this induced geometry. Metric Flow Matching \citep{kapusniak2024metricflowmatching} amortizes geodesic-like interpolation by learning a neural correction to linear paths under a Riemannian metric. Our method directly builds on this perspective. Concretely, we combine the score-based Riemannian metric introduced by \citet{saito2025tangential}, which has been shown to produce high-quality manifold-aware image interpolations, with the amortized neural interpolator of \citet{kapusniak2024metricflowmatching}. The resulting interpolation prior is then used not as an end in itself, but as supervision for learning conditional continuous-time generative dynamics.

\subsection{Flow Matching and Simulation-Free Generative Flows}
Flow Matching \citep{lipman2023flowmatching} replaces solver-based generative training by a regression objective on local vector-field targets defined along an interpolation path between source and target distributions. Rectified Flow \citep{liu2023rectifiedflow} further emphasizes straightening transport trajectories for efficient generation, while Flow Matching on General Geometries \citep{chen2024riemannianflowmatching} extends these ideas beyond Euclidean space. These methods show that powerful generative flows can be learned from local regression targets without backpropagating through an ODE solver.

Our method inherits this simulation-free viewpoint, but uses it in a different role. Rather than learning a transport map between static source and target distributions, we learn latent-conditioned temporal dynamics from local targets induced by score-based interpolation priors between observed states. In this sense, our work combines the regression-based efficiency of flow-matching-style training with a temporal generative setting and a learned geometric prior for intermediate evolution.

\subsection{Stochastic Temporal Prediction and Multimodal Futures}
Many temporal prediction problems are inherently multimodal, and stochastic predictors have long been used to address this. In video prediction, learned-prior latent-variable models such as SVG \citep{denton2018stochastic} represent uncertainty over future frames through latent stochastic variables. In continuous time, uncertainty can be modeled through latent ODEs with probabilistic latent states \citep{rubanova2019latent} or through latent SDEs \citep{li2020latent}, which capture stochastic latent evolution more directly.

These approaches are closely related to our probabilistic formulation in that they recognize the need to model distributions over possible futures rather than single deterministic trajectories. However, they do not combine this stochasticity with manifold-aware supervision of intermediate observation-space states. Our method fills this gap by coupling latent stochasticity to a score-guided continuous-time vector field, so that multimodal future evolution is modeled together with a geometry-aware prior on intermediate states.

\subsection{Stability, Contraction, and Robust Rollouts}
A separate line of work studies robustness and numerical stability in learned continuous-time systems. For Neural ODEs, Jacobian and kinetic regularization have been used to control stiffness and improve numerical behavior \citep{finlay2020train}, while temporal regularization schemes such as STEER simplify the learned dynamics and reduce solver complexity \citep{ghosh2020steer}. More recently, contractivity-promoting penalties have been proposed to improve robustness to perturbations by encouraging local contraction of the learned flow \citep{xu2023robust}.

At a more theoretical level, contraction analysis provides a control-theoretic language for reasoning about incremental stability of nonlinear dynamical systems \citep{slotine1998contraction}, while control contraction metrics extend these ideas toward trajectory tracking and nonlinear control design \citep{manchester2017control}. Our robustness objective is most closely related to this family of ideas. However, unlike prior regularizers that primarily constrain infinitesimal flow properties indirectly, our method supervises the desired corrective direction relative to a score-guided interpolation path. This yields a path-relative notion of stability that is tailored to the learned manifold geometry rather than imposed purely as a generic smoothness or contraction prior.

\subsection{Scientific Dynamical Data}
Continuous-time and generative models are increasingly relevant in scientific settings, where data often arise from dynamical systems and intermediate states are physically or semantically meaningful. For PDE-governed fields, operator-learning methods such as the Fourier Neural Operator \citep{li2020fourier} provide powerful predictive baselines, while generative transport methods such as Flow Matching \citep{lipman2023flowmatching} offer a general framework for learning dynamics through local vector fields. In molecular systems, score-based generative approaches have been used to model molecular dynamics and conformation generation~\citep{wu2023diffmd}.

These application areas motivate our experiments, but they are not the primary source of our methodological novelty. Our contribution is not a domain-specific scientific model, but rather a general framework for continuous-time generative dynamics. The scientific experiments therefore serve to highlight a regime where this formulation is especially meaningful, namely when the plausibility of intermediate states might matter as much as endpoint accuracy.

\section{Background: Score-induced Manifold}
\label{app:score_manifold_background}

This section summarizes the geometric background used to construct the interpolation prior.
The key idea is to use a pretrained score-based generative model as a source of local manifold geometry.
The score model provides a vector field whose Jacobian encodes tangent and normal directions of the learned data manifold. This induces a Riemannian metric under which geodesics preferentially move tangentially to the manifold~\citep{saito2025tangential}.
Following the metric-flow-matching perspective~\citep{kapusniak2024metricflowmatching}, we avoid solving geodesic boundary-value problems directly and instead train an amortized interpolator by minimizing the corresponding kinetic energy.

\subsection{Score-based Generative Models}

We briefly recall the flow-matching view of score-based generative modeling and fix the notation used to define the score-induced metric.
Throughout this appendix, we use \(r\in[0,1]\) for the denoising time of a pretrained generative model, reserving \(t\) for physical time in the temporal data and \(\tau\) for interpolation paths.
The convention is that \(r=0\) corresponds to a simple noise distribution and \(r=1\) corresponds to clean data.

Let \(p_0\) be a simple base distribution on \(\mathbb R^d\), typically \(\mathcal N(0,I)\), and let \(p_1=p_{\mathrm{data}}\).
Flow matching \citep{lipman2023flowmatching} constructs a probability path \((p_r)_{r\in[0,1]}\) connecting \(p_0\) to \(p_1\), together with a velocity field \(u_r:\mathbb R^d\to\mathbb R^d\) satisfying the continuity equation
\[
\partial_r p_r(x)
+
\nabla_x\cdot\bigl(p_r(x)u_r(x)\bigr)
=
0.
\]
A model \(u_\psi(x,r)\) is trained by regressing to conditional velocities along a prescribed conditional path.
In general, one may choose a Gaussian conditional path
\[
x_r=\alpha_r x_1+\beta_r x_0,
\qquad
x_0\sim p_0,\quad x_1\sim p_1,
\]
with differentiable schedules \(\alpha_r,\beta_r\).
In the remainder of this section, we use the linear path
\[
x_r
=
r x_1+(1-r)x_0,
\qquad
x_0\sim\mathcal N(0,I),
\]
for which the conditional velocity is simply
\[
\dot x_r
=
x_1-x_0.
\]
The flow-matching objective is therefore
\[
\mathcal L_{\mathrm{FM}}(\psi)
=
\mathbb E_{r,x_0,x_1}
\left[
\left\|
u_\psi(x_r,r)-(x_1-x_0)
\right\|_2^2
\right],
\]
and the population minimizer satisfies
\[
u_r(x)
=
\mathbb E[x_1-x_0\mid x_r=x].
\]

For this linear Gaussian path, the marginal velocity field determines the score of \(p_r\).
Since
\[
x_r
=
r x_1+(1-r)x_0,
\qquad
x_0\sim\mathcal N(0,I),
\]
the conditional distribution of \(x_r\) given \(x_1\) is
\[
p_r(x\mid x_1)
=
\mathcal N\bigl(x;rx_1,(1-r)^2I\bigr).
\]
Its conditional score is
\[
\nabla_x\log p_r(x\mid x_1)
=
-\frac{x-rx_1}{(1-r)^2}
=
-\frac{x_0}{1-r}.
\]
By Fisher's identity,
\[
s_r(x)
:=
\nabla_x\log p_r(x)
=
\mathbb E\!\left[
\nabla_x\log p_r(x\mid x_1)
\mid x_r=x
\right]
=
-\frac{1}{1-r}
\mathbb E[x_0\mid x_r=x].
\]
On the other hand,
\[
u_r(x)
=
\mathbb E[x_1-x_0\mid x_r=x].
\]
Using \(x=rx_1+(1-r)x_0\), we have
\[
x_1-x_0
=
\frac{x-x_0}{r}-x_0
=
\frac{x}{r}-\frac{x_0}{r},
\]
and therefore
\[
u_r(x)
=
\frac{x}{r}
-
\frac{1}{r}\mathbb E[x_0\mid x_r=x].
\]
Substituting
\[
\mathbb E[x_0\mid x_r=x]
=
-(1-r)s_r(x)
\]
gives
\[
u_r(x)
=
\frac{x}{r}
+
\frac{1-r}{r}s_r(x).
\]
Equivalently, for \(r\in(0,1)\),
\[
s_r(x)
=
\frac{r}{1-r}u_r(x)
-
\frac{1}{1-r}x.
\]
Thus the score can be recovered from the flow-matching vector field by a known time-dependent affine transformation:
\[
s_\psi(x,r)
:=
\frac{r}{1-r}u_\psi(x,r)
-
\frac{1}{1-r}x.
\]

This relation is the linear-path analogue of the usual score identities in diffusion models.
For example, if a diffusion model uses the variance-preserving noising process \citep{song2021scorebasedsde}
\[
x_r=\alpha_r x_1+\beta_r\epsilon,
\qquad
\epsilon\sim\mathcal N(0,I),
\]
then
\[
s_r(x)
=
-\frac{1}{\beta_r}\mathbb E[\epsilon\mid x_r=x].
\]
Thus a noise-prediction model \(\epsilon_\psi(x,r)\) defines the score estimate
\[
s_\psi(x,r)
=
-\frac{1}{\beta_r}\epsilon_\psi(x,r),
\]
while a denoising model \(\hat x_{1,\psi}(x,r)\) defines
\[
s_\psi(x,r)
=
\frac{\alpha_r\hat x_{1,\psi}(x,r)-x}{\beta_r^2}.
\]
In all cases, the pretrained generative model provides a score estimate \(s_\psi(x,r)\) and hence a score Jacobian
\[
J_\psi(x,r)
:=
D_x s_\psi(x,r).
\]

In this paper, the pretrained model is used only as a geometric prior.
The score-induced geometry below is defined through \(J_\psi(x,r)\), not through samples from the pretrained model itself.
A technical subtlety is that the flow-matching transport is trained from \(r=0\) to \(r=1\), where \(r=1\) corresponds to clean data.
For the linear path, the score identity contains the factor \((1-r)^{-1}\), and common diffusion parameterizations have the analogous vanishing-noise singularity at the clean endpoint.
Thus the score and its Jacobian are not numerically reliable exactly at \(r=1\).

Whenever we need score-based geometric information for a clean state \(x\), we first move it to a denoising time
\[
r_\ell=1-\delta,
\qquad
0<\delta<1.
\]
This lifting is defined by the inverse flow map of the pretrained model:
\[
\operatorname{Lift}_{r_\ell}(x)
:=
\Phi^\psi_{1\to r_\ell}(x),
\]
where \(\Phi^\psi_{a\to b}\) denotes the flow map of
\[
\frac{d}{dr}X_r=u_\psi(X_r,r).
\]
We then compute scores and score Jacobians at the lifted point
\[
\bar x
=
\operatorname{Lift}_{r_\ell}(x),
\qquad
s_\psi(\bar x,r_\ell),
\qquad
J_\psi(\bar x,r_\ell).
\]
After geometric operations have been performed at denoising time \(r_\ell\), we map back to clean data space by applying the forward flow
\[
\operatorname{Denoise}_{r_\ell\to1}(\bar x)
:=
\Phi^\psi_{r_\ell\to1}(\bar x).
\]
By construction, lifting and denoising are inverse operations along the learned flow:
\[
\operatorname{Denoise}_{r_\ell\to1}
\bigl(
\operatorname{Lift}_{r_\ell}(x)
\bigr)
=
\Phi^\psi_{r_\ell\to1}\!\circ\Phi^\psi_{1\to r_\ell}(x)
=
x,
\]
up to numerical integration error. Thus score-induced geometric quantities are evaluated on a slightly noisy point on the same learned flow trajectory, where the score model is well-defined and numerically stable. Note as well that lifting and denoising operations are differentiable, since the underlying score-based generative model is typically differentiable.

\subsection{Riemannian Geometry}

A Riemannian metric on an open set \(U\subseteq\mathbb R^d\) assigns to each \(x\in U\) a symmetric positive definite matrix
\[
G(x)\in\mathbb R^{d\times d}.
\]
It defines a position-dependent inner product
\[
\langle u,v\rangle_{G(x)}
:=
u^\top G(x)v,
\qquad
u,v\in T_xU\simeq\mathbb R^d,
\]
and corresponding norm
\[
\|u\|_{G(x)}^2 := u^\top G(x)u.
\]

For a smooth curve \(\gamma:[0,1]\to U\), its Riemannian length and energy are
\[
\mathrm{Len}_G(\gamma)
=
\int_0^1
\|\dot\gamma(t)\|_{G(\gamma(t))}\,dt,
\]
and
\[
\mathcal E_G(\gamma)
=
\frac12
\int_0^1
\|\dot\gamma(t)\|_{G(\gamma(t))}^2\,dt
=
\frac12
\int_0^1
\dot\gamma(t)^\top G(\gamma(t))\dot\gamma(t)\,dt.
\]
A geodesic between endpoints \(x_0,x_1\in U\) is a critical point of this energy under the boundary constraints
\[
\gamma(0)=x_0,
\qquad
\gamma(1)=x_1.
\]
Constant-speed length-minimizing geodesics are minimizers of \(\mathcal E_G\).

\subsection{Score-induced Riemannian Metric}

Following~\citet{saito2025tangential}, we define a metric from the Jacobian of the score field at a fixed denoising time.
Let \(r_m\in(0,1)\) denote the denoising time at which the metric is evaluated.
Using the score estimate \(s_\psi(\cdot,r_m)\) from the pretrained generative model, define
\[
J_\psi(x,r_m)
:=
D_x s_\psi(x,r_m).
\]
The score-induced metric is
\[
G_\psi(x,r_m)
:=
J_\psi(x,r_m)^\top J_\psi(x,r_m).
\]
For a tangent vector \(u\in\mathbb R^d\), the induced quadratic form is
\[
\|u\|_{G_\psi(x,r_m)}^2
=
u^\top G_\psi(x,r_m)u
=
\|J_\psi(x,r_m)u\|_2^2.
\]

Since \(G_\psi(x,r_m)\) is positive semidefinite by construction, it need not be strictly positive definite.
If a nondegenerate Riemannian metric is required, one may use the Tikhonov-regularized metric
\[
G_{\psi,\varepsilon}(x,r_m)
:=
G_\psi(x,r_m)+\varepsilon I,
\qquad
\varepsilon>0.
\]
Following~\citet{saito2025tangential}, we found this regularization to have little practical effect and therefore use the unregularized score-induced metric in our experiments.

The geometric interpretation is that the score varies slowly along directions tangent to the learned data manifold and rapidly along directions normal to it.
Thus, directions \(u\) for which \(\|J_\psi(x,r_m)u\|_2\) is small are treated as approximately tangent, while directions with large \(\|J_\psi(x,r_m)u\|_2\) are treated as approximately normal.
The metric therefore assigns low cost to tangent motion and high cost to normal motion, encouraging paths to remain tangent to the score-induced data manifold.

Since clean data points lie at \(r=1\), whereas the metric is evaluated at \(r_m<1\), we first lift a clean point \(x\) to denoising time \(r_m\) using the learned inverse flow,
\[
\bar x
=
\operatorname{Lift}_{r_m}(x)
=
\Phi^\psi_{1\to r_m}(x).
\]
All geometric quantities associated with a clean state \(x\) are therefore evaluated at its lifted counterpart~\(\bar x\):
\[
G_\psi^{\mathrm{clean}}(x)
:=
G_\psi(\bar x,r_m)
=
G_\psi\!\left(\Phi^\psi_{1\to r_m}(x),r_m\right).
\]
The quantity \(G_\psi^{\mathrm{clean}}(x)=G_\psi(\operatorname{Lift}_{r_m}(x),r_m)\) is therefore not itself a metric tensor on clean data space. A metric on clean data space could instead be induced by pulling \(G_\psi\) back through the lifting map.

\subsection{Geodesics on the Score Manifold}

Given two clean endpoint states \(x_0,x_1\in\mathbb R^d\), we define their score-induced interpolation by performing the geometric construction at the fixed denoising time \(r_m<1\).
We therefore first lift the endpoints using the pretrained flow,
\[
\bar x_0=\operatorname{Lift}_{r_m}(x_0),
\qquad
\bar x_1=\operatorname{Lift}_{r_m}(x_1),
\]
and perform the geometric computation in the lifted space.

For a lifted path \((\bar\gamma_t)_{t\in[0,1]}\), define the pointwise metric energy density
\[
e_G(\bar\gamma_t,\dot{\bar\gamma}_t)
:=
\frac12
\|\dot{\bar\gamma}_t\|_{G(\bar\gamma_t)}^2
=
\frac12
\|J_{\bar\gamma_t}\dot{\bar\gamma}_t\|_2^2,
\]
where \(J_{\bar\gamma_t}=D_xs_\psi(x,r_m)|_{x=\bar\gamma_t}\).
The corresponding metric energy is
\[
\mathcal E_G(\bar\gamma)
:=
\int_0^1 e_G(\bar\gamma_t,\dot{\bar\gamma}_t)\,dt
=
\frac12
\int_0^1
\|\dot{\bar\gamma}_t\|_{G(\bar\gamma_t)}^2
\,dt.
\]
Equivalently, using \(G(\bar\gamma_t)=J_{\bar\gamma_t}^{\top}J_{\bar\gamma_t}\),
\[
\mathcal E_G(\bar\gamma)
=
\frac12
\int_0^1
\|J_{\bar\gamma_t}\dot{\bar\gamma}_t\|_2^2
\,dt.
\]

The score-induced geodesic between the lifted endpoints \(\bar x_0\) and \(\bar x_1\) is the path
\[
\bar\gamma^\star
\in
\underset{\bar\gamma}{\mathrm{argmin}}\
\mathcal E_G(\bar\gamma)
\quad
\text{subject to}
\quad
\bar\gamma_0=\bar x_0,\quad
\bar\gamma_1=\bar x_1.
\]
This energy formulation is the squared counterpart of the metric length objective in the main text.
For constant-speed minimizers, minimizing length and minimizing energy yield the same geodesic path, up to parameterization.

Since normal motion is expensive under \(G\), minimizing \(\mathcal E_G\) encourages \(\dot{\bar\gamma}_t\) to remain tangent to the score-induced data manifold.
Thus the induced interpolation tends to move along manifold directions rather than along the Euclidean straight line between endpoints.
Solving this boundary-value problem directly for every endpoint pair would be expensive, so we instead train an amortized interpolator that approximates low-energy score-induced geodesics and can later be reused as an interpolation prior.

\subsection{Training of the Interpolator}
\label{app:interp_training}

Following the amortized-geodesic idea of Metric Flow Matching~\citep{kapusniak2024metricflowmatching}, we learn a neural bridging correction \(\varphi_\omega\) that augments linear interpolation in lifted space.
Given clean endpoints \(x_0,x_1\), define
\[
\bar x_0=\operatorname{Lift}_{r_m}(x_0),
\qquad
\bar x_1=\operatorname{Lift}_{r_m}(x_1).
\]
The lifted interpolation path is parameterized as
\[
\bar\gamma_t^\omega(x_0,x_1;\alpha)
=
(1-t)\bar x_0+t\bar x_1
+
\alpha\,t(1-t)\,
\varphi_\omega(\bar x_0,\bar x_1,t),
\qquad
t\in[0,1],
\]
where \(\alpha\in[0,1]\) controls the strength of the learned correction.
The factor \(t(1-t)\) enforces the endpoint constraints
\[
\bar\gamma_0^\omega=\bar x_0,
\qquad
\bar\gamma_1^\omega=\bar x_1.
\]
When \(\alpha=0\), the path is lifted linear interpolation.
During training, we gradually ramp \(\alpha\) from \(0\) to \(1\), so optimization begins from the linear baseline and progressively introduces the learned score-geometric correction.

For endpoint pairs \((x_0,x_1)\sim\mu_{\mathrm{pair}}\), we train \(\varphi_\omega\) by minimizing the expected pointwise metric energy at randomly sampled interpolation times:
\[
\mathcal L_{\mathrm{geo}}(\omega)
=
\mathbb E_{(x_0,x_1)\sim\mu_{\mathrm{pair}}}
\mathbb E_{t\sim\mathcal U[0,1]}
\left[
e_G
\left(
\bar\gamma_t^\omega,
\dot{\bar\gamma}_t^\omega
\right)
\right].
\]
Equivalently,
\[
\mathcal L_{\mathrm{geo}}(\omega)
=
\mathbb E_{(x_0,x_1),t}
\left[
\frac12
\left\|
J_{\bar\gamma_t^\omega}
\dot{\bar\gamma}_t^\omega
\right\|_2^2
\right],
\]
where
\[
J_{\bar\gamma_t^\omega}
=
D_xs_\psi(x,r_m)|_{x=\bar\gamma_t^\omega}.
\]
The tangent \(\dot{\bar\gamma}_t^\omega=\partial_t\bar\gamma_t^\omega\) and the metric-vector product \(J_{\bar\gamma_t^\omega}\dot{\bar\gamma}_t^\omega\) are computed using automatic differentiation and Jacobian-vector products, avoiding explicit construction of the score Jacobian.

\begin{algorithm}[t]
\caption{Training the score-induced interpolator}
\label{alg:interpolator_training}
\begin{algorithmic}[1]
\Require Frozen pretrained flow/score model \(u_\psi,s_\psi\); endpoint-pair distribution \(\mu_{\mathrm{pair}}\); metric time \(r_m<1\); ramp schedule \(\alpha_n\in[0,1]\)
\For{\(n=1,\ldots,N\)}
    \State Sample endpoint pairs \((x_0,x_1)\sim\mu_{\mathrm{pair}}\)
    \State Lift endpoints: \(\bar x_0\gets \Phi^\psi_{1\to r_m}(x_0)\), \(\bar x_1\gets \Phi^\psi_{1\to r_m}(x_1)\)
    \State Sample interpolation times \(t\sim\mathcal U[0,1]\)
    \State Set \(\alpha\gets\alpha_n\)
    \State Evaluate the lifted path
    \[
    \bar\gamma_t^\omega
    \gets
    (1-t)\bar x_0+t\bar x_1
    +
    \alpha t(1-t)\varphi_\omega(\bar x_0,\bar x_1,t)
    \]
    \State Compute the lifted tangent \(\dot{\bar\gamma}_t^\omega=\partial_t\bar\gamma_t^\omega\) by JVP
    \State Compute the score-Jacobian product \(J_{\bar\gamma_t^\omega}\dot{\bar\gamma}_t^\omega\) by JVP through \(s_\psi(\cdot,r_m)\)
    \State Minimize \(\frac12\|J_{\bar\gamma_t^\omega}\dot{\bar\gamma}_t^\omega\|_2^2\) with respect to \(\omega\)
\EndFor
\end{algorithmic}
\end{algorithm}

After training, we fix \(\alpha=1\) and freeze the interpolator.
For clean endpoints \(x_0,x_1\), the lifted interpolation path is
\[
\bar\gamma_t^\omega(x_0,x_1)
=
(1-t)\bar x_0+t\bar x_1
+
t(1-t)\varphi_\omega(\bar x_0,\bar x_1,t).
\]
The corresponding clean interpolation path is obtained by denoising each lifted point,
\[
\gamma_t^\omega(x_0,x_1)
=
\operatorname{Denoise}_{r_m\to1}
\bigl(
\bar\gamma_t^\omega(x_0,x_1)
\bigr)
=
\Phi^\psi_{r_m\to1}
\bigl(
\bar\gamma_t^\omega(x_0,x_1)
\bigr).
\]

The tangent target used in the main objective is the derivative of the clean path,
\[
\dot\gamma_t^\omega
=
D\Phi^\psi_{r_m\to1}
\bigl(
\bar\gamma_t^\omega
\bigr)
\dot{\bar\gamma}_t^\omega.
\]
We compute this derivative by a Jacobian-vector product through the denoising flow.
This requires differentiating through the interpolator and through the lifting/denoising operations, which are differentiable because they are defined by the pretrained neural flow.

For finite-step targets, we evaluate the clean path at two nearby interpolation times and form the secant
\[
\Delta_h\gamma_t^\omega
=
\frac{
\gamma_{t+h}^\omega-\gamma_t^\omega
}{h},
\qquad h>0,
\]
with tangent limit
\[
\Delta_0\gamma_t^\omega
=
\dot\gamma_t^\omega.
\]
If the original physical interval has length \(\Delta t_k=t_{k+1}-t_k\), normalized and physical velocities differ by the factor \((\Delta t_k)^{-1}\).
In the main text, this rescaling is absorbed into the definition of the local target \(\Delta_h X_t\).

\subsection{Uniform Metric Speed and Phase of the Interpolator}
\label{app:uniform_speed}

This subsection states and proves the geometry--phase factorization of the interpolation objective asserted in the main text. We work with a Riemannian metric \(G\) as in the preceding subsections. The results apply verbatim to the lifted interpolation family of Appendix~\ref{app:interp_training} with the score-induced metric evaluated along the lifted path, and we suppress the lift notation for readability.

Let \(\gamma:[0,1]\to\mathbb R^d\) be piecewise \(C^1\) and regular, i.e. \(\dot\gamma_t\neq0\) for almost every \(t\), and recall the metric length \(\mathrm{Len}_G(\gamma)\) and the metric energy \(\mathcal E_G(\gamma)\) defined above. Define the arc-length function and the \emph{phase} of \(\gamma\) by
\[
s_G(t)
:=
\int_0^t\|\dot\gamma_u\|_{G(\gamma_u)}\,du,
\qquad
\beta_\gamma(t)
:=
\frac{s_G(t)}{\mathrm{Len}_G(\gamma)}
\in[0,1].
\]
The phase is absolutely continuous and nondecreasing with \(\beta_\gamma(0)=0\) and \(\beta_\gamma(1)=1\), and
\[
\gamma=\tilde\gamma\circ\beta_\gamma,
\]
where \(\tilde\gamma\) denotes the constant-speed reparameterization of the same geometric image. We say that \(\gamma\) has \emph{uniform metric speed} if
\[
s_G(t)=t\,\mathrm{Len}_G(\gamma)
\qquad
\text{for all }t\in[0,1],
\]
i.e. if equal parameter increments sweep out equal metric arc length. Under the score-induced metric, uniform metric speed means uniform progress as measured by the learned data-manifold geometry. The Euclidean speed of the same path, however, is in general non-uniform and slows precisely where the metric assigns high cost.

\begin{proposition}[The energy objective factorizes into geometry and phase]
\label{prop:uniform_speed}
For every piecewise-\(C^1\) regular path \(\gamma\),
\begin{equation}
\label{eq:energy_phase_factorization}
\mathcal E_G(\gamma)
=
\frac12\,\mathrm{Len}_G(\gamma)^2
\int_0^1\dot\beta_\gamma(t)^2\,dt,
\qquad
\int_0^1\dot\beta_\gamma(t)^2\,dt\ge1,
\end{equation}
with equality in the second relation if and only if \(\|\dot\gamma_t\|_{G(\gamma_t)}=\mathrm{Len}_G(\gamma)\) for almost every \(t\), i.e. if and only if \(\gamma\) has uniform metric speed. Consequently, \(\mathcal E_G(\gamma)\ge\frac12\mathrm{Len}_G(\gamma)^2\), and over any family of paths closed under monotone reparameterization, minimizing the energy is equivalent to minimizing the length and selecting the uniform-speed parameterization of a minimizing image. Moreover, if \(\gamma\) has uniform metric speed with \(L:=\mathrm{Len}_G(\gamma)\), then for every time-warp \(\beta\in C^1([0,1])\) with \(\beta(0)=0\), \(\beta(1)=1\), and \(\dot\beta>0\), the reparameterized path \(\gamma\circ\beta\) has the same image, the same length \(L\), metric speed \(\dot\beta(t)L\), and energy \(\mathcal E_G(\gamma\circ\beta)=\frac12L^2\int_0^1\dot\beta(t)^2\,dt\). Conversely, every continuous speed profile \(\varsigma:[0,1]\to\mathbb R_{>0}\) with \(\int_0^1\varsigma(t)\,dt=L\) can be realized by a time-warp \(\beta(t)=\frac1L\int_0^t\varsigma(u)\,du\).
\end{proposition}

\begin{proof}
By the definition of the arc-length function, \(\dot s_G(t)=\|\dot\gamma_t\|_{G(\gamma_t)}\) for almost every \(t\), and hence
\[
\dot\beta_\gamma(t)
=
\frac{\|\dot\gamma_t\|_{G(\gamma_t)}}{\mathrm{Len}_G(\gamma)}
\qquad
\text{for almost every }t\in[0,1].
\]
Substituting this identity into the definition of the energy gives
\[
\mathcal E_G(\gamma)
=
\frac12\int_0^1\|\dot\gamma_t\|_{G(\gamma_t)}^2\,dt
=
\frac12\,\mathrm{Len}_G(\gamma)^2\int_0^1\dot\beta_\gamma(t)^2\,dt,
\]
which is the identity in \eqref{eq:energy_phase_factorization}.

Since \(\beta_\gamma(1)-\beta_\gamma(0)=1\), the phase rate \(\dot\beta_\gamma\) has mean one on \([0,1]\), and expanding its squared deviation from that mean gives
\[
\int_0^1\bigl(\dot\beta_\gamma(t)-1\bigr)^2dt
=
\int_0^1\dot\beta_\gamma(t)^2\,dt
-2\underbrace{\int_0^1\dot\beta_\gamma(t)\,dt}_{=1}
+1
=
\int_0^1\dot\beta_\gamma(t)^2\,dt-1 .
\]
Since the left-hand side integrates a nonnegative function, this implies \(\int_0^1\dot\beta_\gamma(t)^2\,dt\ge1\).
Therefore, equality \(\int_0^1\dot\beta_\gamma(t)^2\,dt=1\) forces \((\dot\beta_\gamma-1)^2=0\), hence \(\dot\beta_\gamma=1\), almost everywhere.

Intuitively, the normalization means the path must cover its entire metric length in unit time no matter how it is parameterized, so every stretch traversed faster than average is offset by one traversed slower. Squaring charges more for the fast stretch than the slow one refunds, and only the perfectly even rate \(\dot\beta_\gamma\equiv1\) avoids the surcharge.

By the identity for \(\dot\beta_\gamma\) above, this holds if and only if \(\|\dot\gamma_t\|_{G(\gamma_t)}=\mathrm{Len}_G(\gamma)\) for almost every \(t\), which is uniform metric speed. Combining both relations in \eqref{eq:energy_phase_factorization} gives the energy lower bound \(\mathcal E_G(\gamma)\ge\frac12\mathrm{Len}_G(\gamma)^2\).

The length is invariant under monotone reparameterization by the substitution rule, while the phase factor \(\int_0^1\dot\beta_\gamma(t)^2\,dt\) attains its minimum value \(1\) exactly at uniform speed. Hence, within any family of paths closed under monotone reparameterization, minimizing \(\mathcal E_G\) minimizes \(\mathrm{Len}_G\) over the geometric images realized by the family and selects the uniform-speed parameterization of a minimizing image.

It remains to prove the statements about time-warps. Let \(\gamma\) have uniform metric speed and let \(\beta\) be an admissible time-warp. Since \(\beta\) is a bijection of \([0,1]\), the path \(\gamma\circ\beta\) has the same image as \(\gamma\). By the chain rule and uniform speed,
\[
\Bigl\|\frac{d}{dt}(\gamma\circ\beta)_t\Bigr\|_{G((\gamma\circ\beta)_t)}
=
\dot\beta(t)\,\|\dot\gamma_{\beta(t)}\|_{G(\gamma_{\beta(t)})}
=
\dot\beta(t)\,L .
\]
Integrating gives
\[
\mathrm{Len}_G(\gamma\circ\beta)
=
\int_0^1\dot\beta(t)\,L\,dt
=
L,
\qquad
\mathcal E_G(\gamma\circ\beta)
=
\frac12\int_0^1\dot\beta(t)^2L^2\,dt .
\]
Finally, fix a continuous speed profile \(\varsigma>0\) with \(\int_0^1\varsigma(t)\,dt=L\). The requirement \(\dot\beta(t)\,L=\varsigma(t)\) together with \(\beta(0)=0\) determines
\[
\beta(t)
=
\frac1L\int_0^t\varsigma(u)\,du
\]
uniquely. This \(\beta\) is \(C^1\) with \(\dot\beta=\varsigma/L>0\) and satisfies \(\beta(1)=\frac1L\int_0^1\varsigma(u)\,du=1\), hence it is an admissible time-warp, and it is the unique one realizing the profile \(\varsigma\).
\end{proof}

The interpolation objective of Appendix~\ref{app:interp_training} is the expectation of \(\mathcal E_G\) over endpoint pairs, since sampling \(t\sim\mathcal U[0,1]\) gives an unbiased single-sample estimator of the energy integral. Proposition~\ref{prop:uniform_speed} therefore applies path-wise: training drives each interpolation toward a shortest admissible path under \(G\) \emph{and}, independently, toward the uniform-speed traversal of that path. Because the correction \(t(1-t)\varphi_\omega(\bar x_0,\bar x_1,t)\) depends explicitly on \(t\), monotone reparameterizations of a fixed admissible image are realizable within the interpolation family, so this selection is effected by the objective rather than assumed. For the flat metric \(G\equiv I\), the unique minimizer with fixed endpoints is the straight linear interpolation traversed at constant Euclidean speed. The trained interpolator thus strictly generalizes the default of standard transport-based objectives (linear interpolation at constant-speed) and reduces to it exactly when the metric is uninformative.

When no temporal side-information is available, we argue that uniform metric speed is a neutral choice: it is the unique parameterization determined by the geometry alone, and any deviation from it encodes temporal knowledge that the two endpoint states do not contain. At the same time, the last part of Proposition~\ref{prop:uniform_speed} shows that the temporal profile is a free one-dimensional control, decoupled from the geometry. An explicit monotone, endpoint-pinned family of time-warps parameterized by a scalar function \(\kappa\) is
\[
\beta_\kappa(t)
=
\frac{\int_0^t\exp(\kappa(u))\,du}{\int_0^1\exp(\kappa(u))\,du},
\]
which recovers the identity phase for constant \(\kappa\). Composing a trained interpolation with such a warp changes only the phase factor in \eqref{eq:energy_phase_factorization} and leaves the image, the length, and consequently the tubular neighborhood, the nearest-point projection, and the transverse displacements of Appendix~\ref{app:prop2} unchanged. Re-phasing modifies the ideal field \(v^\star\) only tangentially, so the transverse guarantees of Proposition~\ref{prop:exp_error_decay} and Proposition~\ref{prop:approx_tube} are unaffected. Changing the interpolator's speed, or learning \(\kappa\) from a temporal criterion such as timestamp side-information or a physics-based residual, is therefore possible without perturbing the learned geometry.

\section{Proofs and Derivations}
\label{app:proofs}

\subsection{ELBO and Discretized Likelihood Lower Bound}
\label{app:prop1}

First, we derive the ELBO stated in \eqref{eq:local_elbo}.
As usual, we augment the local evolution model
\[
p(\Delta_h x_t \mid x_{\le t})
=
\int p_\theta(\Delta_h x_t \mid x_t,z)\, p_\vartheta(z\mid x_{\le t})\,dz
\]
with a variational posterior \(q_\phi ( z|x)\) and apply Jensen's inequality:
\begin{align*}
    \log p(\Delta_h x_t \mid x_{\le t}) &=\log \int \frac{q_\phi ( z|x)}{q_\phi ( z|x)}\,p_\theta(\Delta_h x_t \mid x_t,z)\, p_\vartheta(z\mid x_{\le t})\,dz\\
    &\ge \mathbb E_{q_\phi} \bigl[\log p_\theta(\Delta_h x_t \vert x_t,z) + \log p_\vartheta(z\mid x_{\le t}) - \log q_\phi ( z|x)\bigr]\\
    &= \mathbb E_{q_\phi}
\bigl[\log p_\theta(\Delta_h x_t \vert x_t,z)\bigr]
-
\mathrm{KL}\bigl(q_\phi(z\vert x)\,\|\,p_\vartheta(z\vert x_{\le t})\bigr).
\end{align*}

The training objective is formed by averaging over this ELBO, as shown in \eqref{eq:avg_local_elbo}.
Therefore, we will refer to this as
\begin{align*}
    \mathbb E_{t,h}\bigl[\log p_\theta(\Delta_h x_t \vert x_{\le t})\bigr]
    &\ge
    \mathbb E_{t,h,q_\phi}
    \bigl[\log p_\theta(\Delta_h x_t \vert x_t,z)\bigr]
    -
    \mathbb E_t\Bigl[
    \mathrm{KL}\bigl(q_\phi(z\vert x)\,\|\,p_\vartheta(z\vert x_{\le t})\bigr)
    \Bigr]\\
    &=: \mathcal{L}_\text{avg}^\text{ELBO}(x).
\end{align*}

Finally, to show that our loss can be interpreted as a trajectory-level lower bound in the discrete case, we assume the natural joint model described in the main text.
Concretely, we assume the model of a globally shared latent drawn once per trajectory to be given by
\begin{equation}
\label{eq:global_model_app}
p(x_0,\Delta_0,\dots,\Delta_{K-1})
=
p_0(x_0)\int p_\vartheta(z\mid x_0)
\left(\prod_{k=0}^{K-1} p_\theta(\Delta_k\mid x_k,z)\right) dz .
\end{equation}

\begin{propone*}[Discretized likelihood lower bound]
\label{prop:lower_bound_app}
Consider the discrete case with a uniform step \(h>0\), grid \(t_k=t_0+kh\) for
\(k=0,\dots,K\) with \(x_k = x_{t_k}\), and write
\[
x=(x_0,x_1,\dots,x_K),
\qquad
\hat x=(x_0,\Delta_0,\dots,\Delta_{K-1}),
\qquad
\Delta_k := \Delta_h x_{t_k} = \frac{x_{k+1}-x_k}{h}.
\]
Correspondingly, we consider the case where \(t\) is obtained by sampling \({k\sim\mathrm{Unif}\{0,\dots,K-1\}}\) and setting \(t=t_k\).
We assume an initial density \(p_0(x_0)\), the model of \Eqref{eq:global_model_app} for \(\hat x\), and \(x_k\) to be of dimension \(d\). Then, for every choice of \(\theta\), \(\phi\), and \(\vartheta\),
\[
\log p(x)
\ \ge\
\log p_0(x_0)
+
K\,\mathcal{L}_\text{avg}^\text{ELBO}(x)
-
Kd\log h ,
\]
where the additive terms \(\log p_0(x_0)\) and \(-Kd\log h\) do not depend on the
trainable parameters.

\end{propone*}

\begin{proof}
First, we relate \(p(\hat x)\) to \(p(x)\) by change of variables. Per factor, we have to account for the determinant of the Jacobian \(h\,I_d\) due to the bijection \(\Delta_k \mapsto x_{k+1}\):
\begin{align*}
    p(\hat x)
    &=p_0(x_0)\int p_\vartheta(z\mid x_0)
    \left(\prod_{k=0}^{K-1} p_\theta(\Delta_k\mid x_k,z)\right) dz\\
    &= p_0(x_0)\,h^{Kd}\int p_\vartheta(z\mid x_0)
    \left(\prod_{k=0}^{K-1} p_\theta(x_{k+1}\mid x_k,z)\right) dz\\
    &=h^{Kd}\,p(x).
\end{align*}

Secondly, we can obtain the ELBO for \(p(\hat x)\) as
\begin{align*}
    \log p(\hat x) &\ge \log p_0(x_0) + \sum_{k=0}^{K-1} \mathbb E_{q_\phi}
\bigl[\log p_\theta(\Delta_k \vert x_k,z)\bigr] - \mathrm{KL}\bigl(q_\phi(z\vert x)\,\|\,p_\vartheta(z\vert x_0)\bigr)\\
&=\log p_0(x_0) + K\,\mathcal{L}_\text{avg}^\text{ELBO}(x) + \sum_{k=1}^{K-1}\mathrm{KL}\bigl(q_\phi(z\vert x)\,\|\,p_\vartheta(z\vert x_{\le k})\bigr),
\end{align*}
where \(\mathcal{L}_\text{avg}^\text{ELBO}\) reduces to a sum since for the choice of fixed \(h\) and discrete values for \(t\) we have \(\mathbb E_{t,h}[f]=\frac1K\sum_{k=0}^{K-1}f(t_k,h)\).

Finally, combining both statements, we arrive at
\begin{align*}
    \log p(x) &= -Kd\,\log h + \log p(\hat x)\\ &\ge -Kd\,\log h + \log p_0(x_0) + K\,\mathcal{L}_\text{avg}^\text{ELBO}(x) + \underbrace{\sum_{k=1}^{K-1}\mathrm{KL}\bigl(q_\phi(z\vert x)\,\|\,p_\vartheta(z\vert x_{\le k})\bigr)}_{\ge 0}\\
    &\ge -Kd\,\log h + \log p_0(x_0) + K\,\mathcal{L}_\text{avg}^\text{ELBO}(x)
\end{align*}
\end{proof}

\begin{remark}
    Our proposed objective considers \(K-1\) additional KL-divergence terms compared with the canonical ELBO for the full trajectory model. This can be related to Posterior Matching \citep{strauss2022posterior} and allows arbitrary conditioning, such as demonstrated in Experiment~\ref{sec:toy}.
\end{remark}

\subsection{Exponential Decay of Transverse Displacement}
\label{app:prop2}

As in the main text, let \((x_t)_{t\in[0,T]}\) be a \(C^2\) reference trajectory. We assume that it admits a \(C^2\) embedded extension with nonvanishing velocity to an open interval containing \([0,T]\). Let \(\widetilde\Gamma\) denote the image of this smooth embedded extension, and let
\[
\Gamma:=\{x_t:t\in[0,T]\}\subset\widetilde\Gamma
\]
denote the compact reference path segment. All nearest-point projections below are defined with respect to the extended curve \(\widetilde\Gamma\), while the estimates are required only in a tubular neighborhood of the compact segment \(\Gamma\).
Let
\[
\tau:\widetilde\Gamma\to\mathbb R^d
\]
denote the \(C^1\) tangent velocity field along the embedded extension. Since the extension has nonvanishing velocity, the normal projection
\[
\Pi^\perp(\xi)
:=
I-\frac{\tau(\xi)\tau(\xi)^\top}{\|\tau(\xi)\|^2},
\qquad
\xi\in\widetilde\Gamma,
\]
is well defined.

\begin{assumption}[Tubular neighborhood]
\label{ass:tubular_neighborhood}
For sufficiently small \(\rho>0\), there is an open tubular neighborhood \(U_\rho\) of \(\Gamma\) on which the nearest-point projection
\[
\pi:U_\rho\to\widetilde\Gamma
\]
is uniquely defined and \(C^1\). Consequently, each \(x\in U_\rho\) admits the decomposition
\[
x=\pi(x)+e(x),
\qquad
e(x):=x-\pi(x)\in N_{\pi(x)}\widetilde\Gamma
=
\operatorname{Im}\Pi^\perp(\pi(x)).
\]
\end{assumption}

We define the ideal path-relative vector field on \(U_\rho\) by
\[
v^\star(x)
:=
\tau(\pi(x))-\lambda e(x),
\qquad
\lambda>0.
\]

We shall use the following elementary consequences of the tubular-projection geometry.

\begin{lemma}[Tubular projection geometry]
\label{lem:tubular_projection_geometry}
After possibly shrinking the tube radius \(\rho>0\), the nearest-point projection
\[
\pi:U_\rho\to \widetilde\Gamma
\]
is uniquely defined and \(C^1\). Moreover, for every \(x\in U_\rho\) and every \(w\in\mathbb R^d\),
\[
D\pi(x)w\in T_{\pi(x)}\widetilde\Gamma .
\]
For every \(\xi\in\Gamma\),
\[
D\pi(\xi)=P_{T_\xi\widetilde\Gamma},
\]
where \(P_{T_\xi\widetilde\Gamma}\) denotes the Euclidean orthogonal projection onto the tangent line \(T_\xi\widetilde\Gamma\). Equivalently, if
\[
T(\xi):=\frac{\tau(\xi)}{\|\tau(\xi)\|}
\]
is the unit tangent, then
\[
D\pi(\xi)=T(\xi)T(\xi)^\top .
\]
Consequently, if \(n\in N_\xi\widetilde\Gamma\), then
\[
D\pi(\xi)n=0.
\]
\end{lemma}

\begin{proof}
The existence and \(C^1\)-regularity of the nearest-point projection in a sufficiently small tubular neighborhood follows from the regularity theory of metric projections onto smooth embedded submanifolds; see, e.g., \citet{foote1984regularity}.

Since \(\pi\) has image in the embedded curve \(\widetilde\Gamma\), the tangent-map property for \(C^1\) maps gives
\[
D\pi(x)w\in T_{\pi(x)}\widetilde\Gamma,
\qquad
\forall x\in U_\rho,\ \forall w\in\mathbb R^d;
\]
see \citet[Ch.~3, ``The Differential of a Smooth Map'']{lee2013smooth}.
Equivalently, this can be checked directly in local coordinates.
Let \(\gamma:I\to\mathbb R^d\) be a regular local parametrization of \(\widetilde\Gamma\) around \(\pi(x)\), so that locally
\[
\pi(x)=\gamma(s(x))
\]
for some \(C^1\) scalar projection coordinate \(s:U_\rho\to I\).
Then the chain rule gives
\[
D\pi(x)w
=
D\gamma(s(x))\,Ds(x)[w]
=
\gamma'(s(x))\,Ds(x)[w],
\]
which is a scalar multiple of \(\gamma'(s(x))\), and hence lies in \(T_{\pi(x)}\widetilde\Gamma\).

It remains to identify \(D\pi\) on the reference path.
For \(\xi\in\Gamma\), the projection restricts to the identity on \(\widetilde\Gamma\), and therefore its differential is the identity on \(T_\xi\widetilde\Gamma\).

Now let \(n\in N_\xi\widetilde\Gamma\).
Choose a regular local parametrization \(\gamma:I\to\mathbb R^d\) of \(\widetilde\Gamma\) with
\[
\xi=\gamma(\sigma(0)).
\]
For \(s\) sufficiently small, write
\[
\pi(\xi+sn)=\gamma(\sigma(s)).
\]
The nearest-point optimality condition gives
\[
\bigl(\xi+sn-\gamma(\sigma(s))\bigr)^\top \gamma'(\sigma(s))=0 .
\]
Differentiating with respect to \(s\) and evaluating at \(s=0\) yields
\[
\bigl(n-\gamma'(\sigma(0))\sigma'(0)\bigr)^\top\gamma'(\sigma(0))
+
\bigl(\xi-\gamma(\sigma(0))\bigr)^\top\gamma''(\sigma(0))\sigma'(0)
=
0 .
\]
Since \(\xi=\gamma(\sigma(0))\), the second term vanishes.
Since \(n\in N_\xi\widetilde\Gamma\), we also have
\[
n^\top \gamma'(\sigma(0))=0.
\]
Therefore
\[
-\|\gamma'(\sigma(0))\|^2\sigma'(0)=0.
\]
Because \(\gamma\) is regular, \(\gamma'(\sigma(0))\neq 0\), and hence
\[
\sigma'(0)=0.
\]
Consequently,
\[
D\pi(\xi)n
=
\gamma'(\sigma(0))\sigma'(0)
=
0.
\]

Thus
\[
D\pi(\xi)\big|_{T_\xi\widetilde\Gamma}=\mathrm{Id},
\qquad
D\pi(\xi)\big|_{N_\xi\widetilde\Gamma}=0.
\]
Hence \(D\pi(\xi)\) is precisely the Euclidean orthogonal projection onto \(T_\xi\widetilde\Gamma\), namely
\[
D\pi(\xi)=P_{T_\xi\widetilde\Gamma}
=
T(\xi)T(\xi)^\top .
\]
The final claim follows immediately.
\end{proof}

\begin{proptwo*}[Exponential decay of transverse displacement]
\label{prop:exp_error_decay_app}
Under the tubular-neighborhood assumption, the vector field \(v^\star\) is path-relative transversely exponentially stable on \(U_\rho\) with rate \(\lambda\), up to the possible longitudinal exit time from the tubular neighborhood.
\end{proptwo*}

\begin{proof}
We first record the basic orthogonality identity used below.
By Lemma~\ref{lem:tubular_projection_geometry}, for every \(x\in U_\rho\) and every \(w\in\mathbb R^d\),
\[
D\pi(x)w\in T_{\pi(x)}\widetilde\Gamma .
\]
By the tubular decomposition, the displacement
\[
e(x)=x-\pi(x)
\]
lies in \(N_{\pi(x)}\widetilde\Gamma\). Hence
\begin{equation}
\label{eq:normal_projection_identity}
e(x)^\top D\pi(x)w=0,
\qquad
\forall x\in U_\rho,\ \forall w\in\mathbb R^d .
\end{equation}

Now define
\[
V(x):=\frac12\|e(x)\|^2
=
\frac12\|x-\pi(x)\|^2 .
\]
For any \(w\in\mathbb R^d\),
\[
De(x)[w]
=
w-D\pi(x)w,
\]
and therefore
\[
DV(x)[w]
=
e(x)^\top De(x)[w]
=
e(x)^\top\bigl(w-D\pi(x)w\bigr).
\]
Using \eqref{eq:normal_projection_identity}, this reduces to
\begin{equation}
\label{eq:grad_distance_identity}
DV(x)[w]=e(x)^\top w .
\end{equation}
Equivalently, \(\nabla V(x)=e(x)\) on \(U_\rho\).

Let \(y_t\) be a solution of
\[
\dot y_t=v^\star(y_t)
=
\tau(\pi(y_t))-\lambda e(y_t),
\]
with \(y_{t_0}\in U_\rho\).
Let \(T_{\max}\le \infty\) denote the maximal time such that
\[
y_t\in U_\rho
\qquad
\text{for all }t\in [t_0,T_{\max}) .
\]
All computations below are therefore made on \([t_0,T_{\max})\).

Using \eqref{eq:grad_distance_identity}, we obtain
\[
\frac{d}{dt}V(y_t)
=
DV(y_t)[\dot y_t]
=
e(y_t)^\top \dot y_t .
\]
Substituting the definition of \(v^\star\) gives
\[
\frac{d}{dt}V(y_t)
=
e(y_t)^\top \tau(\pi(y_t))
-
\lambda\|e(y_t)\|^2 .
\]
Since
\[
e(y_t)\in N_{\pi(y_t)}\widetilde\Gamma,
\qquad
\tau(\pi(y_t))\in T_{\pi(y_t)}\widetilde\Gamma,
\]
the first term vanishes:
\[
e(y_t)^\top\tau(\pi(y_t))=0.
\]
Consequently,
\[
\frac{d}{dt}V(y_t)
=
-\lambda\|e(y_t)\|^2
=
-2\lambda V(y_t).
\]
Solving this scalar differential equation on \([t_0,T_{\max})\) yields
\[
V(y_t)
=
\exp(-2\lambda(t-t_0))V(y_{t_0}),
\qquad
t\in[t_0,T_{\max}) .
\]
Taking square roots gives
\[
\|e(y_t)\|
=
\exp(-\lambda(t-t_0))\|e(y_{t_0})\|,
\qquad
t\in[t_0,T_{\max}) .
\]
In particular,
\[
\|e(y_t)\|
\le
\exp(-\lambda(t-t_0))\|e(y_{t_0})\|,
\qquad
t\in[t_0,T_{\max}) .
\]

Since \(\|e(y_t)\|\) is non-increasing, the trajectory cannot exit \(U_\rho\) through its lateral, transverse boundary.
Thus the solution satisfies the exponential decay bound for all \(t\in[t_0,T_{\max})\), where \(T_{\max}\le\infty\) is the first possible time at which the trajectory exits the longitudinal boundaries of the tubular neighborhood \(U_\rho\).
This proves path-relative transverse exponential stability with rate \(\lambda\) on its interval of validity inside the tube.
\end{proof}

\subsection{Transverse Contraction}
\label{app:cor21}

For an autonomous vector field \(v\), write
\[
A_v(x):=Dv(x),
\qquad
A_{v,s}(x):=\frac{A_v(x)+A_v(x)^\top}{2}.
\]

\begin{definition}[Transverse contraction {\citep{manchester2017control}}]
A vector field \(v\) is \emph{transversely \(\mu\)-contracting} on a region \(U\) if for every \(x\in U\) and every differential displacement \(\delta x\) satisfying
\[
\delta x^\top v(x)=0,
\]
one has
\[
\delta x^\top A_{v,s}(x)\,\delta x
\le
-\mu\|\delta x\|^2 .
\]
\end{definition}

We next show that the path-relative target field \(v^\star\) is also transversely contracting in the standard sense of the preceding definition.
The proof uses the same tubular setup and the same notation \(\Gamma\), \(\widetilde\Gamma\), \(U_\rho\), \(\pi\), \(e\), and \(v^\star\) introduced in Appendix~\ref{app:prop2}.
Since the reference trajectory is \(C^2\), the tangent velocity \(\tau\) is \(C^1\) along \(\widetilde\Gamma\), and since \(\pi\) is \(C^1\) on a sufficiently thin tube, \(v^\star\) is \(C^1\) on \(U_\rho\).

We first record a compactness principle that allows one to pass from strict transverse contraction on the reference curve to strict transverse contraction in a sufficiently thin neighborhood.

\begin{lemma}[Persistence of strict transverse contraction]
\label{lem:persistence_transverse_contraction}
Let \(v\in C^1(U;\mathbb R^d)\), let \(K\subset U\) be compact, and assume that
\[
v(\xi)\neq 0
\qquad
\text{for all }\xi\in K.
\]
Suppose that there exists \(\mu>0\) such that, for every \(\xi\in K\) and every \(\delta\in\mathbb R^d\) with
\[
\delta^\top v(\xi)=0,
\]
one has
\[
\delta^\top A_{v,s}(\xi)\delta
\le
-\mu\|\delta\|^2 .
\]
Then, for every \(0<\mu'<\mu\), there exists an open neighborhood \(W\subset U\) of \(K\) such that
\[
\delta^\top A_{v,s}(x)\delta
\le
-\mu'\|\delta\|^2
\]
for every \(x\in W\) and every \(\delta\in\mathbb R^d\) satisfying
\[
\delta^\top v(x)=0.
\]
Moreover, \(W\) may be chosen so that \(v\) is bounded away from zero on \(W\).
\end{lemma}

\begin{proof}
% Since \(v\) is continuous and
% \[
% \inf_{\xi\in K}\|v(\xi)\|>0
% \]
% as \(K\) is compact and \(v\neq 0\) on \(K\), there exists a neighborhood of \(K\) on which \(v\) is bounded away from zero.
% We may take \(W\) to be contained in this neighborhood, so that the transversality condition \(\delta^\top v(x)=0\) is non-degenerate throughout \(W\).

% It suffices to prove the contraction claim for unit vectors \(\|\delta\|=1\), since the inequality is homogeneous in \(\delta\).
Since \(v\) is continuous and
\[
\inf_{\xi\in K}\|v(\xi)\|>0
\]
as \(K\) is compact and \(v\neq 0\) on \(K\), there exists an open neighborhood \(W_{\mathrm{nz}}\) of \(K\) on which \(v\) is bounded away from zero.

It suffices to prove the contraction claim for unit vectors \(\|\delta\|=1\), since the inequality is homogeneous in \(\delta\).

Argue by contradiction.
Suppose that the claim fails for some \(0<\mu'<\mu\).
Then for every \(n\in\mathbb N\), there exist
\[
x_n\in U,
\qquad
\operatorname{dist}(x_n,K)<\frac1n,
\qquad
\|\delta_n\|=1,
\qquad
\delta_n^\top v(x_n)=0,
\]
such that
\[
\delta_n^\top A_{v,s}(x_n)\delta_n>-\mu' .
\]
Let \(\xi_n\in K\) be a nearest point in \(K\) to \(x_n\). Then
\[
\|x_n-\xi_n\|
=
\operatorname{dist}(x_n,K)
\to 0.
\]
Because \(K\) is compact, after passing to a subsequence we may assume
\[
\xi_n\to \xi\in K.
\]
Hence also
\[
x_n\to \xi.
\]
The unit sphere is compact, so after passing to another subsequence we may assume
\[
\delta_n\to\delta,
\qquad
\|\delta\|=1.
\]
By continuity of \(v\),
\[
0
=
\lim_{n\to\infty}\delta_n^\top v(x_n)
=
\delta^\top v(\xi).
\]
Thus \(\delta\) is transverse to \(v(\xi)\).
By continuity of \(A_{v,s}\),
\[
\delta^\top A_{v,s}(\xi)\delta
=
\lim_{n\to\infty}
\delta_n^\top A_{v,s}(x_n)\delta_n
\ge
-\mu' .
\]
On the other hand, the assumed transverse contraction inequality on \(K\) gives
\[
\delta^\top A_{v,s}(\xi)\delta
\le
-\mu\|\delta\|^2
=
-\mu .
\]
This contradicts \(\mu'<\mu\).
Therefore there exists an open neighborhood \(W_0\) of \(K\) on which the transverse \(\mu'\)-contraction inequality holds. Taking
\[
W:=W_0\cap W_{\mathrm{nz}}
\]
proves both the contraction claim and the asserted nonvanishing property.
\end{proof}

We now restate and prove the transverse-contraction corollary from the main text.

\begin{cor21*}[Transverse contraction]
\label{cor:standard_transverse_contraction_app}
Under the assumptions of Proposition~\ref{prop:exp_error_decay}, after possibly shrinking the tubular neighborhood \(U_\rho\), the path-relative target field
\[
v^\star(x)
=
\tau(\pi(x))-\lambda\bigl(x-\pi(x)\bigr)
\]
is transversely contracting on \(U_\rho\).
\end{cor21*}

\begin{proof}
We first prove strict transverse contraction on the reference path \(\Gamma\).

Fix \(\xi\in\Gamma\).
Since \(\pi(\xi)=\xi\) and \(e(\xi)=0\), we have
\[
v^\star(\xi)=\tau(\xi).
\]
By assumption, the reference velocity is nonzero on \(\Gamma\). Hence
\[
\alpha_{\min}
:=
\inf_{\xi\in\Gamma}\|\tau(\xi)\|
>0.
\]
In particular, \(v^\star\) is nonzero on \(\Gamma\).

Let \(\delta\in\mathbb R^d\) satisfy the flow-transversality condition
\[
\delta^\top v^\star(\xi)=0.
\]
Since \(v^\star(\xi)=\tau(\xi)\) spans the tangent line \(T_\xi\widetilde\Gamma\), this condition means that
\[
\delta\in N_\xi\widetilde\Gamma .
\]
By Lemma~\ref{lem:tubular_projection_geometry},
\[
D\pi(\xi)\delta=0.
\]

Differentiating
\[
v^\star(x)
=
\tau(\pi(x))-\lambda\bigl(x-\pi(x)\bigr),
\]
we obtain, for any displacement \(\eta\in\mathbb R^d\),
\[
Dv^\star(x)\eta
=
D\tau(\pi(x))D\pi(x)\eta
-\lambda\bigl(I-D\pi(x)\bigr)\eta .
\]
Evaluating at \(x=\xi\) and \(\eta=\delta\), and using \(D\pi(\xi)\delta=0\), gives
\[
Dv^\star(\xi)\delta
=
-\lambda\delta .
\]
Therefore,
\[
\delta^\top A_{v^\star,s}(\xi)\delta
=
\delta^\top\frac{Dv^\star(\xi)+Dv^\star(\xi)^\top}{2}\delta
=
\delta^\top Dv^\star(\xi)\delta
=
-\lambda\|\delta\|^2 .
\]
Here we used the standard identity
\[
x^\top Mx
=
x^\top \frac{M+M^\top}{2}x
\]
for any square matrix \(M\) and vector \(x\), since \(x^\top M^\top x=x^\top Mx\) as both sides are equal to the same scalar.
Thus \(v^\star\) is transversely \(\lambda\)-contracting on the reference path \(\Gamma\).

It remains to extend this strict inequality from \(\Gamma\) to a sufficiently thin tube.
Since \(v^\star\in C^1\) in a tubular neighborhood of \(\Gamma\), since \(\Gamma\) is compact, and since \(v^\star\neq 0\) on \(\Gamma\), Lemma~\ref{lem:persistence_transverse_contraction} applies with
\[
v=v^\star,
\qquad
K=\Gamma,
\qquad
\mu=\lambda,
\qquad
\mu'=\frac{\lambda}{2}.
\]
Hence there exists an open neighborhood \(W\) of \(\Gamma\) on which the transverse \(\lambda/2\)-contraction inequality holds.
Since \(\Gamma\) is compact and \(W\) is an open neighborhood of \(\Gamma\) in \(\mathbb R^d\), there exists a sufficiently small radius \(\rho>0\) such that the tubular neighborhood
\[
U_\rho
=
\{x\in\mathbb R^d:\operatorname{dist}(x,\Gamma)<\rho\}
\]
is contained in \(W\).
After shrinking the original tubular neighborhood accordingly, for every \(x\in U_\rho\) and every displacement \(\delta\) satisfying
\[
\delta^\top v^\star(x)=0,
\]
one has
\[
\delta^\top A_{v^\star,s}(x)\delta
\le
-\frac{\lambda}{2}\|\delta\|^2.
\]
Hence \(v^\star\) is transversely \(\lambda/2\)-contracting on \(U_\rho\).
\end{proof}

\subsection{Approximation Error and Tube Stability}
\label{app:prop3}

We now restate and prove Proposition~\ref{prop:approx_tube}. The proof uses the same tubular setup and the same notation \(\Gamma\), \(\widetilde\Gamma\), \(U_\rho\), \(\pi\), \(e\), and \(v^\star\) introduced in Appendix~\ref{app:prop2}. In particular, we require Assumption~\ref{ass:tubular_neighborhood}, and we use again the Lyapunov function
\[
V(x)
:=
\frac12\|e(x)\|^2,
\]
together with the gradient identity \(\nabla V(x)=e(x)\) on \(U_\rho\), established in \eqref{eq:grad_distance_identity} in the proof of Proposition~\ref{prop:exp_error_decay}.

Let \(\hat v:U_\rho\to\mathbb R^d\) be a locally Lipschitz vector field, interpreted as an approximation of the ideal field \(v^\star\), such as the learned velocity field for a fixed latent and step size. Define its \emph{transverse approximation error}
\[
\varepsilon_\perp
:=
\sup_{x\in U_\rho}
\bigl\|\Pi^\perp(\pi(x))\bigl(\hat v(x)-v^\star(x)\bigr)\bigr\|
\le
\sup_{x\in U_\rho}\|\hat v(x)-v^\star(x)\| .
\]

\begin{propthree*}[Approximation error contracts to a tube]
\label{prop:approx_tube_app}
Suppose Assumption~\ref{ass:tubular_neighborhood} holds and that
\[
\varepsilon_\perp<\lambda\rho .
\]
Let \(y_t\) be a solution of \(\dot y_t=\hat v(y_t)\) with \(y_{t_0}\in U_\rho\), and let \(T_{\max}\le\infty\) denote the maximal time such that \(y_s\in U_\rho\) for all \(s\in[t_0,t]\). Then, for all \(t\in[t_0,T_{\max})\),
\begin{equation}
\label{eq:approx_tube_bound}
\|e(y_t)\|
\le
\exp(-\lambda(t-t_0))\|e(y_{t_0})\|
+
\frac{\varepsilon_\perp}{\lambda}
\bigl(1-\exp(-\lambda(t-t_0))\bigr)
\le
\max\Bigl(\|e(y_{t_0})\|,\ \frac{\varepsilon_\perp}{\lambda}\Bigr).
\end{equation}
In particular, the trajectory cannot exit \(U_\rho\) through its lateral, transverse boundary, so that \(T_{\max}\) is again the first possible longitudinal exit time from the tubular neighborhood, and if \(T_{\max}=\infty\), then
\[
\limsup_{t\to\infty}\|e(y_t)\|
\le
\frac{\varepsilon_\perp}{\lambda}.
\]
\end{propthree*}

\begin{proof}
All computations below are made on \([t_0,T_{\max})\), where the trajectory remains in \(U_\rho\). Since \(\pi\) is \(C^1\) on \(U_\rho\) and \(\hat v\) is locally Lipschitz, the map \(t\mapsto V(y_t)\) is \(C^1\) on this interval.

We first derive a differential inequality for the transverse displacement. Using the gradient identity \eqref{eq:grad_distance_identity} and splitting the approximate field into the ideal field and the field error, \(\hat v=v^\star+(\hat v-v^\star)\), the derivative of the Lyapunov function decomposes as
\begin{equation}
\label{eq:approx_lyapunov_split}
\frac{d}{dt}V(y_t)
=
e(y_t)^\top\dot y_t
=
e(y_t)^\top\hat v(y_t)
=
\underbrace{e(y_t)^\top v^\star(y_t)}_{\text{ideal field}}
\;+\;
\underbrace{e(y_t)^\top\bigl(\hat v(y_t)-v^\star(y_t)\bigr)}_{\text{field error}} .
\end{equation}

The first term is the transverse contraction of Proposition~\ref{prop:exp_error_decay}. The tubular decomposition gives \(e(y_t)\in N_{\pi(y_t)}\widetilde\Gamma\) and \(\tau(\pi(y_t))\in T_{\pi(y_t)}\widetilde\Gamma\), so the tangential part of the ideal field is orthogonal to the transverse displacement and drops out, leaving
\begin{equation}
\label{eq:approx_ideal_term}
e(y_t)^\top v^\star(y_t)
=
\underbrace{e(y_t)^\top\tau(\pi(y_t))}_{=\,0}
-
\lambda\|e(y_t)\|^2
=
-\lambda\|e(y_t)\|^2 .
\end{equation}

The second term involves only the transverse approximation error.
The projector \(\Pi^\perp(\pi(y_t))\) can be applied to \(e(y_t)\) without changing it, because it lies in \(N_{\pi(y_t)}\widetilde\Gamma\).
Furthermore, the projector is symmetric, so by Cauchy--Schwarz
\begin{align}
e(y_t)^\top\bigl(\hat v(y_t)-v^\star(y_t)\bigr)
&=
\bigl(\Pi^\perp(\pi(y_t))\,e(y_t)\bigr)^\top\bigl(\hat v(y_t)-v^\star(y_t)\bigr)
\nonumber\\[2pt]
&=
e(y_t)^\top\,\Pi^\perp(\pi(y_t))\bigl(\hat v(y_t)-v^\star(y_t)\bigr)
\nonumber\\[2pt]
&\le
\|e(y_t)\|\,\bigl\|\Pi^\perp(\pi(y_t))\bigl(\hat v(y_t)-v^\star(y_t)\bigr)\bigr\|
\nonumber\\[2pt]
&\le
\|e(y_t)\|\,\varepsilon_\perp .
\label{eq:approx_error_term}
\end{align}

Substituting \eqref{eq:approx_ideal_term} and \eqref{eq:approx_error_term} into \eqref{eq:approx_lyapunov_split}, and then expressing the transverse displacement through the Lyapunov function by \(\|e(y_t)\|^2=2V(y_t)\), we obtain
\begin{align}
\frac{d}{dt}V(y_t)
&\le
-\lambda\|e(y_t)\|^2
+
\|e(y_t)\|\,\varepsilon_\perp
\nonumber\\[2pt]
&=
-2\lambda V(y_t)
+
\sqrt{2V(y_t)}\,\varepsilon_\perp .
\label{eq:approx_lyapunov_ineq}
\end{align}
The contraction enters with the square of the transverse displacement while the approximation error enters only linearly, which is what makes the two effects balance at a finite radius. Only the normal component of the field error appears in this estimate.

We next integrate \eqref{eq:approx_lyapunov_ineq} by a Gr\"onwall-type comparison. Since \(\|e(y_t)\|=\sqrt{2V(y_t)}\) need not be differentiable where it vanishes, fix \(\delta>0\) and set
\[
u_\delta(t)
:=
\sqrt{2V(y_t)+\delta^2}
\ge
\max\bigl(\|e(y_t)\|,\ \delta\bigr).
\]
Then \(u_\delta\) is \(C^1\) on \([t_0,T_{\max})\), and by \eqref{eq:approx_lyapunov_ineq},
\[
\dot u_\delta(t)
=
\frac{1}{u_\delta(t)}\frac{d}{dt}V(y_t)
\le
-\lambda\,\frac{u_\delta(t)^2-\delta^2}{u_\delta(t)}
+
\frac{\|e(y_t)\|}{u_\delta(t)}\,\varepsilon_\perp
\le
-\lambda u_\delta(t)+\lambda\delta+\varepsilon_\perp,
\]
where the first inequality uses \(2V(y_t)=u_\delta(t)^2-\delta^2\), and the second uses \(\delta^2/u_\delta(t)\le\delta\) and \(\|e(y_t)\|\le u_\delta(t)\).
We now multiply both sides by \(\exp(\lambda(t-t_0))\), and obtain
\[
\exp(\lambda(t-t_0))\dot u_\delta(t)
+
\lambda\exp(\lambda(t-t_0))u_\delta(t)
\le
\bigl(\varepsilon_\perp+\lambda\delta\bigr)\exp(\lambda(t-t_0)).
\]
We identify the left side with the product rule applied to \(\exp(\lambda(t-t_0))u_\delta(t)\):
\[
\frac{d}{dt}\Bigl(\exp(\lambda(t-t_0))u_\delta(t)\Bigr)
=
\exp(\lambda(t-t_0))\dot u_\delta(t)
+
\lambda\exp(\lambda(t-t_0))u_\delta(t)
\le
\bigl(\varepsilon_\perp+\lambda\delta\bigr)\exp(\lambda(t-t_0)) .
\]

Both sides are now integrated from \(t_0\) to \(t\). The left side is the integral of a derivative, and the right side is an elementary exponential integral,
\begin{align*}
\int_{t_0}^{t}\frac{d}{ds}\Bigl(\exp(\lambda(s-t_0))u_\delta(s)\Bigr)\,ds
&=
\exp(\lambda(t-t_0))u_\delta(t)-u_\delta(t_0),
\\[2pt]
\bigl(\varepsilon_\perp+\lambda\delta\bigr)\int_{t_0}^{t}\exp(\lambda(s-t_0))\,ds
&=
\bigl(\varepsilon_\perp+\lambda\delta\bigr)
\Bigl[\frac{\exp(\lambda(s-t_0))}{\lambda}\Bigr]_{s=t_0}^{s=t}\\[2pt]
&=
\frac{\varepsilon_\perp+\lambda\delta}{\lambda}\bigl(\exp(\lambda(t-t_0))-1\bigr).
\end{align*}
The first quantity is therefore at most the second,
\[
\exp(\lambda(t-t_0))u_\delta(t)-u_\delta(t_0)
\le
\frac{\varepsilon_\perp+\lambda\delta}{\lambda}\bigl(\exp(\lambda(t-t_0))-1\bigr),
\]
and adding \(u_\delta(t_0)\) to both sides and then multiplying by the positive number \(\exp(-\lambda(t-t_0))\) gives
\[
u_\delta(t)
\le
\exp(-\lambda(t-t_0))u_\delta(t_0)
+
\frac{\varepsilon_\perp+\lambda\delta}{\lambda}
\bigl(1-\exp(-\lambda(t-t_0))\bigr).
\]
Letting \(\delta\downarrow0\) and using \(\|e(y_t)\|\le u_\delta(t)\) together with \(u_\delta(t_0)\to\|e(y_{t_0})\|\) gives the first inequality in \eqref{eq:approx_tube_bound}.

For the second inequality, abbreviate \(\theta:=\exp(-\lambda(t-t_0))\), which lies in \((0,1]\) for \(t\ge t_0\). The bound just obtained reads \(\|e(y_t)\|\le\theta\|e(y_{t_0})\|+(1-\theta)\varepsilon_\perp/\lambda\), and the two weights \(\theta\) and \(1-\theta\) are nonnegative and add up to one. Raising each of the two weighted values to the larger of them therefore gives
\begin{align*}
\theta\|e(y_{t_0})\|+(1-\theta)\frac{\varepsilon_\perp}{\lambda}
&\le
\theta\max\Bigl(\|e(y_{t_0})\|,\ \frac{\varepsilon_\perp}{\lambda}\Bigr)
+
(1-\theta)\max\Bigl(\|e(y_{t_0})\|,\ \frac{\varepsilon_\perp}{\lambda}\Bigr)
\\[2pt]
&=
\bigl(\theta+(1-\theta)\bigr)\max\Bigl(\|e(y_{t_0})\|,\ \frac{\varepsilon_\perp}{\lambda}\Bigr)
\\[2pt]
&=
\max\Bigl(\|e(y_{t_0})\|,\ \frac{\varepsilon_\perp}{\lambda}\Bigr),
\end{align*}
which is the second inequality in \eqref{eq:approx_tube_bound}.

It remains to rule out lateral exit. The initial point lies in \(U_\rho\), which gives \(\|e(y_{t_0})\|<\rho\), and dividing the assumption \(\varepsilon_\perp<\lambda\rho\) by \(\lambda>0\) gives \(\varepsilon_\perp/\lambda<\rho\). Both entries of the maximum are therefore strictly smaller than the tube radius, and the second inequality in \eqref{eq:approx_tube_bound} bounds the transverse displacement by that maximum, so
\[
\|e(y_t)\|
\ \le\
\bar\rho
:=
\max\Bigl(\|e(y_{t_0})\|,\ \frac{\varepsilon_\perp}{\lambda}\Bigr)
\ <\
\rho
\qquad
\text{for all }t\in[t_0,T_{\max}).
\]
A lateral exit at a finite \(T_{\max}\) would mean that the transverse displacement reaches the tube radius, so that \(\|e(y_t)\|\to\rho\) as \(t\uparrow T_{\max}\) by continuity, and the strict bound above rules this out. Hence \(T_{\max}\) is the first possible time at which the trajectory leaves \(U_\rho\) through its longitudinal ends, exactly as in Proposition~\ref{prop:exp_error_decay}.

Finally, suppose \(T_{\max}=\infty\). Since \(\lambda>0\), the factor \(\exp(-\lambda(t-t_0))\) tends to \(0\) as \(t\to\infty\), so in the first inequality of \eqref{eq:approx_tube_bound} the term carrying the initial transverse displacement vanishes and the factor \(1-\exp(-\lambda(t-t_0))\) tends to \(1\), which leaves
\[
\limsup_{t\to\infty}\|e(y_t)\|
\le
\lim_{t\to\infty}
\Bigl(
\exp(-\lambda(t-t_0))\|e(y_{t_0})\|
+
\frac{\varepsilon_\perp}{\lambda}\bigl(1-\exp(-\lambda(t-t_0))\bigr)
\Bigr)
=
\frac{\varepsilon_\perp}{\lambda},
\]
which is the last assertion of the proposition and concludes the proof.
\end{proof}

Note that without our proposed correction, i.e. \(\lambda=0\), the same argument yields only
\[\frac{d}{dt}V(y_t)\le\sqrt{2V(y_t)}\,\varepsilon_\perp\]
and hence the linear-in-time drift
\[
\|e(y_t)\|
\le
\|e(y_{t_0})\|
+
\varepsilon_\perp\,(t-t_0).
\]
The corrected targets of \eqref{eq:corrected_secant_target} thus upgrade unbounded error growth to a time-uniform tube of radius \(\varepsilon_\perp/\lambda\). Larger \(\lambda\) tightens the tube but stiffens the target field.

\subsection{Secant correction for finite-step targets}
\label{app:secant_objective}

The continuous-time target in \eqref{eq:target_vf} adds the transverse correction \(-\lambda e\) to the tangent velocity. As shown in Proposition~\ref{prop:exp_error_decay}, this makes the norm of a transverse perturbation decay exactly as \(\exp(-\lambda h)\) over a time interval of length \(h\). We seek the analogous correction for the finite-step secant target \(\Delta_h x_t\).

Without a correction, applying the secant target at the perturbed point \(x_t+e\) gives
\[
(x_t+e)+h\Delta_h x_t
=
x_{t+h}+e.
\]
Thus the reference point advances from \(x_t\) to \(x_{t+h}\), but the perturbation is transported unchanged. The corrected target should instead leave a displacement of norm \(\exp(-\lambda h)\|e\|\) from \(x_{t+h}\).

This condition does not determine a unique target: it only requires the endpoint to lie on a sphere around \(x_{t+h}\). We therefore select, among all targets with the prescribed decay, the one closest to the original secant target. The equality constraint below is the finite-step analogue of this exact norm decay, measured relative to the secant endpoint \(x_{t+h}\).

\begin{proposition}[Minimal finite-step correction]
\label{prop:minimal_secant_correction}
Fix \(h>0\), \(\lambda>0\), and two distinct reference points \(x_t\) and \(x_{t+h}\). Let
\[
\Delta_h x_t
:=
\frac{x_{t+h}-x_t}{h},
\qquad
e:=\sigma_k\eta,
\qquad
\sigma_k>0,
\qquad
\|\eta\|=1,
\qquad
\eta^\top(x_{t+h}-x_t)=0.
\]
Define the corrected target by
\begin{equation}
\label{eq:min_correction}
\widetilde{\Delta}_h x_t
:=
\operatorname*{arg\,min}_{v}
\|v-\Delta_h x_t\|
\quad
\text{subject to}
\quad
\|(x_t+e)+hv-x_{t+h}\|
=
\exp(-\lambda h)\|e\|.
\end{equation}
Then the minimizer is unique and is given by
\[
\widetilde{\Delta}_h x_t
=
\Delta_h x_t
+
c(h,\lambda)e,
\qquad
c(h,\lambda)
=
\frac{\exp(-\lambda h)-1}{h}.
\]
In particular, one corrected step maps
\[
x_t+e
\longmapsto
x_{t+h}+\exp(-\lambda h)e.
\]
\end{proposition}

\begin{proof}
Write the candidate target as the secant target plus an unknown correction,
\[
v=\Delta_h x_t+u.
\]
Since \(x_t+h\Delta_h x_t=x_{t+h}\), its update from the perturbed point satisfies
\begin{align*}
(x_t+e)+hv-x_{t+h}
&=
(x_t+e)+h(\Delta_h x_t+u)-x_{t+h}\\
&=
e+hu.
\end{align*}
Therefore, \eqref{eq:min_correction} is equivalent to
\[
\operatorname*{minimize}_{u}\ \|u\|
\qquad
\text{subject to}
\qquad
\|e+hu\|=\exp(-\lambda h)\|e\|.
\]
Set \(q:=\exp(-\lambda h)\in(0,1)\). For every feasible correction \(u\), the reverse triangle inequality yields
\begin{align*}
h\|u\|
&=
\|(e+hu)-e\|\\
&\ge
\|e\|-\|e+hu\|\\
&=
(1-q)\|e\|.
\end{align*}
Hence any feasible correction must satisfy
\[
\|u\|
\ge
\frac{1-q}{h}\|e\|.
\]

Consider now the radial correction
\[
u^\star
:=
-\frac{1-q}{h}e
=
\frac{q-1}{h}e.
\]
It is feasible because
\[
e+hu^\star
=
e-(1-q)e
=
qe,
\]
and therefore
\[
\|e+hu^\star\|
=
q\|e\|.
\]
Moreover,
\[
\|u^\star\|
=
\frac{1-q}{h}\|e\|,
\]
so \(u^\star\) attains the lower bound and is a minimizer.

To see uniqueness, equality in the reverse triangle inequality requires \(e+hu\) to point in the same direction as \(e\). The constraint fixes its norm to \(q\|e\|\), and hence
\[
e+hu=qe.
\]
Solving for \(u\) gives \(u=u^\star\). Consequently,
\[
\widetilde{\Delta}_h x_t
=
\Delta_h x_t+u^\star
=
\Delta_h x_t
+
\frac{\exp(-\lambda h)-1}{h}e,
\]
which proves the claim.
\end{proof}

The orthogonality assumption is not needed to solve the minimization problem: the same minimizer is obtained for any nonzero perturbation \(e\). Its role is geometric. Writing \(s:=x_{t+h}-x_t\), the condition \(e\perp s\) gives
\[
\|e-\alpha s\|^2
=
\|e\|^2+\alpha^2\|s\|^2,
\qquad
\alpha\in\mathbb R.
\]
Hence \(x_t\) is the nearest point on the secant line to \(x_t+e\), and \(\|e\|=\sigma_k\) is its transverse displacement from that line. Moreover, since \(\Delta_h x_t=s/h\), the correction is orthogonal to the secant velocity, and therefore
\[
\|\widetilde{\Delta}_h x_t\|^2
=
\|\Delta_h x_t\|^2
+
c(h,\lambda)^2\sigma_k^2.
\]
Thus the correction leaves the longitudinal component unchanged, mirroring \(-\lambda e\perp\tau\) in the continuous-time field. This is why finite-step perturbations are chosen orthogonal to the secant rather than only to the tangent: for a curved path, tangent-normal noise can have a component along the secant and would therefore also modify the longitudinal step.

The finite-step coefficient converges to the continuous-time correction as \(h\to0\). Indeed,
\[
\exp(-\lambda h)
=
1-\lambda h+\mathcal O(h^2),
\]
and therefore
\[
c(h,\lambda)
=
\frac{\exp(-\lambda h)-1}{h}
=
-\lambda+\mathcal O(h).
\]
If \(x\) is differentiable at \(t\), \(\dot x_t\neq0\), and the secant-normal directions satisfy \(\eta_h\to\eta_0\), then \(\eta_0\perp\dot x_t\) and
\[
\widetilde{\Delta}_h x_t
\longrightarrow
\dot x_t-\lambda\sigma_k\eta_0,
\]
recovering the transverse correction in \eqref{eq:target_vf}.

\section{Datasets and Training Details}
\label{app:data_training_details}

This section provides implementation details for the experiments in the main paper.
We focus on choices that affect the training objective, target construction, optimization procedure, datasets, and model classes.
Low-level engineering details, such as tensor reshaping, data-loader indexing, and hardware-specific mixed-precision settings, are omitted and will be provided with the released code.

All experiments were trained on a single GPU in a shared compute cluster, using NVIDIA L40S or RTX A6000 GPUs with 48GB of memory for the 2D, PDE, and molecular experiments, and a single data-center GPU of the NVIDIA A100-40GB class for the KTH experiments.
The full generative-model experiments reported in the paper were trained for at most three days per run.
For the 2D, PDE, and molecular experiments, training of the auxiliary flow-matching model and the score-induced interpolator was limited to at most one day each on the same hardware. Each stage of the KTH latent pipeline (frame autoencoder, flow-matching model, interpolator, generative model) trains on the order of one to two days on a single GPU.

\subsection{Implementation Details: Full Training Algorithm}
\label{app:full_training_alg}

This subsection records how the robustified latent objective of the main text is realized in code. The constant factor \(1/(2\sigma^2)\) is absorbed into the weight of the reconstruction term, and the interpolation model used to construct \((X_t,\Delta_h X_t)\) is pretrained and frozen during this stage.

\paragraph{Local target sampling.}
Training examples are discrete sequences \(X=(x_0,\ldots,x_{T-1})\).
At each optimization step we sample adjacent endpoint pairs \((x_i,x_{i+1})\), a normalized interpolation time \(t\in[0,1]\), and a step size \(h\ge 0\), so that every batch mixes tangent queries at \(h=0\) with finite-step secant queries at \(h>0\).
For score-induced interpolation, the frozen prior returns these targets through the lifted-space interpolation and denoising procedure of Appendix~\ref{app:interp_training}, and experiments using linear interpolation call the same interface with the Euclidean linear path.

\paragraph{Transverse robustness.}
When the transverse robustness objective is enabled, the query point is perturbed by \(\sigma_k\eta\), with \(\sigma_k\) drawn from a small set of noise scales and \(\eta\) projected to be orthogonal to \(\Delta_h X_t\), so that the noise probes directions transverse to the tangent for \(h=0\) and to the secant for \(h>0\).
The regression target is then the corrected target (\eqref{eq:corrected_secant_target}) of the main text, whose finite-step coefficient is derived in Appendix~\ref{app:secant_objective}.
Setting \(\sigma_k=0\) disables the objective and recovers ordinary local target regression.

\paragraph{KL gradients.}
In most experiments we stop gradients from the KL term into the posterior parameters, while the posterior still receives gradients from the reconstruction term.
The KL therefore acts as a prior-matching objective rather than as a force that directly contracts the posterior, which is related to posterior matching approaches~\citep{strauss2022posterior}.
The trained components are the vector field and, when latent conditioning is enabled, the two encoders, optimized with AdamW~\citep{loshchilov2017decoupled} and learning-rate warmup.

\begin{algorithm}[t]
\caption{Training the continuous-time generative vector field}
\label{alg:full_training}
\begin{algorithmic}[1]
\Require Discrete training sequences \(X=(x_0,\ldots,x_{T-1})\); frozen interpolation prior; vector field \(v_\theta\); optional posterior \(q_\phi\); optional prior \(p_\vartheta\)
\For{each optimization step}
    \State Sample sequences and adjacent endpoint pairs \((x_i,x_{i+1})\)
    \State Sample interpolation times \(t\in[0,1]\) and step sizes \(h\ge 0\)
    \State Use the frozen interpolation prior to obtain \((X_t,\Delta_h X_t)\)
    \If{transverse robustness is enabled}
        \State Sample a noise scale \(\sigma_k\) and \(\eta\) with \(\|\eta\|=1\) and \(\eta^\top\Delta_h X_t=0\)
        \State Set \(\widetilde X_t \gets X_t+\sigma_k\eta\)
        \State Set \(\widetilde{\Delta}_h X_t \gets \Delta_h X_t+c(h,\lambda)\sigma_k\eta\)
    \Else
        \State Set \(\widetilde X_t \gets X_t\) and \(\widetilde{\Delta}_h X_t\gets\Delta_h X_t\)
    \EndIf
    \If{latent conditioning is enabled}
        \State Sample \(z\sim q_\phi(z\mid X)\)
        \State Predict \(\widehat V\gets v_\theta(\widetilde X_t,z,h)\)
        \State Compute the KL term between \(q_\phi(z\mid X)\) and \(p_\vartheta(z\mid X_{\le t})\)
    \Else
        \State Predict \(\widehat V\gets v_\theta(\widetilde X_t,h)\)
        \State Set the KL term to zero
    \EndIf
    \State Compute \(\mathcal L_{\mathrm{rec}}=\|\widehat V-\widetilde{\Delta}_h X_t\|^2\)
    \State Update trainable parameters using \(\mathcal L_{\mathrm{rec}}+\mathcal L_{\mathrm{KL}}\)
\EndFor
\end{algorithmic}
\end{algorithm}

\paragraph{Main training hyperparameters.}
Table~\ref{tab:main_hparams} summarizes the main optimization hyperparameters and model sizes for the reported experiments.

\begin{table}[h]
\centering
\caption{Main optimization hyperparameters and parameter counts. Parameter counts refer to trainable parameters in the corresponding stage (vector field plus latent encoders, where applicable).}
\label{tab:main_hparams}
\resizebox{\linewidth}{!}{
\begin{tabular}{llccc}
\toprule
Experiment & Stage & Batch size & Learning rate & Parameters \\
\midrule
2D example & FM model & 256 & $10^{-3}$  & 215.6k \\
2D example & Interpolator & 256 & $10^{-3}$  & 215.8k \\
2D example & Generative model & 32 seq. & $10^{-3}$ & 2.15M \\
\midrule
KTH & Frame autoencoder & 32 & $10^{-4}$ & 55.3M \\
KTH & FM model & 128 & $10^{-4}$ & 24.9M \\
KTH & Interpolator & 128 & $10^{-4}$ & 1.78M \\
KTH & Generative model & 16 seq. & $10^{-4}$ & 28.74M \\
\midrule
Gray--Scott & FM model & 128 & $10^{-4}$ & 1.78M \\
Gray--Scott & Interpolator & 64 & $10^{-4}$  & 1.78M \\
Gray--Scott & Generative model & 64 seq. & $10^{-4}$  & 9.21M \\
\midrule
Navier--Stokes & FM model & 128 & $10^{-4}$  & 1.78M \\
Navier--Stokes & Interpolator & 64 & $10^{-4}$  & 1.78M \\
Navier--Stokes & Generative model & 128 seq. & $10^{-4}$  & 63.44M \\
\midrule
Molecular dynamics & FM model & 1024 & $10^{-4}$  & 1.20M \\
Molecular dynamics & Interpolator & 64 & $10^{-3}$ & 735.5k \\
Molecular dynamics & Generative model & 256 & $2\cdot 10^{-4}$  & 1.46M \\
\bottomrule
\end{tabular}
}
\end{table}

\subsection{Implementation Details: Neural Architectures}
\label{app:architectures}

We use different backbone classes depending on the state representation, and all of them expose the same interface to the training objective, mapping a state \(X_t\), a step size \(h\), and an optional latent variable \(z\) to a local tangent or secant target. The low-dimensional trajectories of the 2D example are modeled with residual multilayer perceptrons and temporal convolutional encoders. PDE-valued spatiotemporal fields use a U-Net vector-field backbone with latent cross-attention, together with convolutional--temporal encoders for the posterior and prior. Molecular trajectories use an equivariant graph neural network operating directly on atomic coordinates. The KTH latent-video model uses a related U-Net with spatial-grid latent conditioning, described in Appendix~\ref{data:kth}.

\paragraph{Low-dimensional trajectories.}
For the 2D data, the vector field is a residual MLP with GELU nonlinearities, into which the step size \(h\) and, when enabled, the latent variable \(z\) enter through small embedding MLPs that are added to the hidden representation of each block. The posterior and prior encoders are causal temporal convolutional networks over the flattened states, with exponentially increasing dilations, residual connections, and layer normalization, and their context features are mapped to the parameters of the Gaussian posterior or prior over \(z\).

\paragraph{Image-valued spatiotemporal fields.}
For PDE-valued image fields, the vector field is a multiscale two-dimensional U-Net with a convolutional lifting layer, residual encoder blocks with downsampling, a bottleneck, and a symmetric decoder with skip connections. The step size enters through a Fourier time embedding injected into all residual blocks.

Latent conditioning uses cross-attention rather than additive conditioning. The latent vector is projected into a small set of tokens that the spatial feature maps query through multi-head cross-attention at selected encoder stages, the bottleneck, and selected decoder stages, so the image features retain spatial structure while the latent supplies global trajectory-level information. The posterior and prior encoders embed each frame with a residual CNN, using convolutional residual stages, normalization, adaptive spatial pooling, and a linear projection, and pass the per-frame embeddings through a causal dilated TCN. Causality ensures that the prior at time \(t\) depends only on available context.

\paragraph{Molecular trajectories.}
Molecular states are atom types and three-dimensional coordinates, and the vector field is an EGNN on a fixed molecular graph. Node features are initialized from atom-type embeddings and conditioned on interpolation time and step size through sinusoidal embeddings, and a latent variable, when enabled, is projected to the node-feature dimension and added to all node features. Each layer computes edge messages from source and target node features together with squared interatomic distances, and these messages update node features and coordinates, where the coordinate updates are built from relative displacement vectors, which preserves translation and rotation equivariance. After several message-passing layers, the network outputs either the accumulated coordinate displacement or a coordinate velocity read from the final node features.

\subsection{Demonstrative Example}
\label{data:toy}

The square manifold of Section~\ref{sec:toy} is built from two smooth cubic Bézier arcs per side, one lying closer to the square and one farther away, which gives the inner and the outer branch.
The score model used as interpolation prior is pretrained on noisy samples drawn uniformly from the eight arcs, so it sees the continuous manifold geometry while the sequence model sees only sparse points on closed loops.
Each training sequence alternates between square corners and side midpoints and repeats its first point at the end to close the trajectory.
Small Gaussian noise is added to the eight non-repeated nodes, and the temporal phase is randomized by cyclically shifting the starting point before appending the closure point, so different sequences can share an initial segment and diverge later depending on the unobserved branch choices.

In the experiment, the score-induced interpolator is trained from the pretrained score model and then frozen.
Transverse robustness noise uses scales on a geometric grid in \([0.001, 0.05]\) with contraction rate \(\lambda=10\).
At inference time, we condition on prefixes of different lengths and sample trajectories from the learned prior.

The ablation shown in Table~\ref{tab:toy_ablation} retrains the 2D model over all combinations of
\[\{\text{score},\text{linear targets}\}\times\{\text{vanilla},\text{input noise},\text{corrected}\}\times\{\text{latent on},\text{off}\}.\]
Each of the twelve models changes only the listed components and remains identical otherwise. Linear targets replace the score-induced paths by straight lines between the loop nodes, input noise perturbs the inputs while leaving the targets uncorrected (\(\lambda=0\)), corrected applies the paper's target correction with \(\lambda=10\), and latent-off removes the stochastic framework (both encoders removed and no latent conditioning). All models share the paper model's data loader and training budget (full nine-node loops as training items, all 1024 loops, 1000 epochs corresponding to 32k steps), with one training run per configuration. Three metrics are evaluated for every model. \emph{Off-manifold} is the mean distance of a conditioned one-segment rollout to the analytic manifold. \emph{Branch accuracy} is the fraction of sampled ground-truth loops (random branch configurations) whose conditioned rollout reproduces the full configuration, where each side is classified by the proximity of the trajectory to the central bulge of the inner versus the outer arc. A field that ignores the latent reaches \(\approx 1/16\) by chance. \emph{Rate} is the transverse contraction rate measured as in Appendix~\ref{app:inference_details}.

\subsection{KTH Action}
\label{data:kth}
\label{app:kth_details}
This subsection describes the exact configuration behind Figure~\ref{fig:kth_extrap}, Figure~\ref{fig:kth_interp}, and Tables~\ref{tab:vidode} and~\ref{tab:vidode_extrap}.

\subsubsection{Data and evaluation protocol}
We use the KTH Actions dataset~\citep{schuldt2004recognizing}, which contains 25 subjects performing 6 actions in 599 grayscale videos at $120{\times}160$ and 25\,fps. Following Vid-ODE~\citep{park2021vidode}, each frame is center-cropped to $120{\times}120$ and resized to $128{\times}128$. We use the 16/9 subject split of the official Vid-ODE preprocessing, which holds out persons 2, 3, 5--10, and 22 for testing and leaves 383 training and 216 test videos.

All quantitative results follow Vid-ODE's own test protocol and metric implementation verbatim. The evaluation window is the first 10 frames of every test video, and the two protocols read it as follows.
\begin{itemize}
\item \emph{Interpolation} observes frames $\{1,3,5,7,9\}$ (1-indexed) and is scored on the four held-out frames $\{2,4,6,8\}$.
\item \emph{Extrapolation} observes frames $1$--$5$ and is scored on frames $6$--$10$.
\end{itemize}
The three metrics are computed exactly as in the reference implementation.
\begin{itemize}
\item SSIM on grayscale images, with data range 255, Gaussian weighting, and no sample covariance.
\item LPIPS with the AlexNet variant on $128{\times}128$ inputs scaled to $[-1,1]$.
\item PSNR as the dataset-level value $10\log_{10}(1/\overline{\mathrm{MSE}})$, with MSE on $[0,1]$.
\end{itemize}
All predictions are quantized to 8 bits before scoring, which matches the PNG round trip of the reference implementation.

\begin{table}[t]
\centering
\caption{Optimization settings of the three KTH training stages. A dash marks a setting that the stage configuration does not specify separately. Parameter counts of the generative model are the velocity field followed by the posterior and prior encoders.}
\label{tab:kth_optim}
\resizebox{\linewidth}{!}{
\begin{tabular}{lcccc}
\toprule
& Stage 1 & \multicolumn{2}{c}{Stage 2} & Stage 3 \\
\cmidrule(lr){2-2}\cmidrule(lr){3-4}\cmidrule(lr){5-5}
Setting & Frame autoencoder & Score model & Interpolator & Generative model \\
\midrule
Parameters      & 55.3M & 24.9M & 1.78M & 23.35M $+$ 2.71M / 2.68M \\
Training item   & single frame & single frame latent & consecutive latent pair & length-10 latent window \\
Batch size      & 32 & 128 & 128 & 16 windows, 2 queries each \\
Optimizer       & Adam & AdamW & Adam & AdamW \\
Learning rate   & $10^{-4}$ & $10^{-4}$ & $10^{-4}$ & $10^{-4}$ \\
Weight decay    & none & 0.01 & none & 0.01 (field, rank $\ge2$) \\
Schedule        & 1\% warmup, cosine & 1\% warmup, cosine & 1\% warmup, cosine & 0.3\% warmup, then constant \\
Budget          & 80k steps & 120 epochs & 60k steps & 100k steps \\
Gradient clip   & 1.0 & --- & --- & 1.0, per module \\
EMA decay       & 0.999 (used) & 0.9999 (used) & --- & --- \\
\bottomrule
\end{tabular}
}
\end{table}

\subsubsection{Stage 1: frame autoencoder}
All generative components operate in the latent space of a frame autoencoder with a KL-regularized bottleneck (LDM-style \emph{AutoencoderKL}), trained on individual frames of the KTH training split. The encoder and decoder use GroupNorm(32) and SiLU ResNet blocks with base width 128, channel multipliers $(1,2,4)$, two residual blocks per resolution, and a single-head self-attention block at the bottleneck. We downsample with a factor $f{=}4$ and use latent dimensionality $32{\times}32{\times}4$, which yields 55.3M parameters in total.

Optimization follows Table~\ref{tab:kth_optim}. The loss is $\mathcal{L}_1 + \mathcal{L}_2 + \mathcal{L}_{\mathrm{LPIPS(VGG)}} + 10^{-6}\,\mathrm{KL}$, and random horizontal flips with $p{=}0.5$ serve as an augmentation for autoencoder training only. Test-set reconstruction reaches ${\approx}42$\,dB PSNR.

Frames are given to the autoencoder in $[0,1]$ and encoded from the unflipped videos with the EMA weights and the deterministic posterior \emph{mean}. The resulting latents are normalized per channel by statistics computed over these training-split latents. The autoencoder is frozen for all subsequent stages, and sequence latents for generative training are precomputed once and stored in fp16.

\subsubsection{Stage 2: score model and interpolation path}

\paragraph{Flow-matching score model.}
A UNet with base widths $192/256/384$ over resolutions $32/16/8$, embedding dimension 256, self-attention at $16^2$ and $8^2$, and 24.9M parameters is trained with conditional flow matching between a standard normal at $t{=}0$ and the latent-frame distribution at $t{=}1$. The regression target is $x_1 - x_0$ along the linear path, under an unweighted mean-squared-error loss. Time is sampled per example from an equal mixture of $\mathcal{U}(0,1)$ and $\mathcal{U}(0.8,1)$, which concentrates capacity in the near-data regime that the interpolation path queries. One epoch is one pass over all training-split latent frames, and the EMA weights are the ones used downstream.

\paragraph{Interpolation path.}
Given two consecutive frame latents $x_0, x_1$, both endpoints are \emph{lifted} by integrating the score model's probability-flow ODE for 10 Euler steps of size 0.02, from $t{=}1$ to $t_{\mathrm{lift}}{=}0.8$, which is 20\% of the way toward the noise distribution.
The interpolator acts in the lifted space and is realized as a small UNet with four levels of width 64, no attention, and 1.78M parameters.
Similar to the step size embedding of Stage~3, $s$ enters the UNet through a Gaussian Fourier time embedding.
The path state at $s$ is then \emph{denoised} by integrating the same ODE back from $t_{\mathrm{lift}}$ to $t{=}1$ with 10 Euler steps of 0.02.

The interpolator network is trained with the metric-energy objective of the main text, with the correction weight $\alpha$ warmed linearly from 0 to 1 over the first 10\% of steps. Training pairs are all consecutive latent transitions of the training split, with one path time $s \sim \mathcal{U}(0,1)$ per pair, and the score-induced metric is evaluated at the lifted time $r_m = t_{\mathrm{lift}} = 0.8$.  The \emph{linear} ablation, reported as ``Ours (linear)'' in Tables~\ref{tab:vidode} and~\ref{tab:vidode_extrap}, replaces this construction by plain linear interpolation of the latents and is identical in everything else.

\subsubsection{Stage 3: latent video-generation model}
\paragraph{Architecture.}
The generative model is a conditional velocity field $v_\theta(x, \Delta t, z)$ over frame latents $x \in \mathbb{R}^{4\times32\times32}$, implemented as a UNet with four resolutions $32/16/8/4$, base widths $128/192/256/320$, two pre-activation AdaGN residual blocks per level, multi-head self-attention with 8 heads at $8^2$, $4^2$, and the bottleneck, and 23.35M parameters. The step size $\Delta t$ enters through Gaussian Fourier features, 64 random frequencies at $\sigma{=}2$, followed by a two-layer MLP into the 256-d embedding that modulates every AdaGN. The conditioning variable $z$ is a \emph{spatial} latent grid of shape $8{\times}32{\times}32$, channel-concatenated with $x$ at the network input. The model has no absolute-time input.

Two convolutional encoders map channel-stacked frame latents to per-position diagonal Gaussians over $z$. Their trunk is a $3{\times}3$ convolution to width 128, two residual blocks, a strided convolution to width 192, two residual blocks, nearest-neighbor upsampling back to 128, one residual block, and a $1{\times}1$ head to $(\mu, \log\sigma^2)$, where $\log\sigma^2$ is clamped to $[-8, 4]$ and initialized at $-2$. Input frames are channel-concatenated in ascending temporal order, which gives $5 \cdot 4 = 20$ input channels for the prior and $40$ for the posterior. The interpolation context is stacked the same way, so the temporal spacing of the gaps is implicit in the protocol-specific weights rather than encoded explicitly. The posterior sees the full training window $x_{1:10}$ and has 2.71M parameters, and the prior sees exactly the frames observed under the respective protocol and has 2.68M parameters, the five odd-indexed frames for interpolation and the first five frames for extrapolation. Interpolation and extrapolation use separately trained models, identical in every hyperparameter except the fixed index set the prior encoder observes.

\begin{table}[t]
\centering
\caption{Training details of the KTH generative model. Optimizer settings are in Table~\ref{tab:kth_optim}.}
\label{tab:kth_stage3}
\begin{tabular}{ll}
\toprule
Quantity & Value \\
\midrule
\multicolumn{2}{l}{\emph{Sampling of $(t,\Delta t)$}} \\
Training windows & all non-overlapping length-10 windows (stride 10) \\
Queries per batch & 16 windows $\times$ 2 queries \\
Segment index & $k \sim \mathcal{U}\{1,\dots,9\}$ \\
Intra-segment offset & $t \sim \mathcal{U}(0,1)$, query time $\tau = k+t$ \\
Step size & $\Delta t \sim \mathcal{U}(0,1]$, clamped to $\tau + \Delta t \le 10$ \\
Tangent probability & $p_{\mathrm{tan}} = 0.2$, then $\Delta t = 0$ \\
\midrule
\multicolumn{2}{l}{\emph{Noise augmentation}} \\
Scales & 10 log-spaced values in $[0.01, 0.5]$ \\
Contraction rate & $\lambda = 10$ \\
\bottomrule
\end{tabular}
\end{table}

\paragraph{Training signal.}
Table~\ref{tab:kth_stage3} reports the details of how $(t,\Delta t)$ is sampled during training and specifics of the transverse noise augmentation. Training windows are all non-overlapping length-10 windows of every training video's latent trajectory, one epoch is one shuffled pass over all windows, and no data augmentation is used in Stages~2 and~3. The query time $\tau$ runs on the piecewise path through \emph{consecutive} frames, and the step is drawn globally, so it may cross segment boundaries and is clamped at the window end. Tangent targets are computed by forward-mode differentiation through the path construction, and the latents have unit per-channel variance, which fixes the scale of the tabulated noise values.

\paragraph{Optimization.}
Optimizer settings are those of Table~\ref{tab:kth_optim}, where weight decay applies only to weight tensors of rank $\ge 2$ of the field and leaves biases, normalization gains, and the encoders undecayed.

\subsubsection{Inference}
Given the observed frames, we draw a single latent sample $z$ from the prior with a fixed random seed and generate the full trajectory with one explicit Euler rollout of the secant field,
\[
x_{s+1} = x_s + h\, v_\theta(x_s, h, z),
\qquad
h = 0.25,
\]
which is $4\times$ temporal supersampling relative to the native frame rate. The rollout starts from the protocol's start frame, frame 1 for interpolation and frame 5 for extrapolation. Scored frames are read off at the corresponding integer times and decoded with the frozen autoencoder. For the qualitative figures (Figures~\ref{fig:kth_interp} and~\ref{fig:kth_extrap}) the identical procedure is run with $h = 0.5$, the common display grid on which the Vid-ODE row is queried (see below).

\subsubsection{Vid-ODE baseline}
We retrain Vid-ODE at $128{\times}128$ from the official implementation with its published defaults, separately for the interpolation and extrapolation protocols, on exactly the same data as ours. Those defaults are batch size 8, Adamax at learning rate $10^{-3}$ with 0.99 per-epoch decay, 500 epochs, adversarial weight $3{\times}10^{-3}$, \texttt{dopri5} solvers, and backward-time encoding. The retrained interpolation model closely reproduces the KTH metrics reported by \citet{park2021vidode}, at SSIM 0.912 against 0.911 published and LPIPS 0.049 against 0.048, and it exceeds the published PSNR by 1.3\,dB, 33.08 against 31.77 (cf.\ Table~\ref{tab:vidode}), which confirms a faithful reproduction. Evaluation uses its own tester and metric code unmodified, on the final end-of-training checkpoint.

For the qualitative figures, its continuous decoder is queried in a single recursive pass over its trained time span on the shared 0.5-frame display grid, with all its other settings untouched. This is $4\times$ its input frame rate for interpolation and $2\times$ for extrapolation, the densest continuous-sampling setting demonstrated in the Vid-ODE paper itself, whose appendix Fig.~18 generates 20\,FPS videos from 5\,FPS inputs.
\subsection{PDE Data}
\label{data:pde}

We use two synthetic PDE datasets to evaluate continuous-time generation on scientific spatiotemporal fields.
Both datasets consist of \(64\times64\) scalar fields with \(64\) stored time frames per trajectory, and all PDE experiments in the main paper use a set of \(128\) trajectories.
Training subsequences have length \(16\).
The simulation settings of both datasets are collected in Table~\ref{tab:pde_sim}.

\begin{table}[t]
\centering
\caption{Simulation settings of the two PDE datasets. Both simulators are periodic and store 64 frames per trajectory.}
\label{tab:pde_sim}
\resizebox{\linewidth}{!}{
\begin{tabular}{ll}
\toprule
Setting & Value \\
\midrule
\multicolumn{2}{l}{\emph{Gray--Scott reaction--diffusion}} \\
Domain and grid & periodic \(\Omega=[0,1]^2\), \(64\times64\) \\
Diffusion coefficients & \(D_a=2\cdot 10^{-5}\), \(D_b=10^{-5}\) \\
Feed and kill rate & \(F=0.018\), \(K=0.051\) \\
Integration & spectral half-steps for diffusion, fourth-order Runge--Kutta for the reaction \\
Internal time step & \(\Delta t=1.0\) \\
Internal steps per stored frame & 50 \\
Burn-in & 300 steps \\
Stored field & species \(a\) only \\
Training windows & length 16 \\
\midrule
\multicolumn{2}{l}{\emph{Navier--Stokes}} \\
Domain and simulation grid & periodic \(\Omega=[0,2\pi]^2\), \(96\times96\) \\
Observed field & centered \(64\times64\) crop of the vorticity \\
Viscosity & \(\nu=10^{-3}\), no linear drag \\
Forcing & sampled per trajectory, fixed in time \\
Internal time step & 0.01 \\
Final time & 32.0 \\
Burn-in & 1000 solver steps \\
Training windows & length 16 \\
\bottomrule
\end{tabular}
}
\end{table}

\paragraph{Gray--Scott Reaction--Diffusion.}
We simulate the two-species Gray--Scott reaction--diffusion system on the periodic unit square \(\Omega=[0,1]^2\),
\[
\partial_t a
=
D_a\Delta a
-
ab^2
+
F(1-a),
\]
\[
\partial_t b
=
D_b\Delta b
+
ab^2
-
(F+K)b,
\]
where \(a(t,x)\) and \(b(t,x)\) are concentration fields and the coefficients are those of Table~\ref{tab:pde_sim}.

Initial conditions are sampled from randomized localized perturbations.
Let \(c_j\in[0,1]^2\) denote randomly sampled cluster centers, with periodic distance
\[
d_{\mathrm{per}}(x,c_j)
=
\min_{m\in\mathbb Z^2}
\|x-c_j+m\|_2.
\]
For each cluster, we sample an amplitude \(A_j\), anisotropic widths \(\sigma_{j,1},\sigma_{j,2}\), and a random orientation.
Writing \(R_j\) for the corresponding rotation matrix and
\[
\Sigma_j
=
R_j
\begin{pmatrix}
\sigma_{j,1}^2 & 0\\
0 & \sigma_{j,2}^2
\end{pmatrix}
R_j^\top,
\]
the initial \(b\)-field is formed from a mixture of periodic Gaussian blobs,
\[
b_0(x)
=
\sum_{j=1}^{M}
A_j
\exp\!\left(
-\frac12
\delta_j(x)^\top
\Sigma_j^{-1}
\delta_j(x)
\right)
+
\xi_b(x),
\]
where \(\delta_j(x)\) is the shortest periodic displacement from \(c_j\) to \(x\), and \(\xi_b\) is small Gaussian pixel noise.
The \(a\)-field is initialized near one and anticorrelated with \(b_0\),
\[
a_0(x)
=
1-\rho\, b_0(x)+\xi_a(x),
\]
with small Gaussian noise \(\xi_a\).
Both fields are clipped to a bounded concentration range during simulation.
We store only the \(a\)-field, so each observed trajectory is
\[
X=(a(t_0),a(t_1),\ldots,a(t_{63})).
\]

For the distributional rollout experiment in Table~\ref{tab:gray_scott_dist}, we use the partial-observation structure of this dataset explicitly. The learned models observe only the initial \(A\)-field. To construct a simulator reference distribution conditioned on the same observation, we keep this initial \(A\)-field fixed, resample the hidden initial \(B\)-field from the same clustered Gaussian initial-condition generator, and integrate the full Gray--Scott system forward. This produces multiple plausible \(A\)-field futures for the same observed initial condition. We use \(32\) observed starts. For each start, the hidden field is resampled \(4\) times, which yields \(128\) reference futures, and our model draws \(4\) latent-conditioned rollouts per start, likewise yielding \(128\) samples, while the FNO baseline produces a single rollout per start.

The FNO baseline is trained directly on the normalized \(A\)-field sequences as a deterministic one-step predictor, with the following configuration.
\begin{itemize}
\item Fourier Neural Operator with \((16,16)\) Fourier modes, hidden width \(64\), four spectral layers, and a grid positional encoding, giving 2.41M parameters when complex spectral weights are counted once.
\item AdamW at learning rate \(10^{-3}\) and weight decay \(10^{-6}\), batch size \(32\), and a mean-squared-error loss.
\item Input of four consecutive \(A\)-frames stacked along the channel dimension, with the next \(A\)-frame as target.
\item At evaluation time, initialization with the first four ground-truth frames and autoregressive rollout with a sliding four-frame input window.
\end{itemize}

\paragraph{Navier--Stokes.}
We generate two-dimensional incompressible Navier--Stokes trajectories using \texttt{exponax}~\citep{koehler2024apebench} in vorticity form on the periodic domain \(\Omega=[0,2\pi]^2\),
\[
\partial_t \omega
+
u\cdot\nabla \omega
=
\nu \Delta \omega
+
f,
\]
where \(\omega(t,x)\) is scalar vorticity, \(u(t,x)\in\mathbb R^2\) is the incompressible velocity field, and
\[
\nabla\cdot u=0.
\]
The velocity is recovered from vorticity through the stream function \(\psi\),
\[
u
=
\nabla^\perp \psi
=
(\partial_y\psi,-\partial_x\psi),
\qquad
\Delta\psi=\omega.
\]
The viscosity and the remaining simulator settings are those of Table~\ref{tab:pde_sim}.
The observed data are obtained by cropping the centered \(64\times64\) region from the full vorticity solution, which makes the visible crop an open subsystem influenced by surrounding, unobserved flow.
For numerical stability, simulations are run on a refined grid and then mapped back to the base resolution before cropping.

Initial vorticity fields and forcing fields are sampled as Gaussian random fields in Fourier space.
For a scalar field \(g\), we sample complex Fourier coefficients
\[
\widehat g(k)
=
A\,
\frac{\zeta_k}{(\|k\|^2+\tau^2)^{\alpha/2}},
\qquad
k\in\mathbb Z^2,
\]
with Hermitian symmetry imposed so that \(g\) is real-valued.
Here \(\zeta_k\) are independent complex standard Gaussian coefficients up to the reality constraint, \(A\) controls the amplitude, \(\alpha\) controls spectral decay, and \(\tau\) sets a correlation scale.
The two fields use the parameters of Table~\ref{tab:pde_grf}, and the forcing is fixed over time within each trajectory.

\begin{table}[h]
\centering
\caption{Gaussian random field parameters of the Navier--Stokes initial vorticity and forcing.}
\label{tab:pde_grf}
\begin{tabular}{lcccc}
\toprule
Field & \(\alpha\) & \(\tau\) & \(A\) & Fourier modes \\
\midrule
Initial vorticity & 2.2 & 4.0 & 0.1 & unrestricted \\
Forcing & 1.8 & 2.5 & 2.0 & wavenumbers 3 to 14 \\
\bottomrule
\end{tabular}
\end{table}

For the robustness experiment in the main text, adaptive integration uses a Dormand--Prince solver over a sweep of relative tolerances \(\texttt{rtol}\in\{10^{-1},10^{-2},10^{-3},10^{-4},10^{-5}\}\) with \(\texttt{atol}=0.1\cdot\texttt{rtol}\), evaluated on a uniform 64-point output grid.
The number of function evaluations is not fixed manually but chosen automatically by the adaptive solver for each tolerance, so the reported NFE values reflect the solver's adaptive effort required to satisfy the specified error tolerances.

\subsection{Fine-Time Ground-Truth Reference}
\label{app:gs_fine}

The plausibility experiments of Section~\ref{sec:plausibility} score interpolation paths and trained fields against the true states of the system between the observed frames. We regenerate 8 fresh Gray--Scott trajectories with the paper's simulator and parameters (Appendix~\ref{data:pde}) on seeds unseen by any trained model, storing every 5th internal solver step instead of every 50th. The training grid is thereby an exact subset of the fine grid, at a fine spacing of 5 time units against \(\Delta_{\mathrm{phys}}=50\) per training segment, with every coarse frame coinciding with one fine frame.

The frozen prior and interpolator are evaluated with endpoints lifted at metric time \(r_m=0.9\), and all states are normalized with the interpolator's training statistics. For each of the 63 segments per trajectory and the interior offsets \(s=j/10\), \(j=1,\dots,9\), we score every path against the true fine state at the matched time by the relative \(L_2\) error, the cosine between the path velocity and the true velocity, and the spectral diagnostic of the main text. The last one is evaluated on denormalized fields. Writing \(\hat x(k)\) for the two-dimensional Fourier transform of a field \(x\) and \(\bar x\) for its spatial mean, we average the power spectrum over \(B=32\) rings \(K_b\) of equal width in the radial wavenumber \(\|k\|\),
\[
S_b(x)
=
\frac{1}{|K_b|}\sum_{k\in K_b}\bigl|\hat x(k)\bigr|^2 ,
\]
and compare the path state \(x\) with the true state \(y\) by
\[
\mathrm{spec}(x,y)
=
\bigl|\bar x-\bar y\bigr|
+
\frac{1}{B}\sum_{b=1}^{B}
\bigl|
\log S_b(x)-\log S_b(y)
\bigr| .
\]
The first term tracks the total concentration, and the second compares how much structure sits at each length scale. Values are averaged over segments and offsets and reported as mean \(\pm\) std over the 8 trajectories, with full results in Appendix~\ref{app:res_plausibility}.

\subsection{Wider Frame Spacing: Supervision Comparison}
\label{app:x1w}

\paragraph{Models.}
The generative fields are trained at the two wider segment strides \(w\in\{15,20\}\) fine steps, where \(w=10\) is the training grid and the main text reports \(w=20\). All three share one 3.99M-parameter UNet secant backbone without latent conditioning, the same normalization, computed from the sixteen observed frames, and the same inference procedure. Only the supervision differs. The score-target model regresses the tangent and secant targets of the frozen interpolator with the transverse correction of Appendix~\ref{app:full_training_alg} at \(\lambda=20\), and the linear-target model is identical except that its targets come from linear interpolation. The Neural ODE reads the same backbone as an autonomous field at \(\Delta t=0\) and is trained by integrating each segment with a differentiable ten-step RK4 and regressing the endpoint. All three use AdamW at learning rate \(10^{-4}\) and the same seed.

\paragraph{Baseline validity and evaluation.}
To rule out an underfitted baseline, we require that the Neural ODE fits its training endpoints at least as accurately as the best simulation-free model fits the evaluation endpoints. This holds with a margin of 13 to 21 times, so differences in intermediate quality cannot be attributed to a poorly trained baseline. The data are the 16-frame window of a held-out fine trajectory never seen by the frozen prior or interpolator, and the \(w-1\) fine states inside each of the 15 segments are evaluation-only. Evaluation is teacher-forced per segment, so each segment starts at the ground-truth left frame, integrates the \(\Delta t=0\) field with fixed-step RK4, and is scored with the metrics of Appendix~\ref{app:gs_fine} plus the endpoint error, averaged over the 15 segments.

\subsection{Fully Observed Gray--Scott and Physics Supervision}
\label{app:x2}

\paragraph{Setup.}
The physics experiments use a two-channel Gray--Scott variant that stores both species, which makes the governing PDE residual computable. The setup comprises the following components.
\begin{itemize}
\item 128 fresh training trajectories at the training stride.
\item The fine-reference trajectories of Appendix~\ref{app:gs_fine} with both channels.
\item Per-channel normalization from the training split.
\item A two-channel flow-matching prior.
\item A spectral implementation of the Gray--Scott right-hand side \(F_{\mathrm{GS}}\), with Laplacian and reaction terms matching the data simulator exactly.
\end{itemize}

\paragraph{Supervision variants.}
The interpolator parameterization and architecture are as in the other experiments, a 1.78M-parameter correction network on a lifted linear base with \(t(1-t)\) correction for the score variant. The three variants differ only in the supervision.
\begin{itemize}
\item \emph{V-score} trains on the metric energy only.
\item \emph{V-phys} applies the same correction parameterization directly in data space, without lifting, and trains the path to satisfy the governing equation instead of the metric energy. A segment time \(s\in[0,1]\) covers \(\Delta_{\mathrm{phys}}=50\) physical time units, so the physical velocity along the path is \(\Delta_{\mathrm{phys}}^{-1}\partial_s\gamma(s)\), obtained by a Jacobian-vector product through the correction network, and the loss is the mean squared PDE residual on denormalized states,
\[
\mathcal L_{\mathrm{phys}}
=
\mathbb E_s
\bigl\|
\Delta_{\mathrm{phys}}^{-1}\partial_s\gamma(s)-F_{\mathrm{GS}}(\gamma(s))
\bigr\|^2 ,
\]
where \(F_{\mathrm{GS}}\) collects the diffusion and reaction terms on the right-hand side of the two Gray--Scott equations of Appendix~\ref{data:pde}. The residual vanishes exactly when the path solves the PDE, so this variant asks for a dynamically consistent path rather than a geometrically short one.
\item \emph{V-lin} is the linear path and requires no training.
\end{itemize}
We evaluate as in Appendix~\ref{app:gs_fine} on all 8 trajectories and additionally report the residual RMS in physical units. The residual of the true fine trajectory itself, taken as central differences at the fine spacing, is \(1.79\cdot10^{-5}\) and is quoted as the attainable floor. The comparison of the three variants is reported in Appendix~\ref{app:res_plausibility}.

\paragraph{Physics refinement at inference.}
Per segment, the path is represented by 11 samples at \(s=j/10\) including the endpoints, and the 9 interior states are optimized in denormalized space with Adam at learning rate \(5\cdot10^{-3}\) on the mean squared PDE residual, with the path velocity taken as the central difference in \(s\) and the endpoints held fixed. The operation involves only FFTs and pointwise arithmetic, and no network is in the loop. Refinement costs 2 and 9 seconds per 63-segment trajectory on one GPU for \(K=500\) and \(K=2000\) descent steps. Full results for intermediate budgets are reported in Appendix~\ref{app:res_plausibility}.

\subsection{Molecular Dynamics}
\label{data:md}

\paragraph{Data.}
We use the MD17~\citep{chmiela2017machine} ethanol trajectory, loaded through the \texttt{torch\_geometric} interface in its original frame order. Observed frames are taken every \(100\) native MD steps. The \(99\) raw frames inside each segment are seen by no model and serve as the ground truth between the frames, as in the fine-time reference of Appendix~\ref{app:gs_fine}. The flow-matching prior and the interpolator are fit to the first \(90\%\) of the trajectory. The three generative fields are fit to a single four-frame window of it, while the predictability numbers at the end of this subsection are measured on the held-out \(10\%\). Each molecular state is represented by centered atomic coordinates
\[
x_k = (r_{k,1},\ldots,r_{k,N})\in\mathbb R^{N\times 3},
\qquad
\sum_{i=1}^N r_{k,i}=0,
\qquad
N=9,
\]
together with atom identities and connectivity that are fixed across the trajectory, so centering removes global translation from every frame.

\paragraph{Prior knowledge in the interpolator.}
The physics experiment of Appendix~\ref{app:x2} adds the governing equation to the interpolator. Molecular trajectories carry a different kind of prior knowledge, and we add it here in the same way.

First, we observe that the per-segment interpolation objective of Appendix~\ref{app:interp_training} leads to abrupt changes in the velocity at segment borders, i.e. at observed frames. We hypothesize that the MD setting amplifies this because of the severely underconstrained motion between endpoints.
A physical trajectory evolves smoothly, and on the raw trajectory the angle between the incoming and the outgoing velocity is \(8.9^\circ\), against \(107.2^\circ\) for the two linear segments that an endpoint-only path follows.

The second piece concerns the speed. In the image experiments only the metric speed carries meaning, because distances in pixel space do not, whereas atomic coordinates measure physical distance. Therefore, we can also ask the path to move at a roughly uniform Euclidean speed. For both, we will augment the interpolation objective similar to the experiment in Appendix~\ref{app:x2}.

\paragraph{The augmented interpolation objective.}
We keep the parameterization of Appendix~\ref{app:interp_training}. For the segment between two observed frames \(x_k\) and \(x_{k+1}\), we write \(\bar\gamma^{\omega,k}_t:=\bar\gamma^\omega_t(x_k,x_{k+1};\alpha)\) for the lifted path, where the correction network \(\varphi_\omega\) is shared by all segments. Its velocity is
\[
\dot{\bar\gamma}^{\omega,k}_t
=
\bar x_{k+1}-\bar x_k
+
\alpha(1-2t)\,\varphi_\omega(\bar x_k,\bar x_{k+1},t)
+
\underbrace{\alpha\,t(1-t)\,\partial_t\varphi_\omega(\bar x_k,\bar x_{k+1},t)}_{=\,0\ \text{at}\ t\in\{0,1\}} .
\]
Because the last term vanishes at both ends of a segment, the two one-sided velocities at the frame shared by segments \(k\) and \(k+1\) reduce to
\[
\dot{\bar\gamma}^{\omega,k}_1
=
\bar x_{k+1}-\bar x_k-\alpha\,\varphi_\omega(\bar x_k,\bar x_{k+1},1),
\qquad
\dot{\bar\gamma}^{\omega,k+1}_0
=
\bar x_{k+2}-\bar x_{k+1}+\alpha\,\varphi_\omega(\bar x_{k+1},\bar x_{k+2},0),
\]
which in general differ. To remove that kink, we introduce a junction penalty that drives the two one-sided velocities together at every shared frame,
\[
\mathcal L_{\mathrm{junc}}(\omega)
=
\mathbb E_k
\bigl\|
\dot{\bar\gamma}^{\omega,k}_1-\dot{\bar\gamma}^{\omega,k+1}_0
\bigr\|_2^2 .
\]

On its own, this penalty has a trivial minimizer. The cheapest way to make two velocities agree is to send both to zero, which yields a path that stalls at every observed frame and sprints in between. We therefore introduce a second penalty that keeps the path moving. Let
\[
\rho_k(t)
=
\frac{\bigl\|\dot{\bar\gamma}^{\omega,k}_t\bigr\|}{\|\bar x_{k+1}-\bar x_k\|}
\]
denote the instantaneous speed in units of the mean speed of the same segment. The speed penalty acts only below a floor \(\rho_\star\),
\[
\mathcal L_{\mathrm{speed}}(\omega)
=
\mathbb E_{k,\;t\sim\mathcal U[0,1]}
\bigl[
\max\bigl(0,\rho_\star-\rho_k(t)\bigr)^2
\bigr],
\]
so it never slows the path down and only lifts its stalled parts.

Two choices in this penalty matter. First, the reference \(\|\bar x_{k+1}-\bar x_k\|\) is fixed by the observed endpoints. The optimizer therefore cannot meet the floor by inflating \(\varphi_\omega\) until the lifted path leaves the region where the frozen score model is valid. Second, both norms are Euclidean coordinate norms rather than metric norms, which is the speed prior stated above. We evaluate them in lifted space, where the short denoising flow is close to an isometry, at roughly a tenth of the cost of the clean-space speed. The full objective is
\[
\mathcal L(\omega)
=
\mathcal L_{\mathrm{geo}}(\omega)
+
\lambda_j\,\mathcal L_{\mathrm{junc}}(\omega)
+
\lambda_s\,\mathcal L_{\mathrm{speed}}(\omega),
\]
where \(\mathcal L_{\mathrm{geo}}\) is the metric energy of Appendix~\ref{app:interp_training}. Setting \(\lambda_j=\lambda_s=0\) recovers that objective exactly. We use \(\lambda_j=1\), \(\lambda_s=12\) and \(\rho_\star=0.9\).

\paragraph{Exactness at the observed frames.}
The factor \(t(1-t)\) vanishes at both ends for every \(\omega\), so the path passes through all observed frames exactly, whatever the correction network does. The endpoint values \(\varphi_\omega(\cdot,0)\) and \(\varphi_\omega(\cdot,1)\) that appear in the two one-sided velocities are the degrees of freedom that this factor leaves unconstrained on path positions. Velocity continuity is therefore obtained without moving the path away from the observed frames.

Continuity also carries over from lifted to clean space. The two segments meeting at a shared frame denoise the identical lifted point, so their clean velocities are the same denoising Jacobian applied to the two lifted velocities above. Every internal coordinate \(q\) inherits the property through \(\dot q=\nabla q\cdot\dot\gamma\).

\begin{table}[t]
\centering
\caption{Training configuration of the augmented molecular interpolator. The correction network is the EGNN of Table~\ref{tab:main_hparams}, evaluated on the paired graph of the two lifted endpoints.}
\label{tab:md_interp}
\resizebox{\linewidth}{!}{
\begin{tabular}{ll}
\toprule
Setting & Value \\
\midrule
Training item & consecutive frame triplet, two adjacent segments sharing one frame \\
Correction network & \(735.5\)k-parameter EGNN \\
Lifting & \(10\) Euler probability-flow steps at step size \(0.01\), metric time \(r_m=0.9\) \\
Batch size & \(64\) triplets \\
Optimizer & AdamW at learning rate \(10^{-3}\) \\
Schedule & \(1\%\) linear warmup, cosine decay to \(2\%\) of the peak rate \\
Budget & \(60\,000\) steps \\
Gradient clip & \(1.0\) \\
Seed & \(0\) \\
Correction ramp & \(\alpha\) from \(0\) to \(1\) over the first \(20\%\) of steps \\
Speed penalty & switched on once the ramp has finished \\
Speed evaluation & four stratified interpolation times per segment \\
\bottomrule
\end{tabular}
}
\end{table}

\paragraph{Training the interpolator.}
Training items are consecutive frame triplets of the training split. Every item is therefore a pair of adjacent segments sharing one frame, so the junction penalty has a shared frame to act on at every step. The remaining settings are listed in Table~\ref{tab:md_interp}. The correction strength \(\alpha\) is ramped as in Appendix~\ref{app:interp_training}. We switch the speed penalty on only once that ramp has finished, so that junctions exist before they are smoothed. The tangent, the metric-vector product and the path speed are all obtained by Jacobian-vector products.

\paragraph{Effect of the added terms.}
We measure continuity on the velocity vector. At a shared frame we take the clean-space velocity at \(t=1\) of one segment and at \(t=0\) of the next, both by Jacobian-vector products through the denoising flow, and report the angle between them. Averaged over \(32\) shared frames from eight six-frame stretches at the deployed stride of \(100\), the interpolator trained without the added penalties leaves a kink of \(92.5^\circ\). This is close to the \(107.2^\circ\) of the linear segments it corrects. The augmented interpolator leaves \(28.0^\circ\), against \(8.9^\circ\) for the true trajectory. Over the same stretches the metric energy drops from \(2247\) to \(930\), so the continuity is not bought by leaving the score manifold.

\paragraph{The three generative fields.}
The molecular demonstration from the main text trains three unconditional velocity fields on one fixed trajectory of four consecutive observed frames. This is the first length-four window of the training split and covers raw frames \(0\) to \(300\) in three segments. All three fields share one \(1.46\)M-parameter EGNN with hidden width \(128\) and twelve message-passing layers, without latent conditioning, so the field is autonomous. They also share one optimizer, AdamW at learning rate \(2\cdot10^{-4}\) with batch \(256\) for \(8000\) steps at seed \(0\), and one inference procedure. Only the training signal differs. Since the coordinate updates are built from relative displacements, all three fields are translation- and rotation-equivariant.

The score-target field regresses the tangent and secant targets of the frozen augmented interpolator, drawing the tangent target with probability \(0.2\) and otherwise a secant step of up to two frame intervals. Its targets carry the transverse correction of Appendix~\ref{app:full_training_alg} at \(\lambda=1\), with ten log-spaced noise scales in \([0.01,0.1]\,\text{\AA}\). The linear-target field uses the identical pipeline with the linear path between observed frames as its target source.

The Neural ODE integrates each segment from its ground-truth left frame over one frame interval with a differentiable ten-step RK4 and regresses the endpoint only. We add the Jacobian-Frobenius regularization that is standard for this kind of training~\citep{finlay2020train} at weight \(10^{-4}\), estimated with one Hutchinson probe per solver stage.
Even with this stabilization technique, we faced instability during training, reaching the optimum early and degrading afterwards. Therefore, we report only the best-loss results for the Neural ODE.

The true trajectory carries small-scale thermal jitter that likely no unconditional path can reproduce without knowing the external forcing that produced that trajectory.
Closing this gap probably needs conditioning on the forces that produced the particular trajectory, which we leave to future work.

\subsection{Inference and Figure Details}
\label{app:inference_details}

Unless stated otherwise, all reported results use the raw weights of the trained generative vector fields and the corresponding frozen interpolation prior, and rollouts integrate the learned dynamics without any re-anchoring at observed states. KTH inference is described separately in Appendix~\ref{data:kth}.

\paragraph{2D example (Figures~\ref{fig:toy_example} and~\ref{fig:app_toy}).}
Panel~(a) of Figure~\ref{fig:toy_example} overlays samples of the pretrained score model, obtained by integrating its generative flow with 100 uniform steps from 2048 Gaussian noise samples, with dataset sequences and one ground-truth loop. The conditional rollouts in panels (b)--(d) integrate the learned vector field with a fixed-step solver using 800 uniform steps over the full loop duration \(T=8\), or 400 steps for Figure~\ref{fig:app_toy}, and draw 32 latent samples from the observation-conditioned prior per prefix. Panel~(e) shows the latent Gaussians after projection onto the two leading principal components of the latent covariance.

\paragraph{2D global fields (Figure~\ref{fig:toy_fields}).}
The conditioned panels evaluate the \(\Delta t=0\) field on a \(61\times61\) grid over \([-1.05,1.05]^2\) for the gray streamlines at density 1.4, and on a \(26\times26\) grid for the unit-arrow quiver. Each arrow is colored by the transverse inflow \(c(x)=\langle\hat v(x),-\hat e(x)\rangle\) toward the conditioned reference path \(\Gamma_z\) of the frozen interpolator, which is drawn from 1600 path samples, under a diverging colormap with symmetric limits at the 95th percentile of \(|c|\) and with arrows within 0.01 of the path masked. The path \(\Gamma_z\) is overlaid in black with the loop nodes in green. The conditioning latent is the posterior mean of \(q(z\mid x)\) under each model's own posterior encoder, and the latent-free panel uses the (score, corrected, no-latent) ablation model with the same loop as reference. The setting panel shows 2048 samples of the pretrained score model obtained by integrating its generative flow with 100 uniform steps, its score field at the interpolator's lift time as gray unit arrows on the same \(26\times26\) grid, all candidate conditioning nodes in black, the active loop in green, and the linear path between its nodes dashed.

\paragraph{Transverse contraction measurements (Table~\ref{tab:contraction}).}
For the 2D example, the pointwise eigenvalue map evaluates \(\lambda_\perp(x)=\hat n^\top\tfrac12(Dv+Dv^\top)\hat n\) with exact autograd Jacobians of the \(\Delta t=0\) field on a \(141\times141\) grid, masked to the tube \(\|x-\Gamma_z\|\le0.15\), which is about three times the largest training noise scale. Here \(\hat n\) is the \(90^\circ\)-rotated tangent of \(\Gamma_z\) at the nearest of its 1600 samples, and \(z\) is the posterior mean of the all-inner loop. Finite rates use 64 anchors placed uniformly along \(\Gamma_z\) and perturbed along the local path normal with amplitudes \(\{0.01,0.03,0.05\}\). Both the perturbed and the unperturbed anchor are integrated under RK4 with step 0.005 for \(t\in[0,0.5]\), and the rate is the least-squares slope of \(\log e(t)\) on \(t\in[0,0.2]\), where \(e(t)\) is the separation between the two rollouts, reported as the median over anchors and amplitudes.

The Navier--Stokes measurement applies the identical methodology to the frozen single-sequence checkpoints of Section~\ref{sec:robustness_exp}, using 32 anchor states along the frozen score-interpolation path, 8 Gaussian directions per anchor projected orthogonally to the local field direction, and amplitudes \(\{0.01,0.05,0.1\}\) in normalized units. This gives 768 perturbed rollouts per model, integrated with RK4 at step 0.01 over \(t\in[0,0.5]\). Comparing the two rollouts of the same field isolates the transverse response from its shared advection error, and we report the median of the fitted rates and the fraction of contracting perturbations. For the corrected model the median is the same to two decimals at each of the three amplitudes, at $-$0.90, so one number summarizes the measurement.

\paragraph{Rollout stability (Figures~\ref{fig:contracting} and~\ref{fig:app_robustness}).}
Figure~\ref{fig:contracting} uses three inference regimes.
\begin{itemize}
\item[(a)] Direct secant steps with \(h\in\{0.01,0.1,0.25,0.5\}\).
\item[(b)] Fixed-step RK4 integration with \(h\in\{0.01,0.1,0.25,0.5,1\}\).
\item[(c)] Adaptive Dormand--Prince integration over a sweep of tolerances, where the number of function evaluations is chosen by the solver as described in Appendix~\ref{data:pde}.
\end{itemize}
Reported values are LPIPS against the reference trajectory, averaged over the 64-frame rollout excluding the initial frame. The qualitative comparison in Figure~\ref{fig:app_robustness} uses the secant rollout with \(h=0.5\), and for display, frames are normalized per time step to the 1\%/99\% quantile range of the ground-truth frame.

\paragraph{Gray--Scott (Table~\ref{tab:gray_scott_dist} and Figure~\ref{fig:gray_scott_appendix}).}
Rollouts integrate the \(h=0\) slice of the learned field with fixed-step RK4 using 10 substeps per unit time up to \(T=256\), with latent variables drawn from the prior conditioned on the first observed frame. The sliced Wasserstein distances in Table~\ref{tab:gray_scott_dist} use 256 random unit-norm projections of the flattened, denormalized \(64\times64\) fields, computed independently at each horizon with a fixed seed. The true-against-true reference is a single fixed even/odd split of the 128 simulator futures into two sets of 64, and the reported value is the ratio of the model-against-simulator distance to this reference. For the deterministic FNO baseline, which contributes 32 rollouts, sorted projections are truncated to the smaller sample count before averaging. The qualitative figure shows two prior samples for one test start, with frames read off at \(t\in\{0,8,16,32,64,128,256\}\).

\paragraph{Navier--Stokes sub-frame dynamics (Figure~\ref{fig:ns_interp}).}
For a randomly drawn test sequence, both models are conditioned on the posterior latent of that sequence and integrated with fixed-step RK4 using 50 substeps per frame interval over \(t\in[0,2]\), which is two frame intervals. The enstrophy diagnostic is computed as \(E(t)=\tfrac12\,\overline{\omega^2}\), where \(\omega=\partial_x v-\partial_y u\) is evaluated by central differences under periodic shifts. Since the stored field is single-channel, the same stencil is applied to the stored field itself as a curl proxy. The right panel reports the absolute change \(|E(t_{i+1})-E(t_i)|\) between consecutive solver outputs, normalized by the mean absolute change of the linear-target model.

\paragraph{Molecular dynamics (Figure~\ref{fig:md_panels}).}
All three fields are rolled out freely from the first observed frame of the fitted four-frame window of Appendix~\ref{data:md}, with fixed-step RK4 at 100 substeps per segment. The substep count equals the frame stride, so every rollout step lands on a raw MD frame and every deviation is aligned to the native grid. The fields are queried at their zero-step slice, which is the continuous-time velocity, and no re-anchoring at observed frames takes place. Panels (a) and (b) draw the rolled-out molecule in three dimensions over the first segment for the score-target model and for the linear-target model, with the intermediate configurations faint and the first and the last configuration solid, and with the drawn bonds taken from the static molecular graph of Appendix~\ref{data:md}. Panels (c) and (d) plot the O--H bond length and the C--O--H bond angle of all three rollouts against rollout time, which is counted in observed frame intervals so that the integer ticks and the light vertical lines mark the observed frames, together with the true trajectory at native resolution as a dashed black curve. The atom roles that define these two coordinates are identified once on the first frame and reused for every curve, so the same three atoms are tracked throughout.

\section{Additional Results}
\label{app:additional_results}

\subsection{Demonstrative Example}
\label{app:res_toy}

Figure~\ref{fig:toy_example} shows the setting of the 2D dataset together with conditional rollouts for increasing observed context and the latent distributions.

\begin{figure}[H]
  \centering
  \newcommand{\panelw}{0.19\textwidth}

  \begin{minipage}[t]{\panelw}\centering
    \includegraphics[width=\linewidth]{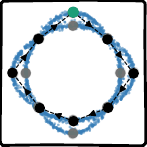}\\
    \small (a)
  \end{minipage}
  \hfill
  \begin{minipage}[t]{\panelw}\centering
    \includegraphics[width=\linewidth]{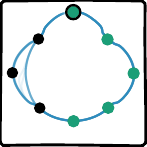}\\
    \small (b)
  \end{minipage}
  \hfill
  \begin{minipage}[t]{\panelw}\centering
    \includegraphics[width=\linewidth]{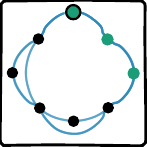}\\
    \small (c)
  \end{minipage}
  \hfill
  \begin{minipage}[t]{\panelw}\centering
    \includegraphics[width=\linewidth]{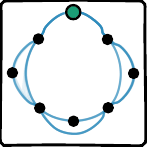}\\
    \small (d)
  \end{minipage}
  \hfill
  \begin{minipage}[t]{\panelw}\centering
    \includegraphics[width=\linewidth]{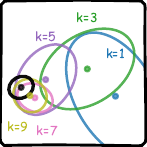}\\
    \small (e)
  \end{minipage}
  \caption{2D example. (a) Learned manifold and one example trajectory. (b)--(d) Conditional trajectory samples for increasing observed context (green). (e) PCA projection of the latent distributions.}
  \label{fig:toy_example}
\end{figure}

Figure~\ref{fig:app_toy} shows additional conditional samples. The observed prefix is shown in green, while the sampled continuations illustrate how the model uses the latent variable to represent unresolved branch choices. As the amount of conditioning increases, the samples preserve the observed arch decisions while maintaining diversity over the unobserved parts of the loop.

\begin{figure}[H]
    \centering
    \includegraphics[width=1.0\linewidth]{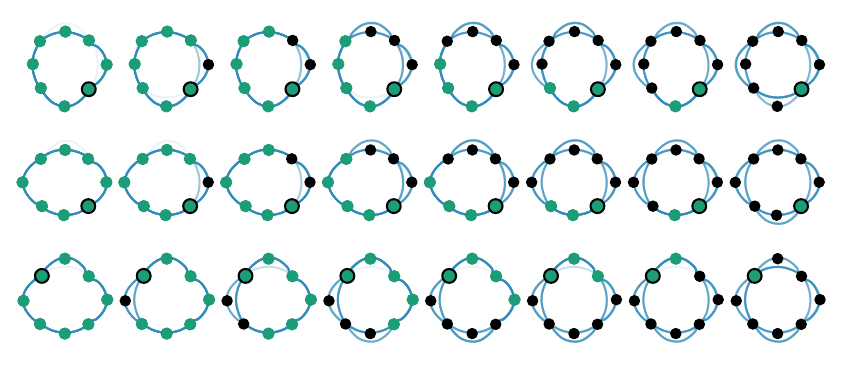}
    \caption{Additional conditional samples on the 2D dataset. The model follows the curved score-induced manifold and preserves stochasticity over future branch choices that are not fixed by the observed context.}
    \label{fig:app_toy}
\end{figure}

Supplementing Figure~\ref{fig:toy_fields}(b), Figure~\ref{fig:app_toy_field2} shows another field for a different latent variable, resulting in another branch configuration.
Switching \(z\) thus switches the entire field, so each latent corresponds to a single trajectory and there is no competing path toward which the transverse correction could pull off-path states.

\begin{figure}[H]
    \centering
    \includegraphics[width=0.3\linewidth]{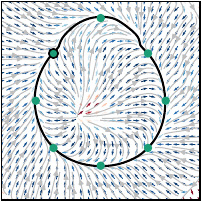}
    \caption{The learned field \(v_\theta(x\mid z)\) conditioned on a second branch configuration (rendering as in Figure~\ref{fig:toy_fields}).}
    \label{fig:app_toy_field2}
\end{figure}

\begin{table}[H]
\centering
\caption{Ablations on the 2D example. Each row removes one component (``\(-\) correction'' keeps the input noise but drops the target correction). Branch acc.\ has chance level \(\approx 1/16\), the training rate is \(-\lambda=-10\), and the metrics are defined in Appendix~\ref{data:toy}.}
\label{tab:toy_ablation}
\small
\begin{tabular}{l ccc}
  \toprule
  Model & off-manifold $\downarrow$ & branch acc. $\uparrow$ & rate $\downarrow$ \\
  \midrule
  ours (score, corrected, latent) & \textbf{0.007} & \textbf{1.000} & \textbf{$-$5.4} \\
  $-$ latent                      & 0.009 & 0.047 & $-$4.0 \\
  $-$ correction                  & 0.024 & \textbf{1.000} & $+$0.0 \\
  $-$ score $\to$ linear          & 0.036 & \textbf{1.000} & $-$3.6 \\
  \bottomrule
\end{tabular}
\end{table}

\subsection{Robustness}
\label{app:res_contracting}

Figure~\ref{fig:contracting} reports the rollout accuracy of the reconstruction experiment of Section~\ref{sec:robustness_exp} under the three inference regimes. The corrected objective attains consistently lower error across solver settings and remains stable at coarse step sizes.

\begin{figure}[H]
  \centering
  \begin{minipage}[t]{.32\textwidth}\centering
    \includegraphics[width=\linewidth]{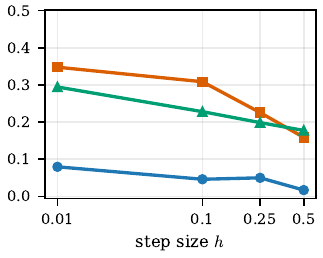}\\
    \small (a) Secant steps with \(h>0\)
  \end{minipage}
  \hfill
  \begin{minipage}[t]{.32\textwidth}\centering
    \includegraphics[width=\linewidth]{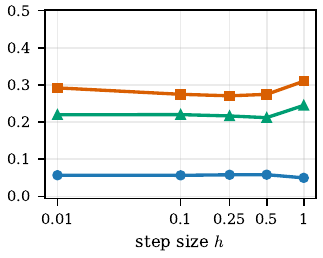}\\
    \small (b) fixed-step ODE solve
  \end{minipage}
  \hfill
  \begin{minipage}[t]{.32\textwidth}\centering
    \includegraphics[width=\linewidth]{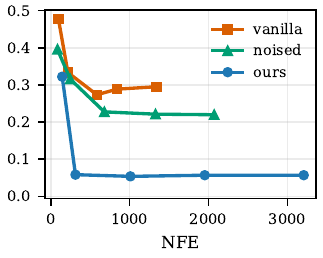}\\
    \small (c) adaptive ODE solve
  \end{minipage}
  \caption{Mean LPIPS over rollout for the reconstruction experiment under three inference regimes.}
  \label{fig:contracting}
\end{figure}

Figure~\ref{fig:app_robustness} provides a qualitative view. Without transverse correction, small deviations from the reference path can compound during rollout. The robustified objective instead trains the vector field to return perturbed states toward the interpolation path, which leads to more stable long-horizon reconstructions.
This demonstrates that the objective which promotes path-relative transverse contraction does not only improve rollouts, but is necessary to prevent catastrophic error accumulation.

\begin{figure}[H]
    \centering
    \includegraphics[width=1.0\linewidth]{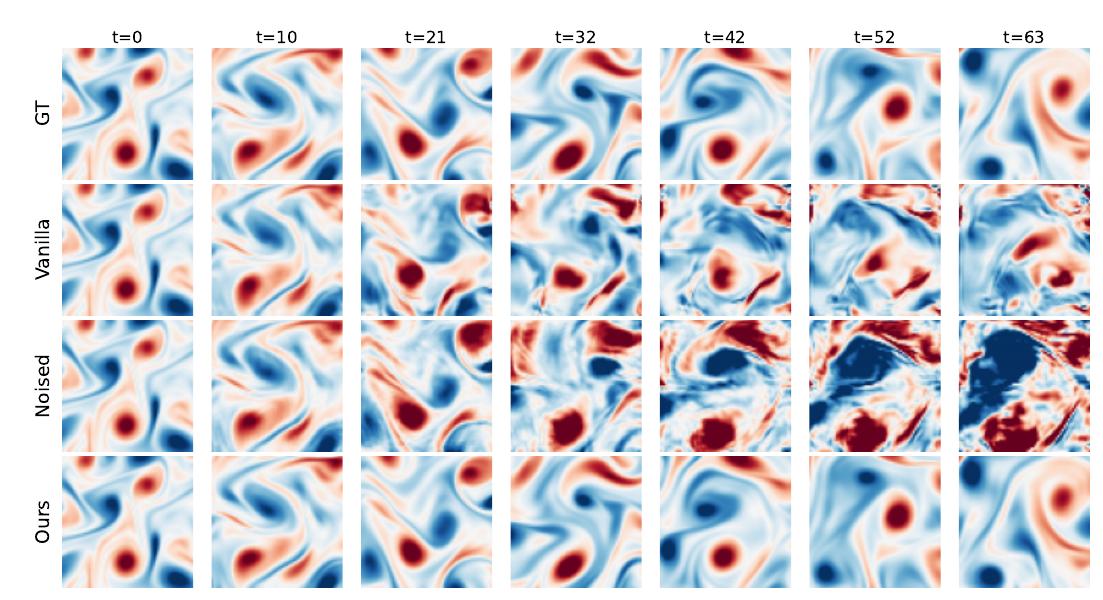}
    \caption{Qualitative robustness comparison for Navier--Stokes reconstruction under coarse secant rollout with step size \(h=0.5\). The transverse robustness objective reduces accumulated drift by training perturbed states to return toward the interpolation path.}
    \label{fig:app_robustness}
\end{figure}

\subsection{Plausibility of Intermediates}
\label{app:res_plausibility}

Tables~\ref{tab:app_x1_full}--\ref{tab:app_x2_refine} report the full results for the experiments of Section~\ref{sec:plausibility}, with data, training, and evaluation details given in Appendices~\ref{app:gs_fine}--\ref{app:x2}.

Table~\ref{tab:app_x1_full} adds the velocity cosine to the two metrics of Table~\ref{tab:gs_interp_fine}. The two paths separate cleanly at every interior offset, and the score path is not only closer to the true intermediates, it also follows the direction of the true motion more faithfully, at a cosine of 0.996 against 0.965 for the linear path.

\begin{table}[H]
\centering
\caption{Frozen interpolation paths against the true fine-time intermediates (mean \(\pm\) std over 8 trajectories).}
\label{tab:app_x1_full}
\small
\begin{tabular}{l ccc}
  \toprule
  Path & rel-$L_2$ $\downarrow$ & cos-vel $\uparrow$ & spectral $\downarrow$ \\
  \midrule
  linear & 0.055 $\pm$ 0.001 & 0.965 $\pm$ 0.000 & 0.471 $\pm$ 0.004 \\
  score  & \textbf{0.019 $\pm$ 0.001} & \textbf{0.996 $\pm$ 0.000} & \textbf{0.134 $\pm$ 0.017} \\
  \bottomrule
\end{tabular}
\end{table}

The training grid corresponds to a segment stride of \(w=10\) fine steps, and the generative fields are retrained at the two wider spacings \(w=15\) and \(w=20\), where the endpoints leave the interior increasingly underdetermined. The main text reports \(w=20\), twice the training spacing, and Table~\ref{tab:app_x1w} reports both spacings with all metrics and the endpoint error, together with the frozen interpolator that supplies the score targets.

The ordering is stable across the two spacings. The Neural ODE has the lowest tracking error and the highest velocity cosine everywhere. Its intermediates are nevertheless the least plausible under the spectral diagnostic at \(w=20\), and its spectral value degrades faster with spacing than that of either simulation-free model, from 0.303 to 0.482, which is how it moves from the middle of the field at \(w=15\) to last place at \(w=20\). The score-supervised field instead stays with its frozen teacher on every metric, at \(w=20\) reaching rel-\(L_2\) 0.115 against the teacher's 0.116 and a spectral value of 0.335 against 0.331. What the field inherits is therefore the plausibility of the paths it was trained on, and training the same backbone on linear targets inherits their implausibility in the same way.

\begin{table}[H]
\centering
\caption{Generative fields at the two wider frame spacings, with the frozen teacher for reference. The spectral diagnostic is defined in Appendix~\ref{app:gs_fine}.}
\label{tab:app_x1w}
\small
\begin{tabular}{ll cccc}
  \toprule
  & Model & rel-$L_2$ $\downarrow$ & cos-vel $\uparrow$ & spectral $\downarrow$ & endpoint $\downarrow$ \\
  \midrule
  \multirow{4}{*}{$w=15$}
  & score interpolator (teacher) & 0.055 & 0.983 & 0.198 & --- \\
  & score targets & 0.053 & 0.984 & \textbf{0.230} & 0.037 \\
  & linear targets & 0.106 & 0.925 & 0.442 & 0.052 \\
  & Neural ODE & \textbf{0.030} & \textbf{0.996} & 0.303 & \textbf{0.003} \\
  \midrule
  \multirow{4}{*}{$w=20$}
  & score interpolator (teacher) & 0.116 & 0.954 & 0.331 & --- \\
  & score targets & 0.115 & 0.954 & \textbf{0.335} & 0.059 \\
  & linear targets & 0.177 & 0.886 & 0.439 & 0.064 \\
  & Neural ODE & \textbf{0.060} & \textbf{0.984} & 0.482 & \textbf{0.003} \\
  \bottomrule
\end{tabular}
\end{table}

Exact recovery of the intermediate dynamics cannot be guaranteed from the score prior alone, so we ask what happens when the governing equations are available and can be supplied to the interpolator instead. This uses a fully observed variant of Gray--Scott that stores both species, which makes the PDE residual computable. Table~\ref{tab:app_x2_variants} compares the three supervision variants. Score geometry alone improves 2.7\(\times\) over linear interpolation without any knowledge of the equations, and training the same interpolator to minimize the PDE residual instead brings the path to within 3\(\times\) of the ground-truth residual floor. The same training objective therefore admits the score prior when nothing else is known and the governing equations when they are available.

\begin{table}[H]
\centering
\caption{Interpolation with score geometry against known dynamics on the fully observed two-species variant (mean \(\pm\) std over 8 trajectories, ground-truth residual floor \(1.79\cdot10^{-5}\)).}
\label{tab:app_x2_variants}
\small
\begin{tabular}{l ccc}
  \toprule
  Variant & rel-$L_2$ $\downarrow$ & cos-vel $\uparrow$ & residual $\downarrow$ \\
  \midrule
  V-lin (no training) & 0.0768 $\pm$ 0.0016 & 0.943 & $7.9\cdot10^{-4}$ \\
  V-score (geometry only) & 0.0283 $\pm$ 0.0007 & 0.992 & $3.5\cdot10^{-4}$ \\
  V-phys (residual supervision) & \textbf{0.0035 $\pm$ 0.0001} & \textbf{0.9998} & {\boldmath$5.4\cdot10^{-5}$} \\
  \bottomrule
\end{tabular}
\end{table}

The residual can also be reduced after training, by descending on it directly at inference time without any network in the loop. Table~\ref{tab:app_x2_refine} reports this for an increasing number of optimization steps \(K\). The score path starts at a lower residual and also refines further, reaching \(5.6\cdot10^{-5}\) at \(K=2000\), which is within \(3.1\times\) of the true trajectory's residual, while the linear path stops at \(1.6\cdot10^{-4}\), or \(8.9\times\). Both paths improve in relative \(L_2\) as well, but the linear path stays between \(2.6\) and \(2.8\) times further from the true states than the score path at every budget, so refinement removes part of the residual that the linear path introduces without recovering the geometry it missed. The \(K=0\) rows coincide with the unrefined paths by construction, and the small residual differences to Table~\ref{tab:app_x2_variants} stem from the discretized path representation of the refinement evaluation.

\begin{table}[H]
\centering
\caption{Physics refinement of interpolation paths on the fully observed two-species variant for increasing refinement budgets \(K\) (mean \(\pm\) std over 8 trajectories where shown).}
\label{tab:app_x2_refine}
\small
\begin{tabular}{lr ccc}
  \toprule
  Init.\ path & $K$ & rel-$L_2$ $\downarrow$ & cos-vel $\uparrow$ & residual $\downarrow$ \\
  \midrule
  linear & 0    & 0.0768 $\pm$ 0.0016 & 0.943 & $7.9\cdot10^{-4}$ \\
  linear & 20   & 0.0771 $\pm$ 0.0016 & 0.944 & $6.4\cdot10^{-4}$ \\
  linear & 100  & 0.0764 $\pm$ 0.0016 & 0.950 & $4.6\cdot10^{-4}$ \\
  linear & 500  & 0.0708 $\pm$ 0.0014 & 0.965 & $2.9\cdot10^{-4}$ \\
  linear & 2000 & 0.0588 $\pm$ 0.0012 & 0.982 & $1.6\cdot10^{-4}$ \\
  \midrule
  score  & 0    & 0.0283 $\pm$ 0.0007 & 0.992 & $3.4\cdot10^{-4}$ \\
  score  & 20   & 0.0278 $\pm$ 0.0006 & 0.993 & $2.6\cdot10^{-4}$ \\
  score  & 100  & 0.0269 $\pm$ 0.0006 & 0.994 & $1.8\cdot10^{-4}$ \\
  score  & 500  & 0.0252 $\pm$ 0.0006 & 0.995 & $1.0\cdot10^{-4}$ \\
  score  & 2000 & \textbf{0.0222 $\pm$ 0.0006} & \textbf{0.997} & {\boldmath$5.6\cdot10^{-5}$} \\
  \bottomrule
\end{tabular}
\end{table}

Figure~\ref{fig:ns_interp} repeats the comparison qualitatively on a second system, Navier--Stokes, using two otherwise identical generative models trained with score-induced and with linear interpolation targets. Since the endpoints are nearby frames of a smooth trajectory, the enstrophy of plausible intermediate states should vary smoothly. Both models capture the coarse transition, but the linear model produces a visibly less regular sub-frame evolution with sharp changes in enstrophy, while the score-supervised model varies smoothly under this diagnostic (setup in Appendix~\ref{app:inference_details}).

\begin{figure}[H]
  \centering
  \begin{minipage}[t]{0.65\linewidth}
    \centering
    \includegraphics[width=\linewidth]{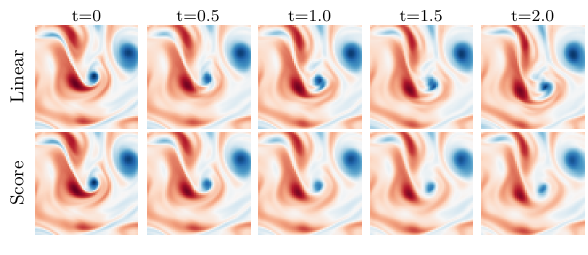}
  \end{minipage}\hfill
  \begin{minipage}[t]{0.3\linewidth}
    \centering
    \includegraphics[width=\linewidth]{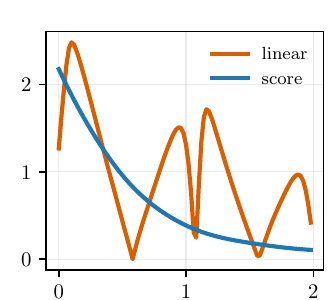}
  \end{minipage}
  \caption{Navier--Stokes sub-frame dynamics for score-induced and linear interpolation targets. Left: generated intermediate vorticity fields. Right: change in enstrophy along the interpolation.}
  \label{fig:ns_interp}
\end{figure}

The molecular-dynamics comparison of Section~\ref{sec:plausibility}, where the Neural ODE fails on both counts at once, is documented with its diagnostics in Appendix~\ref{data:md}.

\subsection{Natural Videos: Interpolation}
\label{app:res_kth}

Figure~\ref{fig:kth_interp} complements the extrapolation comparison of Figure~\ref{fig:kth_extrap} with the interpolation setting of Section~\ref{sec:temporal_sr} (observe every second frame and predict the held-out ones, quantitative results in Table~\ref{tab:vidode}). All model outputs are shown on the shared half-frame display grid, generated in a single rollout as described in Appendix~\ref{data:kth}. At the unsupervised half-times, Vid-ODE's recursively warped outputs exhibit double-exposure ghosting on the moving limbs, the failure mode that its frame metrics at observed timestamps do not expose, while our rollout remains closer to the true trajectory.

\begin{figure}[H]
    \centering
    \includegraphics[width=\linewidth]{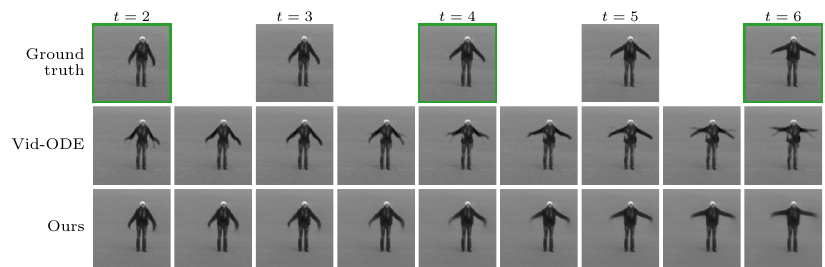}
    \caption{Interpolation on KTH. Rows: ground truth at integer times (observed frames framed in green; half-times have no ground truth), Vid-ODE, and ours, with all model outputs on the shared half-frame display grid. Vid-ODE drifts off the data manifold at intermediate times (ghosted arms).}
    \label{fig:kth_interp}
\end{figure}

\subsection{PDE-governed Spatiotemporal Fields}
\label{app:res_pde}

Figure~\ref{fig:gray_scott_appendix} shows qualitative long-horizon Gray--Scott rollouts from the latent-conditioned model. The model is conditioned only on the first observed concentration field and samples different latent variables from the learned prior. Although training uses subsequences of length \(16\) and the available training trajectories have length at most \(64\), the generated rollouts are shown up to \(T=256\). The samples remain visually plausible while exhibiting distinct futures, reflecting ambiguity induced by the unobserved chemical species.

\begin{figure}[h]
    \centering
    \includegraphics[width=1.0\linewidth]{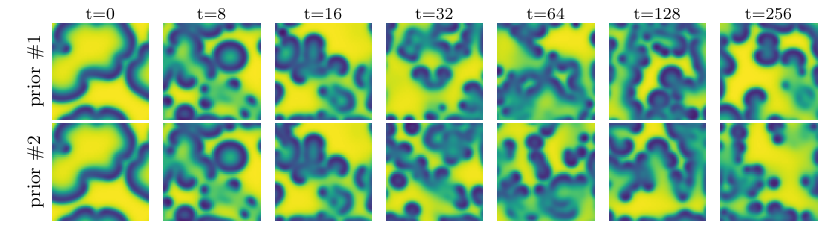}
    \caption{Gray--Scott long-horizon rollouts from the latent-conditioned model. The model is conditioned only on the first observed concentration field and samples different latent variables, producing distinct plausible futures up to \(T=256\).}
    \label{fig:gray_scott_appendix}
\end{figure}

\end{document}